\documentclass{article}

\PassOptionsToPackage{numbers, compress}{natbib}

     \usepackage[final]{neurips_2023}

\usepackage[utf8]{inputenc} 
\usepackage[T1]{fontenc}    
\usepackage{hyperref}       
\usepackage{url}            
\usepackage{booktabs}       
\usepackage{amsfonts}       
\usepackage{nicefrac}       
\usepackage{microtype}      
\usepackage{xcolor}         

\usepackage{titletoc}

\usepackage{graphicx}

\usepackage{amsmath}
\usepackage{amssymb}
\usepackage{mathtools}
\usepackage{amsthm}

\usepackage{wrapfig}

\usepackage{algorithmic}
\usepackage{algorithm}

\usepackage{soul}
\usepackage{multirow}

\newtheorem{theorem}{Theorem}[section]
\newtheorem{proposition}[theorem]{Proposition}
\newtheorem{lemma}[theorem]{Lemma}

\newtheorem{definition}[theorem]{Definition}

\newtheorem{remark}[theorem]{Remark}
\newcommand{\fset}{D}
\newcommand{\fnum}{d}

\DeclareMathOperator*{\argmin}{arg\,min}

\title{SHAP-IQ: Unified Approximation of any-order Shapley Interactions}

\author{%
  Fabian Fumagalli\thanks{denotes equal contribution}\\
  Bielefeld University, CITEC\\ 
  D-33619, Bielefeld, Germany\\
  \texttt{ffumagalli@techfak.uni-bielefeld.de}
  \And
  Maximilian Muschalik$^{*}$\\
  LMU Munich, MCML Munich\\
  D-80539, Munich, Germany\\
  \texttt{maximilian.muschalik@ifi.lmu.de}
  \And
  Patrick Kolpaczki\\
  Paderborn University\\
  D-33098, Paderborn, Germany\\
  \texttt{patrick.kolpaczki@upb.de}
  \And
  Eyke Hüllermeier\\
  LMU Munich, MCML Munich\\
  D-80539, Munich, Germany\\
  \texttt{eyke@ifi.lmu.de}
  \And
  Barbara Hammer\\
  Bielefeld University, CITEC\\
  D-33619, Bielefeld, Germany\\
  \texttt{bhammer@techfak.uni-bielefeld.de}
}

\begin{document}

\maketitle

\begin{abstract}
Predominately in explainable artificial intelligence (XAI) research, the Shapley value (SV) is applied to determine feature attributions for any black box model. Shapley interaction indices extend the SV to define any-order feature interactions. Defining a unique Shapley interaction index is an open research question and, so far, three definitions have been proposed, which differ by their choice of axioms. Moreover, each definition requires a specific approximation technique. Here, we propose SHAPley Interaction Quantification (SHAP-IQ), an efficient sampling-based approximator to compute Shapley interactions for arbitrary cardinal interaction indices (CII), i.e. interaction indices that satisfy the linearity, symmetry and dummy axiom. SHAP-IQ is based on a novel representation and, in contrast to existing methods, we provide theoretical guarantees for its approximation quality, as well as estimates for the variance of the point estimates. For the special case of SV, our approach reveals a novel representation of the SV and corresponds to Unbiased KernelSHAP with a greatly simplified calculation. We illustrate the computational efficiency and effectiveness by explaining language, image classification and high-dimensional synthetic models.
\end{abstract}


\section{Introduction}
Feature attributions are a prevalent approach to interpret black box machine learning (ML) models \cite{Adadi.2022,chen2023algorithms,Lundberg_Lee_2017}.
However, in many real-world applications, such as understanding drug-drug interactions, mutational events or complex language models, quantifying \emph{interactions} between features is essential, too \cite{Wright.2016,DBLP:journals/bmcbi/LiuZG19,Tsang.2020}.
Feature interactions provide a more comprehensive explanation, which can be seen as an enrichment of feature attributions \cite{pmlr-v206-bordt23a,Sundararajan_Dhamdhere_Agarwal_2020,Tsai_Yeh_Ravikumar_2022}.
While feature attributions quantify the contribution of \emph{single} features to the model's prediction or performance, feature interactions quantify the contribution of a \emph{group} of features to the model's prediction or performance.

In this work, we are interested in feature interactions that make use of the Shapley value (SV) and its extension to Shapley interactions.
The SV is a concept from cooperative game theory that has been used, apart from feature attributions \cite{chen2023algorithms}, as a basis for many Shapley-based explanations \cite{Jia.2019,ghorbani.2019,yeh2018representer}.
It distinguishes itself through uniqueness given a set of intuitive axioms.
A number of approaches extend Shapley-based explanations to feature interactions \cite{Grabisch_Roubens_1999,pmlr-v206-bordt23a,Sundararajan_Dhamdhere_Agarwal_2020,Tsai_Yeh_Ravikumar_2022}.
Yet, in contrast to the SV, a ``natural'' extension of the intuitive set of axioms for a unique Shapley interaction index is less clear.
Moreover, its efficient computation is challenging and, so far, approximation approaches are specifically tailored to the particular definition.

In this paper, we consider a more general class of interaction indices, known as cardinal interaction indices (CII) \cite{Grabisch_Roubens_1999}, which covers all currently proposed definitions and all other that satisfy the (generalized) linearity, symmetry and dummy axiom.
We present SHAPley Interaction Quantification (SHAP-IQ), a sampling-based unified approximation method.
It is substantiated by mathematical guarantees and can be applied to \emph{any} CII to approximate any-order interaction scores efficiently.

\paragraph{Contribution.} Our main contributions include:
\begin{itemize}
\item We consider a general form of interaction indices, known as CII (Definition \ref{def_SI}) and establish a novel representation (Theorem \ref{thm::si}), which we utilize to construct SHAP-IQ (Definition \ref{def::shapx}), an efficient sampling-based estimator.\footnote{The \emph{shapiq} package extends on the well-known \emph{shap} library and can be found at \url{https://pypi.org/project/shapiq/}.}
\item We show that SHAP-IQ is unbiased, consistent and provide a general approximation bound (Theorem~\ref{thm::unbiased_consistent}).
 We further prove that SHAP-IQ maintains the efficiency condition for n-Shapley Values \cite{pmlr-v206-bordt23a} and the Shapley Taylor Interaction Index \cite{Sundararajan_Dhamdhere_Agarwal_2020} (Theorem~\ref{thm::s-efficiency}).
\item For the SV, we find a novel representation (Theorem \ref{thm::SV_representation}). 
We further prove that SHAP-IQ is linked to Unbiased KernelSHAP \cite{Covert_Lee_2021} (Theorem \ref{thm::u_ksh}) and greatly simplifies its representation.
\item We use SHAP-IQ to compute any-order n-Shapley Values on different ML models and demonstrate that it outperforms existing baseline methods.
We further contrast different existing CIIs and compare SHAP-IQ to the corresponding baseline approximation method.
\end{itemize}

\begin{figure}[t]
    \centering
    \includegraphics[width=0.9\textwidth]{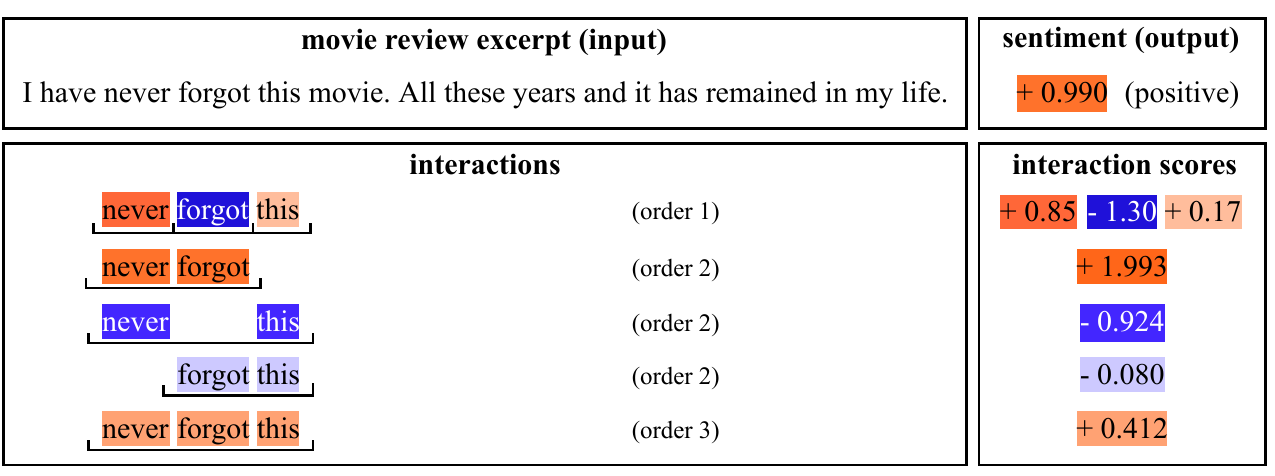}
    \caption{Interaction scores for a movie review excerpt presented to a sentiment analysis model.}
    \label{fig:intro_illustration}
\end{figure}

\section{Related Work}
The Shapley Interaction Index (SII) \cite{Grabisch_Roubens_1999}, its efficiency preserving aggregation as n-Shapley Values (n-SII) \cite{pmlr-v206-bordt23a}, the Shapley Taylor Interaction (STI) \cite{Sundararajan_Dhamdhere_Agarwal_2020} and the Faithful Shapley Interaction Index (FSI) \cite{Tsai_Yeh_Ravikumar_2022} offer different ways of extending the SV to interactions, which extend on the linearity, symmetry and dummy axiom to provide a uniquely defined interaction index. 
SII and STI extend on axiomatic properties of the weighted sum for SV \cite{Shapley.1953}, whereas FSI extends on the axiomatic properties of the Shapley interaction as the solution to a weighted least square solution \cite{Ruiz_Valenciano_Zarzuelo_1996,Ruiz_Valenciano_Zarzuelo_1998}.

In the field of cooperative game theory, interactions have also been studied from a theoretical perspective as the solution of a weighted least square problem and the sum of marginal contributions with constant weights, which both yield a generalized Banzhaf value \cite{Hammer_Holzman_1992,Grabisch_Marichal_Roubens_2000}.

In the ML community, interactions of features have been studied from a practical perspective for text \cite{Murdoch.2018} and image \cite{Tsang.2020b} data, for specific models, such as neural networks \cite{Tsang.2018,Cui.2020,Singh.2019,Janizek.2021} or tree based models \cite{Lundberg.2020}.
Other concepts of interactions have been discussed in \cite{Tsang.2020} using marginal contributions, from a statistical perspective with functional decomposition \cite{Molnar.2019} and improved white box models with interaction terms \cite{Lou.2013}.

To approximate SII and STI, a permutation-based method \cite{Sundararajan_Dhamdhere_Agarwal_2020,Tsai_Yeh_Ravikumar_2022}, as an extension of ApproShapley \cite{Castro.2009}, was suggested, whereas FSI relies on a kernel-based approximation, similar to KernelSHAP \cite{Lundberg_Lee_2017}, which utilizes a representation of the SV as the solution of a weighted least square problem \cite{Charnes_Golany_Keane_Rousseau_1988}.
Unlike these specific approaches, we consider general CIIs, which subsume all of the above mentioned measures, and we propose a generic approximation technique, which can be accompanied by mathematical guarantees.
In case of the SV, our approximation of the CII is related to Unbiased KernelSHAP \cite{Covert_Lee_2021}, which is a variant of KernelSHAP \cite{Lundberg_Lee_2017}.
It is further related to stratified sampling approximations for the SV and our sampling approach can be seen as flexible framework to find the optimum allocation for each stratum \cite{Castro.2017}.

\section{The Cardinal Interaction Index (CII) and Shapley-based Explanations}
\label{sec::shapley-based-explanations}
In this section, we review Shapley-based explanations and introduce the CII, which we aim to approximate in Section \ref{sec::shap-iq}.
We further introduce existing baseline methods for specific CIIs and Unbiased KernelSHAP for the SV, which is linked to our proposed method.

\paragraph{Notations.}
We refer to the model behavior on a set of features $\fset = \{1,\dots,\fnum\}$ as a function $\nu: \mathcal P(D) \to \mathbb{R}$, where $\mathcal P(\fset)$ refers to the power set of $\fset$.
We denote $\nu_0(T) := \nu(T)-\nu(\emptyset)$, which is the default setting in game theory and also known as \emph{set function} \cite{Grabisch_Roubens_1999,DBLP:journals/geb/FujimotoKM06}.
The subsets $S \subseteq D$ refer to the set of features (or players in game theory) of which the \emph{interaction} is computed, where we use lower case letters for the cardinality, i.e. $s := \vert S \vert$.
The maximum order interaction of interest is denoted with $s_0$ and we use the set $\mathcal T_{k} := \{T \subseteq \fset: k \leq t \leq \fnum-k\}$ and the set of interactions $\mathcal S_{s_0} := \{S \subseteq \fset \mid s \leq s_0, S \neq \emptyset\}$.
For a subset $T \subseteq \fset$, we refer to the binary representation as $Z_T=(z_1,\dots,z_\fnum) \in \{0,1\}^\fnum$ with $z_i = \mathbf{1}(i \in T)$ for $i=1,\dots,\fnum$, where $\mathbf{1}$ refers to the indicator function.
We further denote the Shapley kernel \cite{Charnes_Golany_Keane_Rousseau_1988,Lundberg_Lee_2017} as $\mu(t) := \frac{1}{\fnum-1}\binom{\fnum-2}{t-1}^{-1}$.

Removal-based explanations \cite{Covert_Lundberg_Lee_2021} consider a model that is trained on $\fnum$ features, where the goal is to examine a \emph{model behavior} that is defined on a subset of features.
The model behavior $\nu$ for a subset of features could, for instance, be a particular model prediction for one input (local explanation) or an overall measure of model performance (global explanation), if only this subset of features is known \cite{Covert_Lundberg_Lee_2021}.
To quantify the contribution for individual features, the change in model behavior is evaluated, if the feature is removed from the model.
To restrict a ML model on a subset of features, different \emph{feature removal} techniques have been proposed, such as marginalization of features or retraining the model \cite{Covert_Lundberg_Lee_2021}.
To quantify the impact of a single feature $i \in \fset$ on the model behavior $\nu$ it is then intuitive to compute the difference $\delta^\nu_{\{i\}}(T) = \nu(T\cup\{i\})-\nu(T)$ for subsets $T \subseteq D \setminus \{i\}$. 
For a distinct pair of features $(i,j)$ with $i,j \in \fset$, a natural extension is $\delta^\nu_{\{i,j\}}(T) = \nu(T \cup \{i,j\}) - \nu(T) - \delta^\nu_{\{i\}}(T) - \delta^\nu_{\{j\}}(T)$ for $T \in \fset \setminus \{i,j\}$, i.e. subtracting the contribution of single features from the joint impact of both features.
The following definition generalizes this recursion and is known as \emph{discrete derivative} or \emph{$S$-derivative} \cite{DBLP:journals/geb/FujimotoKM06}.
\begin{definition}[Discrete Derivative \cite{DBLP:journals/geb/FujimotoKM06}]
For $S \subseteq \fset$ the \emph{S-derivative of $\nu$ at $T \subseteq D \setminus S$} is
\begin{equation*}
    \delta^\nu_S(T) := \sum_{L\subseteq S}(-1)^{s-l}\nu(T \cup L).
\end{equation*}
\end{definition}
To obtain an attribution score, the marginal contributions on different subsets $T \subseteq \fset \setminus S$ are aggregated using a specific \emph{summary technique}.
In this work, we are interested in the approximation and extension of one particular summary technique for single features $i \in \fset$, called the Shapley value \cite{Shapley.1953}, independent of model behavior and feature removal.
\begin{definition}[Shapley Value (SV) \cite{Shapley.1953}]
The SV is $I^{\text{SV}}(i) = \sum_{T \subseteq \fset\setminus \{i\}}\frac{(\fnum-t-1)!t!}{\fnum!}\delta^\nu_{\{i\}}(T)$, $i \in \fset$.
\end{definition}
The SV is the unique attribution method that fulfills the axioms: symmetry (attributions are independent of feature ordering), linearity (in terms of the model behavior $\nu$), dummy (if a feature does not change $\nu$ then its attribution is zero) and efficiency (the sum of attributions are equal to $\nu_0(\fset)$) \cite{Shapley.1953}.
However, the SV  does not give any information about the interactions between two or more features.
A suitable extension of the SV for interactions of features $S \subseteq \fset$ remains an open question, as different axiomatic extensions have been proposed \cite{Grabisch_Roubens_1999,Sundararajan_Dhamdhere_Agarwal_2020,Tsai_Yeh_Ravikumar_2022}.
In this work, we thus consider a broad class of interaction indices, known as CIIs \cite{Grabisch_Roubens_1999,Tsai_Yeh_Ravikumar_2022}, which subsumes popular choices.

\begin{definition}[Cardinal Interaction Index (CII) \cite{Grabisch_Roubens_1999}]\label{def_SI}
A CII is an interaction index of the form
\begin{align*}
I^m(S) := \sum_{T \subseteq \fset\setminus S} m_{s}(t) \delta^\nu_S(T) \text{ with weights } m_s(t) \text{ for } s=1,\dots,s_0 \text{ and }t=0,\dots,\fnum-s.
\end{align*}
\end{definition}

\begin{remark}
    It was shown that every interaction index satisfying the generalized linearity, symmetry and dummy axioms can be represented as a CII \cite{Grabisch_Roubens_1999}.
    If $\sum_{t=0}^{\fnum-s}\binom{\fnum-s}{t}m_s(t)=1$, then the CII is also referred to as a cardinal-probabilistic interaction index (CPII) \cite{DBLP:journals/geb/FujimotoKM06}.
\end{remark}
In this paper, we present a unified approximation technique for arbitrary CIIs.
\subsection{Shapley Interaction Index (SII) and other CIIs}
In the following, we introduce prominent examples of CIIs.
For further details on the axioms and exact definitions, we refer to the appendix.
The SII \cite{Grabisch_Roubens_1999} is a direct extension of the SV, that relies on an additional \emph{recursive} axiom to obtain a unique CII.

\begin{definition}[Shapley Interaction Index (SII)\cite{Grabisch_Roubens_1999}]
The SII is a CII defined as
\begin{equation*}
    I^{\text{SII}}(S) := \sum_{T \subseteq \fset \setminus S} m^{\text{SII}}_{s}(t) \delta^\nu_S(T) \text{ and } m^{\text{SII}}_{s}(t) := \frac{(\fnum-t-s)!t!}{(\fnum-s+1)!}.
\end{equation*}  
\end{definition}

It has been shown that the SII is a CPII \cite{DBLP:journals/geb/FujimotoKM06}.
In contrast to the SV, the SII does not fulfill the efficiency axiom, which is a desirable property in the context of ML.
Therefore an extension of SII, as well as other interaction indices have been proposed.

\paragraph{n-Shapley Values (n-SII) and other interaction indices.}
The efficiency axiom for interaction indices of \emph{maximum interaction order} $1\leq s_0\leq \fnum$ requires that the sum of $I^m(S)$ up to order $s_0$ equals $\nu_0(\fset)$.
\begin{definition}[Efficiency \cite{Sundararajan_Dhamdhere_Agarwal_2020,Tsai_Yeh_Ravikumar_2022}]
A CII is \emph{efficient of order $s_0$}, if $\sum_{S \in \mathcal S_{s_0}} I^m(S) = \nu_0(\fset)$, where $\mathcal S_{s_0}$ is the set of interactions up to order $s_0$.
\end{definition}

In \cite{pmlr-v206-bordt23a}, an aggregation of SII was proposed to obtain n-SII $I^{\text{n-SII}}_{s_0}(S)$ of order $s_0$ that satisfies efficiency.
Other axiomatic approaches directly require efficiency together with the linearity, symmetry and dummy axioms, and omit the recursive axiom of SII.
However, in contrast to the SV, this axiom alone does not yield a unique interaction index \cite{Sundararajan_Dhamdhere_Agarwal_2020,Tsai_Yeh_Ravikumar_2022}.
The STI \cite{Sundararajan_Dhamdhere_Agarwal_2020} requires the efficiency axiom and an additional interaction distribution axiom.
On the other hand, the FSI \cite{Tsai_Yeh_Ravikumar_2022} requires the efficiency axiom and the faithfulness property, that relates the interaction index to a solution of a constrained weighted least square problem.
The choice of axioms of SII (n-SII), STI and FSI yield a unique interaction index that reduces to the SVs for $s_0=1$.
For FSI, it was shown that the top-order interactions define a CPII \cite[Proposition 21]{Tsai_Yeh_Ravikumar_2022}, which is also easily verified for STI.
All orders of interactions of FSI and STI can in general be represented as a CII, as they fulfill the linearity, symmetry and dummy axioms \cite[Proposition 5]{Grabisch_Roubens_1999}.
However, it was noted that for FSI a simple closed-form solution for lower-order interactions in terms of discrete derivatives remains unclear \cite[Lemma 70]{Tsai_Yeh_Ravikumar_2022}.

\subsection{Baseline Approximations of SII, STI and FSI.}
By definition, the number of evaluations of $\nu$ in $I$, which constitutes the limiting factor in ML, grows exponentially with $\fnum$ and thus, in practice, approximation methods are required.
Currently, there does not exist an approximation for the general CII definition, as each index (SII, STI, FSI) requires a specifically tailored technique.
Approximations of CII can be distinguished into permutation-based approximation (SII and STI) and kernel-based approximation (FSI).
Both extend on existing methods for the SV, namely permutation sampling \cite{Castro.2009} for SII and STI, and KernelSHAP \cite{Lundberg_Lee_2017} for FSI.
For a comprehensive overview of the original SV methods, we refer to the appendix.
We now briefly discuss existing approaches, which will be used as baselines in our experiments.

\paragraph{Permutation-based (PB) Approximation for STI and SII \cite{Sundararajan_Dhamdhere_Agarwal_2020,Tsai_Yeh_Ravikumar_2022}.}
The permutation-based (PB) approximation computes estimates of SII and STI based on a representation of uniformly sampled random permutations $\pi \sim \text{unif}(\mathfrak S_\fset)$, where $\mathfrak S_\fset$ is the set of all permutations, i.e. the set of all ordered sequences of the elements in $\fset$. Then,
\begin{align*}
    I^{\text{SII}}(S) = \mathbb{E}_{\pi \sim \text{unif}(\mathfrak S_\fset)}\left[\mathbf{1}(S \in \pi)\delta_S^\nu\left(u_S^-(\pi_k)\right)\right] \text{ and } I^{\text{STI}}(S) = \mathbb{E}_{\pi \sim \text{unif}(\mathfrak S_\fset)}\left[\delta_S^\nu\left(u_S^-(\pi)\right)\right].
\end{align*}
Here, $u_S^-(\pi)$ refers to the set of indices in $\pi$ preceding the first occurrence of any element of $S$ in $\pi$ and $S \in \pi$ is fulfilled, if all elements of $S$ appear as a consecutive sequence in $\pi$.
The estimators for SII and STI then compute an approximation by Monte Carlo integration by sampling $\pi \sim \text{unif}(\mathfrak S_\fset)$.

\paragraph{Kernel-based (KB) Approximation for FSI \cite{Tsai_Yeh_Ravikumar_2022,Covert_Lee_2021}.}
Kernel-based (KB) approximation estimates FSI based on the representation of $I$ as a solution to a constrained weighted least square problem
\begin{align}\label{eq::fsi-lsq}
     I^{{\text{FSI}}} =  \argmin_{\beta \in \mathbb{R}^{\fnum_{s_0}}}\mathbb{E}_{T \sim p(T)}[(\nu(T)-\sum_{\substack{S \in \mathcal S_{s_0}\\S \subseteq T}}{\beta(S)})^2] \text{ s.t. } \sum_{S \in \mathcal S_{s_0}} \beta(S)= \nu(\fset) \text{ and } \beta(\emptyset) = \nu(\emptyset),
\end{align}
where $p(T) \propto \mu(t)$, is a probability distribution over $\mathcal T_1$ and $d_{s_0} := \vert \mathcal S_{s_0} \vert$.
KB approximation for FSI estimates the expectation using Monte Carlo integration by sampling from $p(T)$ and solves the approximated least-squares problem explicitly, similar to KernelSHAP \cite{Lundberg_Lee_2017,Covert_Lee_2021,Tsai_Yeh_Ravikumar_2022}.
For more details and pseudo code, we refer to the appendix.

\subsection{Unbiased KernelSHAP (U-KSH) for the SV}
U-KSH constitutes a variant of KernelSHAP (KSH) \cite{Lundberg_Lee_2017}, which relies on KB approximation for the SV.
In contrast to KSH, U-KSH is theoretically well understood and it was shown that the estimator is unbiased and consistent \cite{Covert_Lee_2021}.
U-KSH finds an exact solution to (\ref{eq::fsi-lsq}) with $s_0=1$ as
\begin{equation*}
    I^{\text{SV}} = A^{-1}\left(b - \mathbf{1}\frac{\mathbf{1}^T A^{-1} b - \nu_0(\mathbf{1})}{\mathbf{1}^T A^{-1} \mathbf{1}} \right) \text{ where } A := \mathbb{E}[ZZ^T], b = \mathbb{E}[Z \nu_0(Z)] \text{ and } p(Z) \propto \mu(t).
\end{equation*}
U-KSH then approximates this solution using Monte Carlo integration.
\begin{definition}[Unbiased KernelSHAP (U-KSH) \cite{Covert_Lee_2021}]
Given $T_1,\dots,T_K \sim p(T) \propto \mu(t)$ with binary representation $Z_1,\dots,Z_K \in \{0,1\}^\fnum$, U-KSH is defined as
\begin{equation*}
    \hat I^{\text{SV}}_{U} := A^{-1}\left(\hat b - \mathbf{1}\frac{\mathbf{1}^T A^{-1} \hat b - \nu_0(\mathbf{1})}{\mathbf{1}^T A^{-1} \mathbf{1}} \right), \text{ where } \hat b := \frac{1}{K} \sum_{k=1}^K Z_k\nu_0(Z_k).
\end{equation*}
\end{definition}

The main idea of U-KSH is that $A$ can be computed explicitly independent of $\nu$ and only $b$ has to be estimated \cite{Covert_Lee_2021}.
By linking U-KSH to our method (Theorem \ref{thm::u_ksh}), we will show that $\hat I^{\text{SV}}_U$ can be greatly simplified to a weighted sum.

\section{SHAP-IQ: Unified Approximation of any-order CII}\label{sec::shap-iq}
So far, there exists no unified approximation technique for the general CII.
In particular, it is unknown if existing approximation techniques, such PB and KB, generalize to other indices \cite{Sundararajan_Dhamdhere_Agarwal_2020,Tsai_Yeh_Ravikumar_2022,DBLP:journals/dam/DingLCC08}.
Furthermore, PB approximation for SII and STI is very inefficient as each update of all estimates requires a significant number of model evaluations.
KB  approximation for FSI efficiently computes estimates, where one model evaluation can be used to update all interaction scores.
It is, however, impossible to compute only a selection of interaction estimates and theoretical results for the estimator are difficult to establish.
In the following, we introduce SHAP-IQ (Section \ref{sec::shap-iq-method}), a unified sampling-based approximation method that can be applied to \emph{any CII}.
SHAP-IQ is based on a Monte Carlo estimate of a novel representation of the CII and well-known statistical results are applicable.
In the special case of SV, we find a novel representation of the SV and show that SHAP-IQ is linked to U-KSH (Section~\ref{sec::SV_U_KSH}).
SHAP-IQ therefore greatly reduces the computational complexity of U-KSH.
We further show (Section~\ref{sec::s-efficiency}), that the sum of interaction estimates of SHAP-IQ remains constant and therefore maintains the efficiency property for STI and SII.
Interestingly, for FSI this property does not hold, which should be investigated in future research.
All proofs can be found in the appendix.

\subsection{SHAPley Interaction Quantification (SHAP-IQ)}\label{sec::shap-iq-method}

A key challenge in approximating the CII efficiently is that the sum changes for every interaction subset $S$.
We thus first establish a novel representation of the CII.
Based on this representation, we construct SHAP-IQ, an efficient estimator of the CII.
We show that SHAP-IQ is unbiased, consistent and provide a general approximation bound.

Our novel representation of the CII is defined as a sum over all subsets $T \subseteq \fset$. In previous works, it was shown that such a representation does exist for games with $\nu(\emptyset)=0$, if the linearity axiom is fulfilled \cite[Proposition 1]{Grabisch_Roubens_1999}. 
We now explicitly specify this representation and show that the weights, for a CII, only depend on the sizes of $T$ and the intersection $T \cap S$.\footnote{Our representation generalizes a result for SII \cite[Table 3]{Grabisch_Marichal_Roubens_2000} to functions with $\nu(\emptyset) \neq 0 $ and the class of CIIs.}

\begin{theorem}\label{thm::si}
It holds $I^{m}(S) =  \sum_{T \subseteq \fset}\nu_0(T) \gamma_s^m(t,\vert T \cap S \vert)$ with $\gamma_s^m(t,k) :=(-1)^{s-k}m_s(t-k)$.
\end{theorem}
Theorem \ref{thm::si} yields a novel representation of the CII, where the model evaluations $\nu_0(T)$ are independent of $S$.
This allows to utilize every model evaluation to compute all CII scores simultaneously by properly weighting with $\gamma^m_s$.
Notably, our representation relies on $\nu_0$ instead of $\nu$, which constitutes an important choice for approximation, on which we elaborate in the appendix.

To approximate $I$, we introduce a \emph{sampling order} $k_0 \geq s_0$, for which we split the sum in Theorem \ref{thm::si} to subsets with $T \in \mathcal T_{k_0}$ and $T \notin \mathcal T_{k_0}$ and rewrite
\begin{equation*}
    I^m(S) = c_{k_0}(S) + \mathbb{E}_{T \sim p_{k_0}(T)}\left[\nu_0(T) \frac{\gamma^m_s(t,\vert T \cap S\vert)}{p_{k_0}(T)}\right] \text{ with } c_{k_0}(S) := \sum_{T \notin \mathcal T_{k_0}} \nu_0(T) \gamma^m_s(t,\vert T \cap S\vert),
\end{equation*}
where $p_{k_0}$ is over $\mathcal T_{k_0}$.
SHAP-IQ then estimates the CII by Monte Carlo integration.

\begin{definition}[SHAP-IQ]\label{def::shapx}
The \emph{Shapley Interaction Quantification (SHAP-IQ) of order $k_0$} with $K$ samples is
\begin{align*}
    \hat I^m_{k_0}(S) :=  c_{k_0}(S) + \frac 1 K \cdot \sum_{k=1}^K \nu_0(T_k) \frac{\gamma^m_s(t_k,\vert T_k \cap S \vert)}{p_{k_0}(T_k)} \text{ with } T_1,\dots,T_K \sim p_{k_0}(T).
\end{align*}
\end{definition}

SHAP-IQ is outlined in the appendix and we establish the following important theoretical guarantees.

\begin{theorem}\label{thm::unbiased_consistent}
SHAP-IQ is unbiased, $\mathbb{E}\left[\hat I^m_{k_0}(S)\right] = I^m(S)$, and consistent, $\hat I^m_{k_0}(S) \overset{K \to \infty}{\to} I^m(S)$.
With $\sigma^2(S) := \mathbb{V}\left[\nu_0(T)\frac{\gamma^m_s(\vert T \vert,\vert T \cap S\vert)}{p_{k_0}(T)} \right]$ and $\epsilon>0$, it holds $\mathbb{P}(\vert \hat I^m_{k_0}(S) - I^m(S) \vert > \epsilon) \leq \frac 1 K \frac{\sigma^2(S)}{ \epsilon^2}$.
\end{theorem}

SHAP-IQ provides efficient estimates of all CII scores with important theoretical guarantees.
The sample variance $\hat \sigma^2$ can further be used for statistical analysis of the estimates.

\paragraph{Finding the sampling order $k_0$ and distribution $p_{k_0}$.}
In line with KSH and U-KSH \cite{Lundberg_Lee_2017,Covert_Lee_2021}, we find $k_0$ in an iterative procedure, outlined in the appendix.
We consider \emph{sampling weights} $q(t) \geq 0$ for $0\leq t \leq \fnum$ that grow symmetrically towards the center and consider a distribution $p_{k_0}(T) \propto q(t)$.
Given a budget $M$ and initial $k_0 = 0$, we consider $I^m(S) = \mathbb{E}_{T \sim p_{k_0}(T)}\left[\nu_0(T) \frac{\gamma^m_s(t,\vert T \cap S\vert)}{p_{k_0}(T)}\right]$ and iteratively increase $k_0$, if for a subset $T$ of size $k_0$ and $\fnum - k_0$, the condition $M \cdot p_{k_0}(T) \geq 1$ is fulfilled.
The budget is then decreased by the number of subsets of that size, i.e. $2 \binom{\fnum}{k_0}$.
This essentially verifies iteratively, if the expected number of subsets exceeds the total number of subsets.
For more details and possible choices of \emph{sampling weights} $q$, we refer to the appendix.

\paragraph{Computational Complexity.}
In contrast to PB approximations, SHAP-IQ allows to iteratively update \emph{all} interaction estimates with \emph{one single} model evaluation for any-order interactions.
The weights $\gamma^m_s(t,k)$ used for the updates can be efficiently precomputed.
The updating process can be implemented efficiently using Welford's algorithm \cite{Welford_1962}, where estimates have to be maintained for all interactions sets, i.e. $d_{s_0}$ in total.
In contrast to KB approximation, which requires to solve a weighted least square optimization problem with $d_{s_0}$ variables, the computational effort per interaction increases linearly for SHAP-IQ.
Furthermore, SHAP-IQ even allows to update selected interaction estimates, whereas, for instance, KB approximation for FSI requires to estimate all interactions.
For more details on the implementation and computational complexity of the baseline methods, we refer to the appendix.

\subsection{SHAP-IQ for the Shapley Value}\label{sec::SV_U_KSH}
In this section, we show that SHAP-IQ, in the special case of single feature subsets $s_0=1$, yields novel insights into the SV.
Furthermore, SHAP-IQ corresponds to U-KSH and greatly simplifies its calculation.
Utilizing Theorem~\ref{thm::si}, we find a novel representation of the SV for every feature $i \in \fset$. 
\begin{theorem}\label{thm::SV_representation}
With $c_1(i) = \frac{\nu_0(\fset)}{\fnum}$ the SV is $I^{\text{SV}}(i) = c_1(i) +  \sum_{T \in \mathcal T_1}\nu_0(T)\mu(t)\left[ \mathbf{1}(i \in T) - \frac{t}{\fnum}\right]$.
\end{theorem}
SHAP-IQ admits a similar form (see appendix) and corresponds to U-KSH $\hat I^{\text{SV}}_{U}$.

\begin{theorem}[SHAP-IQ simplifies U-KSH]\label{thm::u_ksh}
For $p(T) \propto \mu(t)$ it holds that $\hat I^{\text{SV}}_U=\hat I^m_1$.
\end{theorem}
Theorem~\ref{thm::u_ksh} implies that the U-KSH estimator can be computed using the SHAP-IQ estimator, which greatly simplifies the calculation to a weighted sum.
The main idea of the proof relies on the observation that not only $A$ can be explicitly computed, but also $A^{-1}$, cf. the appendix.

\subsection{The Sum of Interaction Scores and SHAP-IQ Efficiency}\label{sec::s-efficiency}
In this section, we are interested in the sum of CII scores, which we link to a property of SHAP-IQ estimates to maintain the efficiency axiom.
By Theorem~\ref{thm::si}, we have $\sum_{S \in \mathcal S_{s_0}} I^m(S) = \sum_{T \subseteq \fset} \nu_0(T) \sum_{S \in \mathcal S_{s_0}} \gamma^m_s(t,\vert T \cap S \vert)$.
For the SV, by Theorem~\ref{thm::SV_representation}, this sum is zero for every $T \in \mathcal T_1$.
For higher order CIIs, we introduce the following definition.
\begin{definition}\label{def::s_efficiency}
A CII is \emph{s-efficient}, if $\sum_{S \subseteq \fset, \vert S \vert = s_0} \gamma^m_s(t,\vert T \cap S \vert) = 0$ for every $T \in \mathcal T_{s_0}$.
\end{definition}
\begin{theorem}\label{thm::s-efficiency}
SII and STI are s-efficient.
In particular, SHAP-IQ estimates maintain efficiency for n-SII and STI.
\end{theorem}

Further, if a CII is s-efficient, then the sum of SHAP-IQ estimates remains constant.
Although we did not provide a rigorous statement, it is easy to validate numerically that FSI is not s-efficient.
This finding suggests that there are conceptional differences between these indices, that should be further investigated in future work.
Using s-efficiency it is also possible to find an explicit formula for the sum of interaction scores for SII, which we give in the appendix.

\section{Experiments}
\label{sec:Experiments}

We conduct multiple experiments to illustrate the approximation quality of SHAP-IQ compared to current baseline approaches.\footnote{All code and implementations for conducting the experiments can be found at \url{https://github.com/FFmgll/shapiq}. Running the experiments required a computational cost of approximately $2\,000$ CPU hours. For more details we refer to the appendix.}
We showcase SHAP-IQ estimates on any-order SII (n-SII), on top-order STI and FSI.
For each interaction index, we use its specific approximation method as a baseline.
For SII and STI, we use the PB approximation and for FSI the KB approximation, further described in the appendix.
We then compute n-SII based on the estimated SII values.
We compare the baseline methods with SHAP-IQ using $p(T) \propto \mu(t)$.
For each iteration we evaluate the approximation quality with different budgets up to a maximum budget of $2^{14}$ model evaluations.
To account for variation, we randomly evaluate the approximation method on 50 randomly chosen instances, further described below.
To quantify the approximation quality, we compute multiple evaluation metrics for each interaction order: mean-squared error (MSE), MSE for the top-K interactions (MSE@K) and the ratio (precision) of estimated top-K interactions (Prec@K).
The top-K interactions are determined in regards to their absolute value.

\paragraph{Models.}
For a language model (LM), we use a fine-tuned version of the DistilBERT transformer architecture \cite{Sanh.2019} on movie review sentences from the original \emph{IMDB} dataset \cite{Maas.2011,Lhoest_Datasets_A_Community_2021} for sentiment analysis, i.e. $\nu$ has values in $[-1,1]$.
In the LM, for a given sentence, different feature coalitions are computed by masking absent features in a tokenized sentence.
The implementation is based on the \emph{transformers} API \cite{Wolf_Transformers_State-of-the-Art_Natural_2020}.
We randomly sample $50$ reviews of length $\fnum=14$ and explain each model prediction.
For an image classification model (ICM), we use ResNet18 \cite{resnet18} pre-trained on ImageNet \cite{ImageNet} as provided by \emph{torch} \cite{torch.2017}.
We randomly sample $50$ images and explain the prediction of the corresponding true class.
To obtain the prediction of different coalitions, we pre-compute super-pixels with SLIC \cite{SLIC,scikit-image} to obtain a function on $\fnum=14$ features and apply mean imputation on absent features.
For a high-dimensional synthetic model with $\fnum=30$, we use a \emph{sum of unanimity model} (SOUM) $\nu(T):= \sum_{n=1}^N a_n \mathbf{1}(Q_n \subseteq T)$, where $N=50$ interaction subsets $Q_1,\dots,Q_N \subseteq \fset$ are chosen uniformly from all subset sizes and $a_1,\dots,a_N \in \mathbb{R}$ are generated uniformly $a_n \sim \text{unif}([0,1])$.
Note that the SOUM could also be viewed as an extension of the induced subgraph game \cite{DBLP:journals/mor/DengP94} for a hypergraph with edges of different order.
We randomly generate $50$ instances of such SOUMs.

\paragraph{Ground-Truth (GT) Values.}
For the LM and the ICM we compute the ground-truth (GT) values explicitly using the representation from Theorem \ref{thm::si}.
For the high-dimensional SOUM it is impossible to compute the GT values naively.
However, due to the linearity of the CII and the simple structure of a SOUM, we can compute the exact GT values of for any CII efficiently, cf. the appendix.

\subsection{Approximation of any-order SII and n-SII scores using SHAP-IQ}
In this experiment, we apply SHAP-IQ on SII and compute estimates for the LM and the ICM up to order $s=4$.
We then compare the estimates with the baseline using the GT values for each order.
The results are shown in Figure~\ref{fig:exp-sii}.
We display the MSE for the LM (left) and the Prec@10 for the LM (middle) and ICM (right).
We further compute the n-SII estimates by aggregating the SII estimates with $s_0=4$ and visualize positive and negative interactions on single individuals as proposed in \cite{pmlr-v206-bordt23a}.
Thereby, interactions are distributed equally among each participating feature, which was justified in \cite[Theorem 6]{pmlr-v206-bordt23a}.
This representation amplifies the variance of our sampling-based estimator.
We thus also present SHAP-IQ without sampling, i.e. $c_{k_0}$.
The results are shown in Figure~\ref{fig:exp-n-sii} (left) and from left to right: GT values, SHAP-IQ, SHAP-IQ without sampling and baseline.
Lastly, we illustrate the n-SII scores estimates for $s_0=3$ of a movie review excerpt classified by the LM (right), where the interactions (``is'',``not''), (``not'',``bad''), and (``'ll'',``love'',``this'') yield a highly positive score.

The results show that SHAP-IQ outperforms the baseline methods across different models and metrics.
For the n-SII visualization, we conclude that the SHAP-IQ estimator without sampling is preferable, which yields more accurate results than SHAP-IQ and the baseline methods.
In general, SHAP-IQ without sampling performs surprisingly strong, and we encourage further work in this direction.

\begin{figure}[t!]
    \centering
    \begin{minipage}[c]{0.31\columnwidth}
        \includegraphics[width=\textwidth]{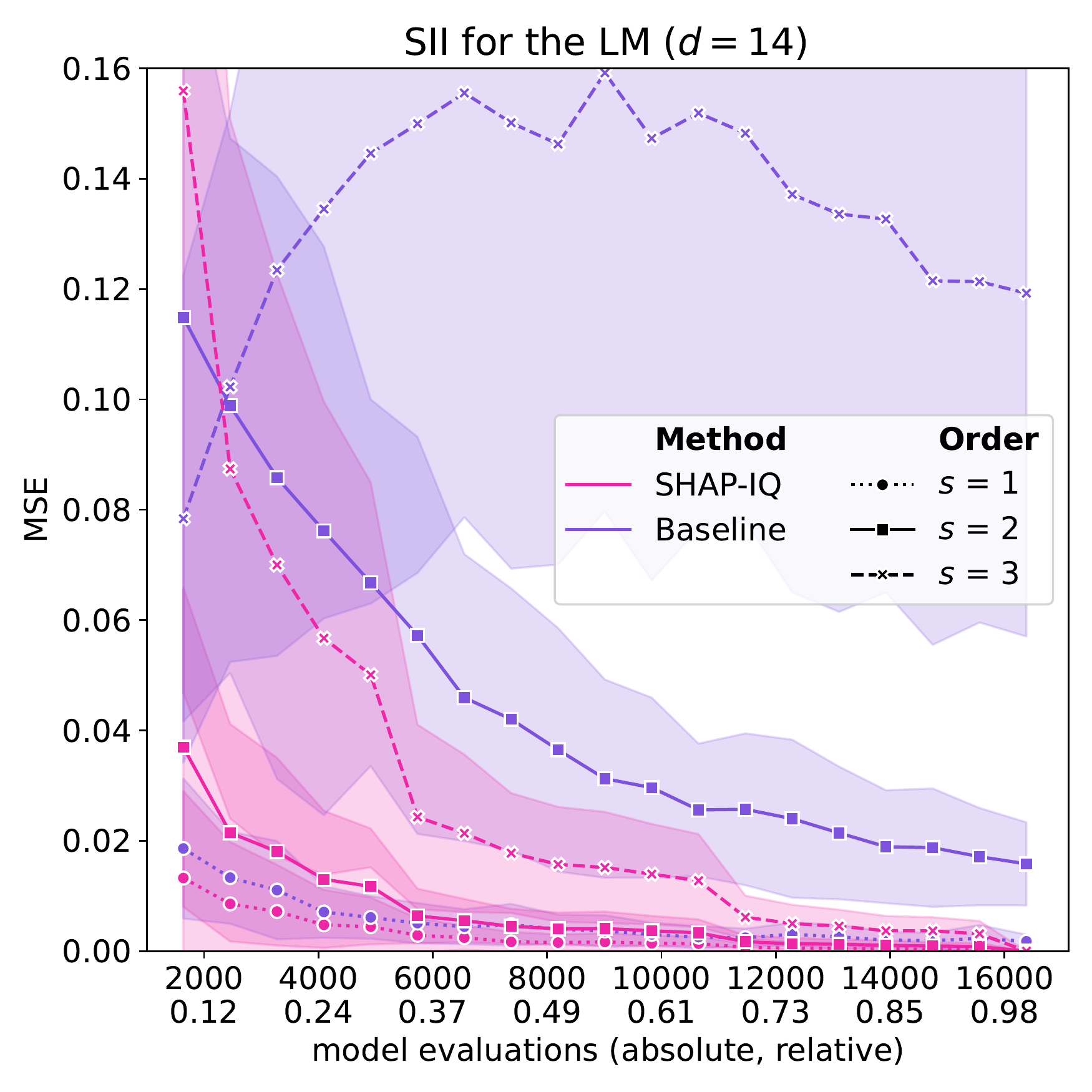}
    \end{minipage}
    \hfill
    \begin{minipage}[c]{0.31\columnwidth}
        \includegraphics[width=\textwidth]{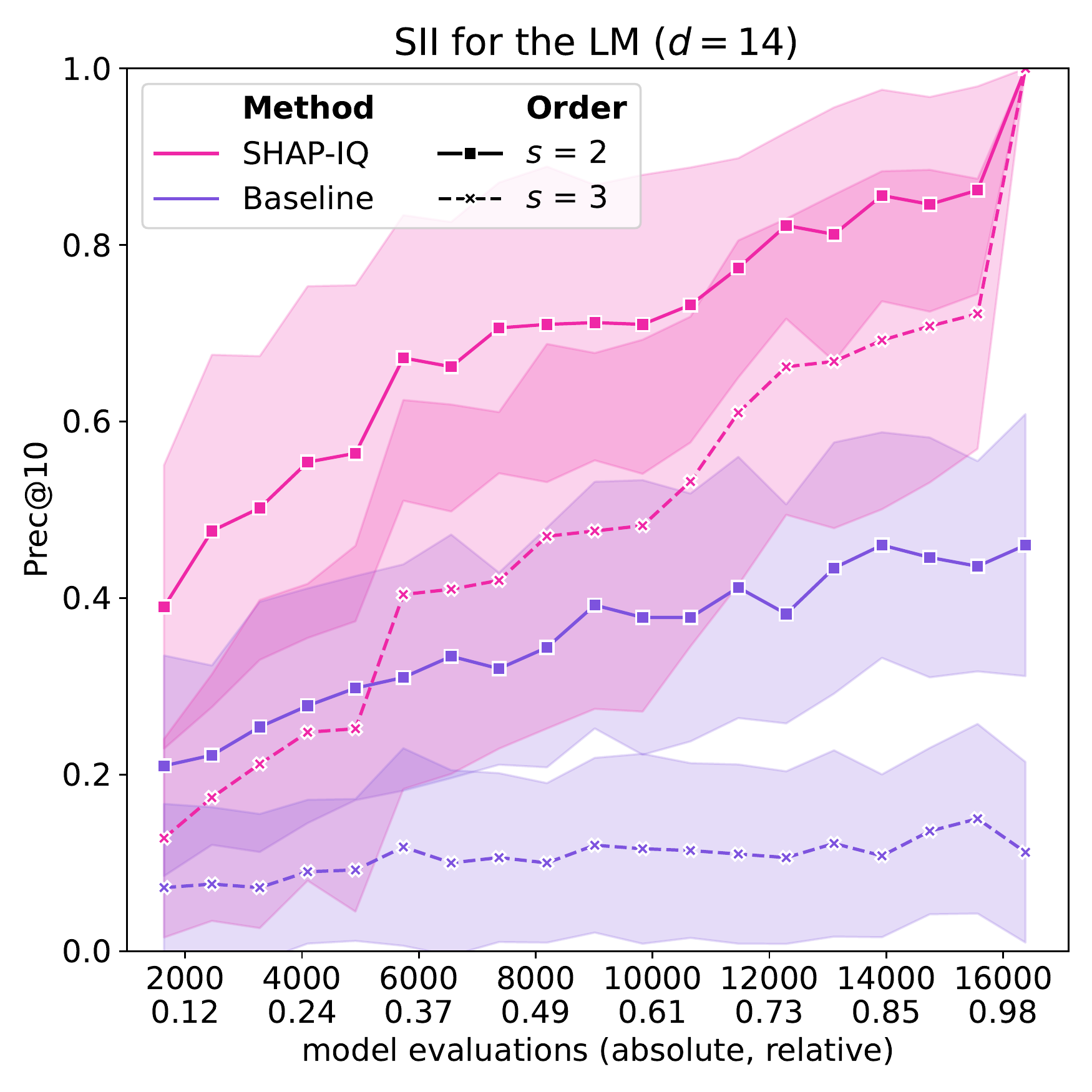}
    \end{minipage}
    \hfill
    \begin{minipage}[c]{0.31\columnwidth}
        \includegraphics[width=\textwidth]{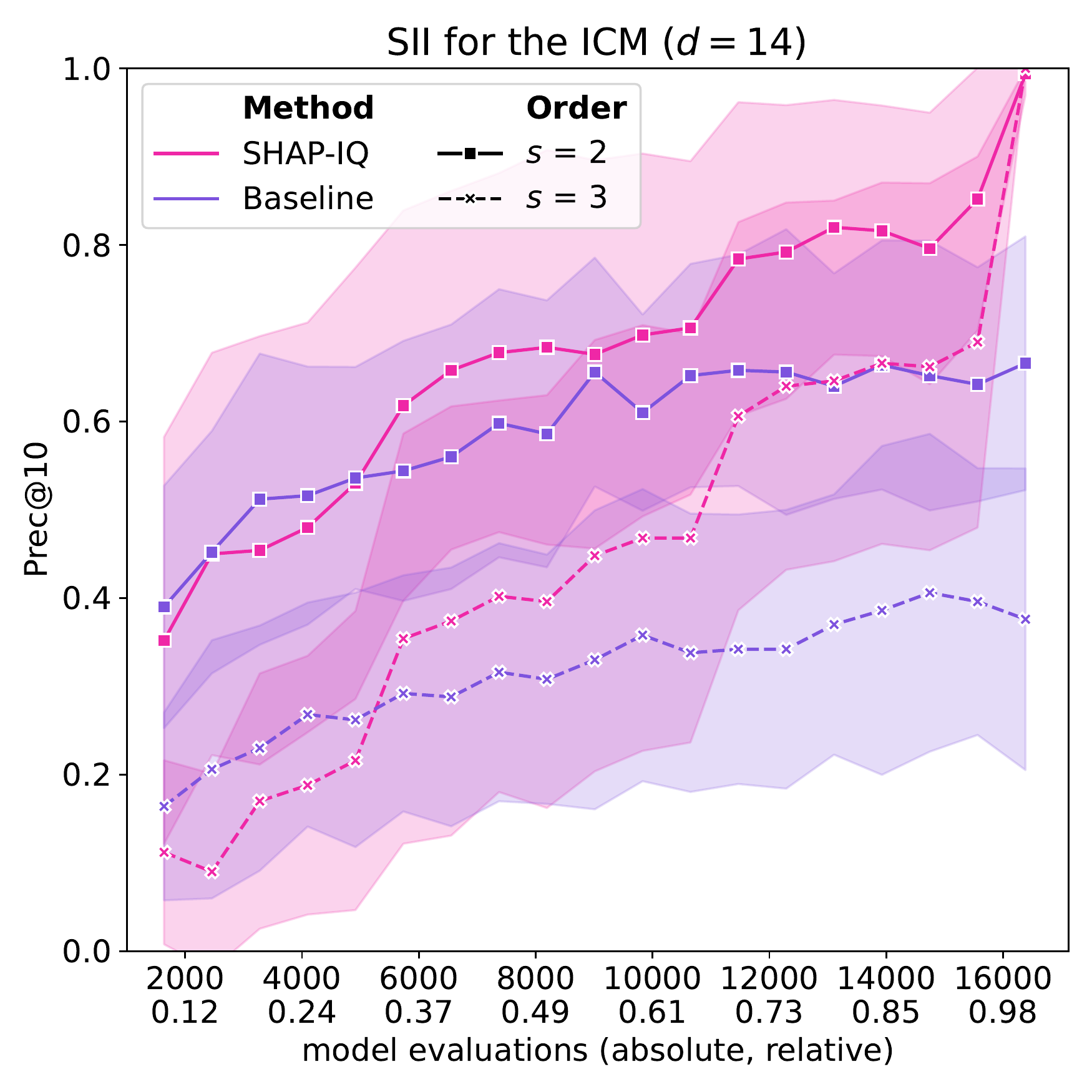}
    \end{minipage}
    \caption{Approximation quality of SHAP-IQ and the baseline for orders $s=1,2,3$ of SII measured by MSE for the LM (left) and Prec@10 for orders $s=2,3$ for the LM (middle) and ICM (right).}
    \label{fig:exp-sii}
\end{figure}

\begin{figure}[t]    
    \begin{minipage}[c]{0.485\columnwidth}
        \includegraphics[width=\textwidth]{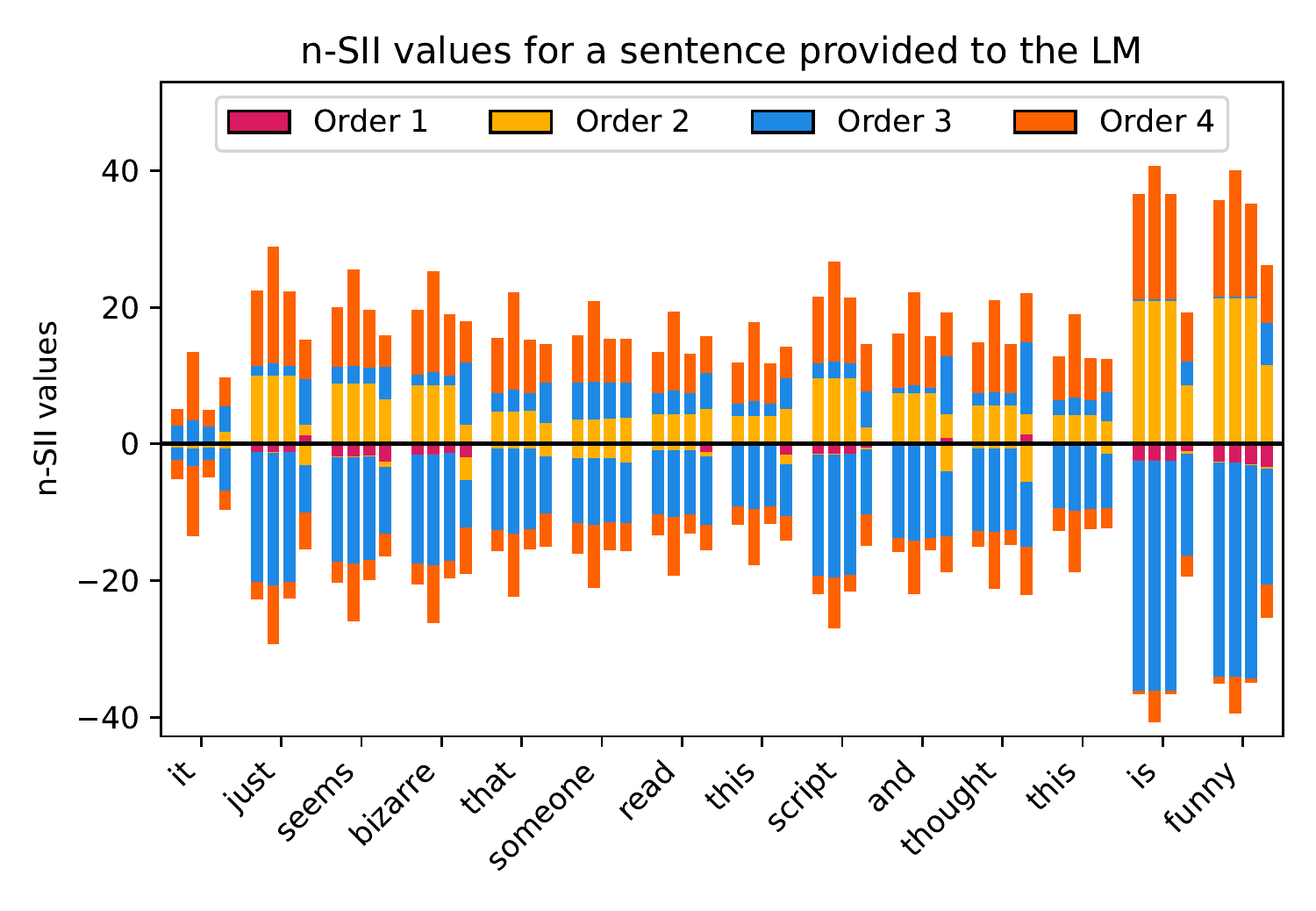}
    \end{minipage}
    \hfill
    \begin{minipage}[c]{0.485\columnwidth}
        \includegraphics[width=\textwidth]{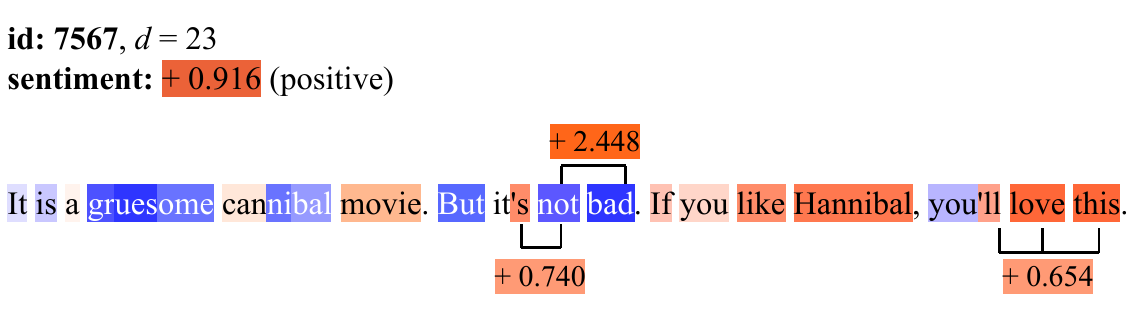}
    \end{minipage}
    \caption{Visualization of n-SII and $s_0=4$ \protect{\cite{pmlr-v206-bordt23a}} (left) with (from left to right): GT, SHAP-IQ, SHAP-IQ without sampling, and baseline. Estimated n-SII scores ($s_0=3$) for a movie review (right).}
    \label{fig:exp-n-sii}
\end{figure}

\subsection{Approximation of different CIIs using SHAP-IQ}
\label{sec_experiments_different_ciis}

In this experiment, we apply SHAP-IQ on different CIIs, namely SII, STI and FSI.
We compute top-order interactions for $s_0=3$ and compare the results with the baselines.
Our results are shown in Figure~\ref{fig:exp-cii-comparison} (left) and further experiments and results can be found in the appendix.
For the LM, SHAP-IQ clearly outperforms the baseline for SII and STI.
For FSI, SHAP-IQ is outperformed by the KB approximation of the baseline.
As SII and STI rely on PB approximation, our results indicate that KB approximation is more effective than PB approximation for this setting, which is in line with the strong performance of KernelSHAP \cite{Lundberg_Lee_2017} for the SV.
However, SHAP-IQ, in contrast to KB approximation, provides a solid mathematical foundation with theoretical guarantees and we now consider a high-dimensional synthetic game, where SHAP-IQ outperforms all baselines.

\begin{figure}[t]
    \centering
    \begin{minipage}[t]{0.31\columnwidth}
        \includegraphics[width=\textwidth]{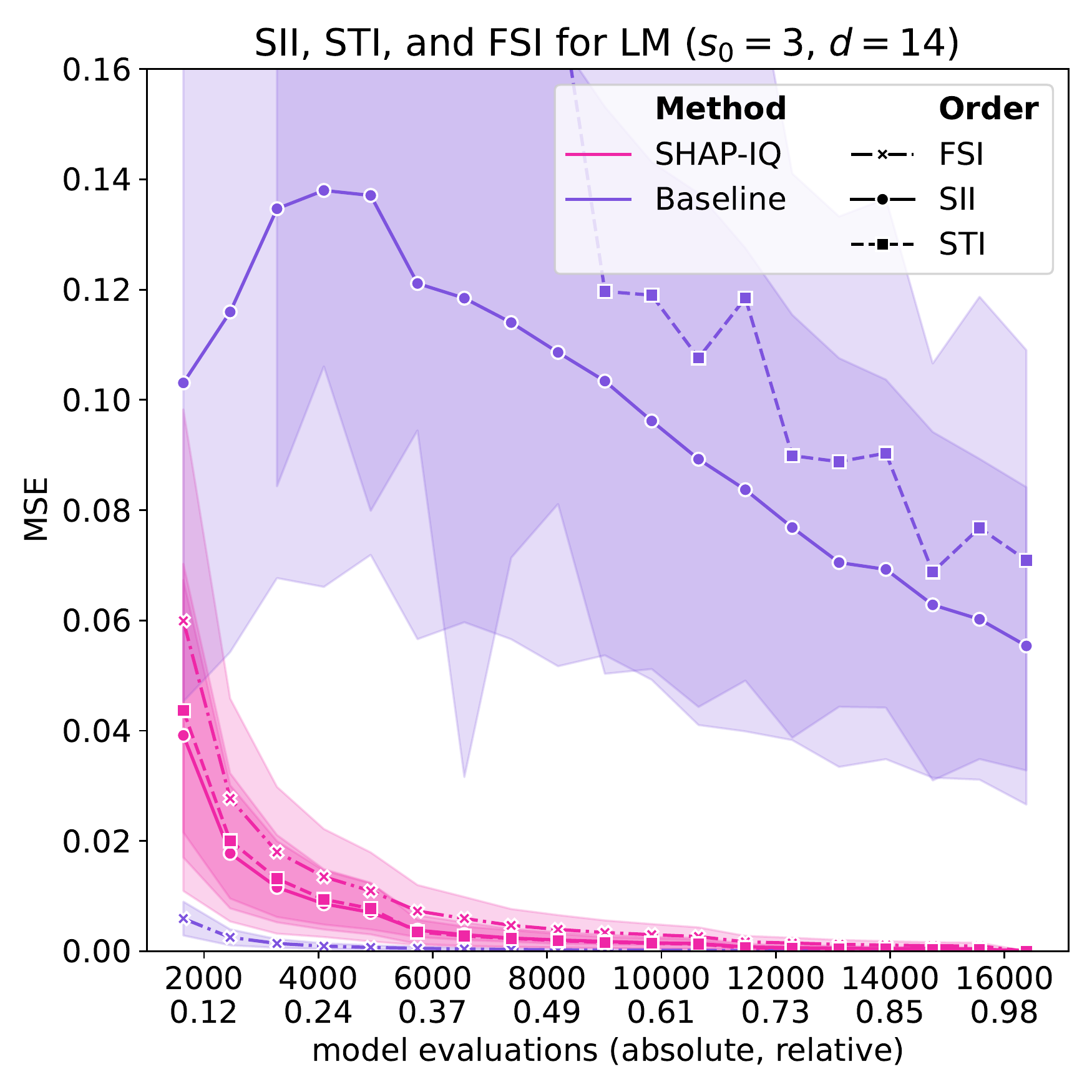}
    \end{minipage}
    \hfill
    \begin{minipage}[t]{0.31\columnwidth}
        \includegraphics[width=\textwidth]{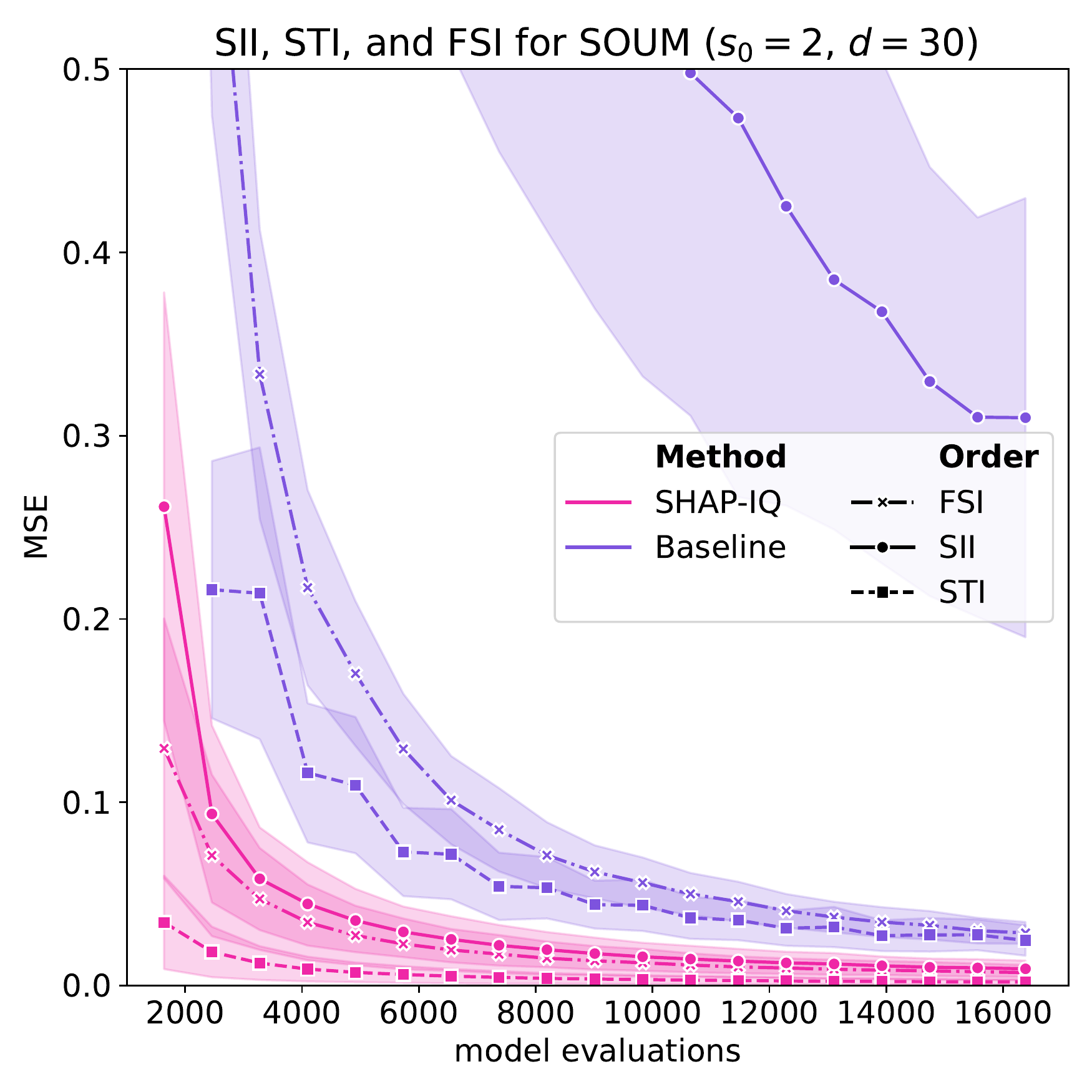}
    \end{minipage}
    \hfill
    \begin{minipage}[t]{0.31\columnwidth}
        \includegraphics[width=\textwidth]{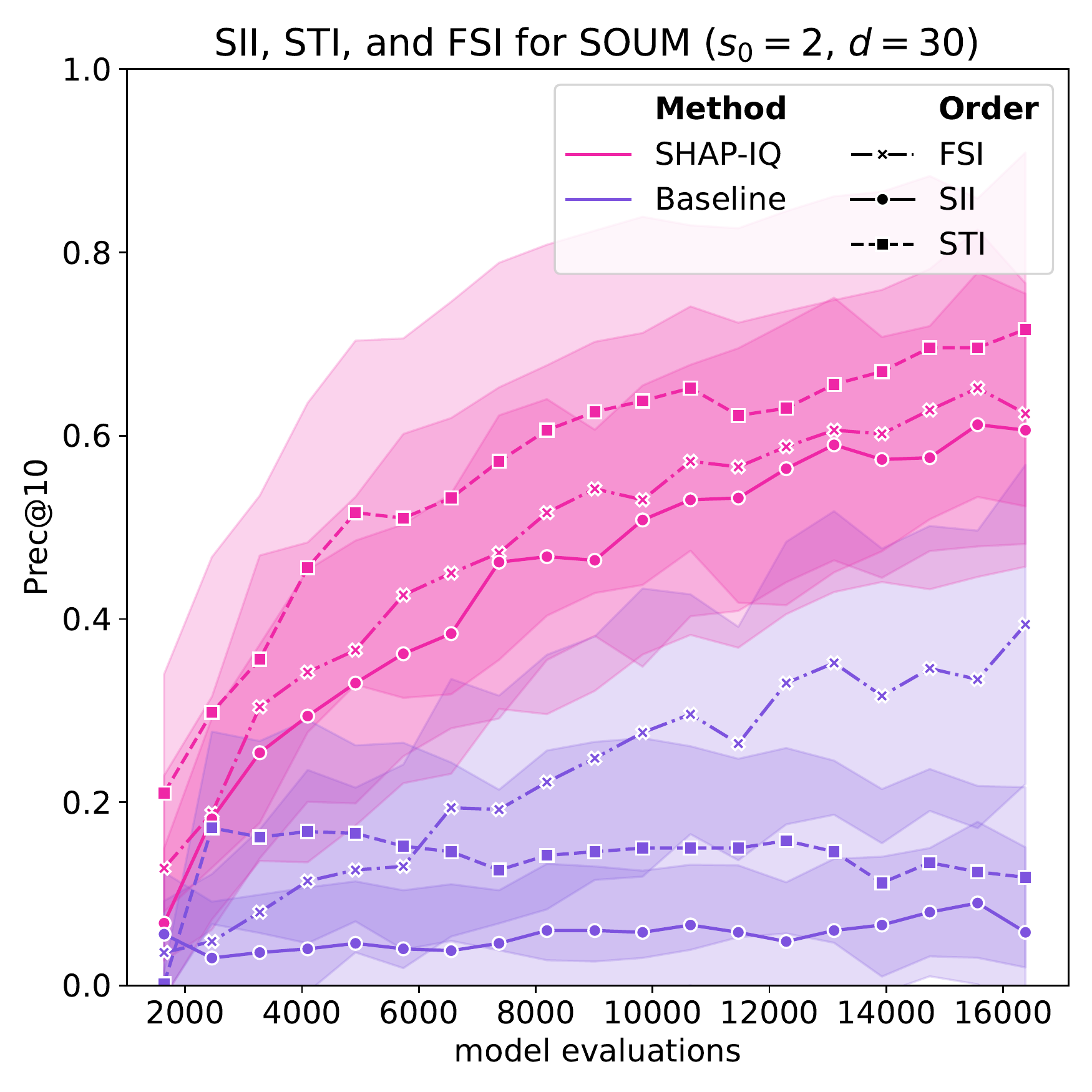}
    \end{minipage}
    \caption{Approximation quality for top-order interactions of SII, STI, and FSI of the LM with $s_0 = 3$ (left) and the SOUM with $s_0=2$ (middle and right).}
    \label{fig:exp-cii-comparison}
\end{figure}

\paragraph{SHAP-IQ on high-dimensional synthetic models.}
For the SOUM we compute the average and standard deviation of each evaluation metric for SHAP-IQ and the baselines for pairwise interactions ($s_0=2$) of each index.
Our results are shown in Figure~\ref{fig:exp-cii-comparison} (middle and right) and further experiments and results can be found in the appendix.
SHAP-IQ outperforms all baseline methods in this setting, in particular, the KB approximation of FSI that performed strongly in ML context. 
The experiment highlights that there exists no approximation method that performs universally best.

\paragraph{Runtime Analysis.}
The runtime of SHAP-IQ is affected by different parameters.
The computation of the sampling order $k_0$ is a constant time operation given a number of features (cf. Algorithm 2 in the appendix).
While the pre-computation of the weights ($m_s$) scales linearly with the number of features, the additional computational burden is negligible as it does not depend on $\nu_0$.
The main computational cost stems from the model evaluations (access to the value function $\nu_0$), which is bounded by a model's inference time.
To illustrate the runtime performance, we compare SHAP-IQ with the baseline methods on the LM using different number of model evaluations $K$.
Figure~\ref{exp:runtime} displays the runtime of SHAP-IQ and the corresponding baseline approaches, including all pre-computations.
With increasing $K$ the runtime complexity scales linearly, but the overhead of SHAP-IQ remains low. 
Note that the difference in STI can be attributed to less than $K$ model evaluations, which is required to maintain efficiency, cf. lines 15-16 in Algorithm 6 of the appendix.

\begin{wrapfigure}{r}{0.43\textwidth}
\begin{center}
    \includegraphics[width=0.43\textwidth]{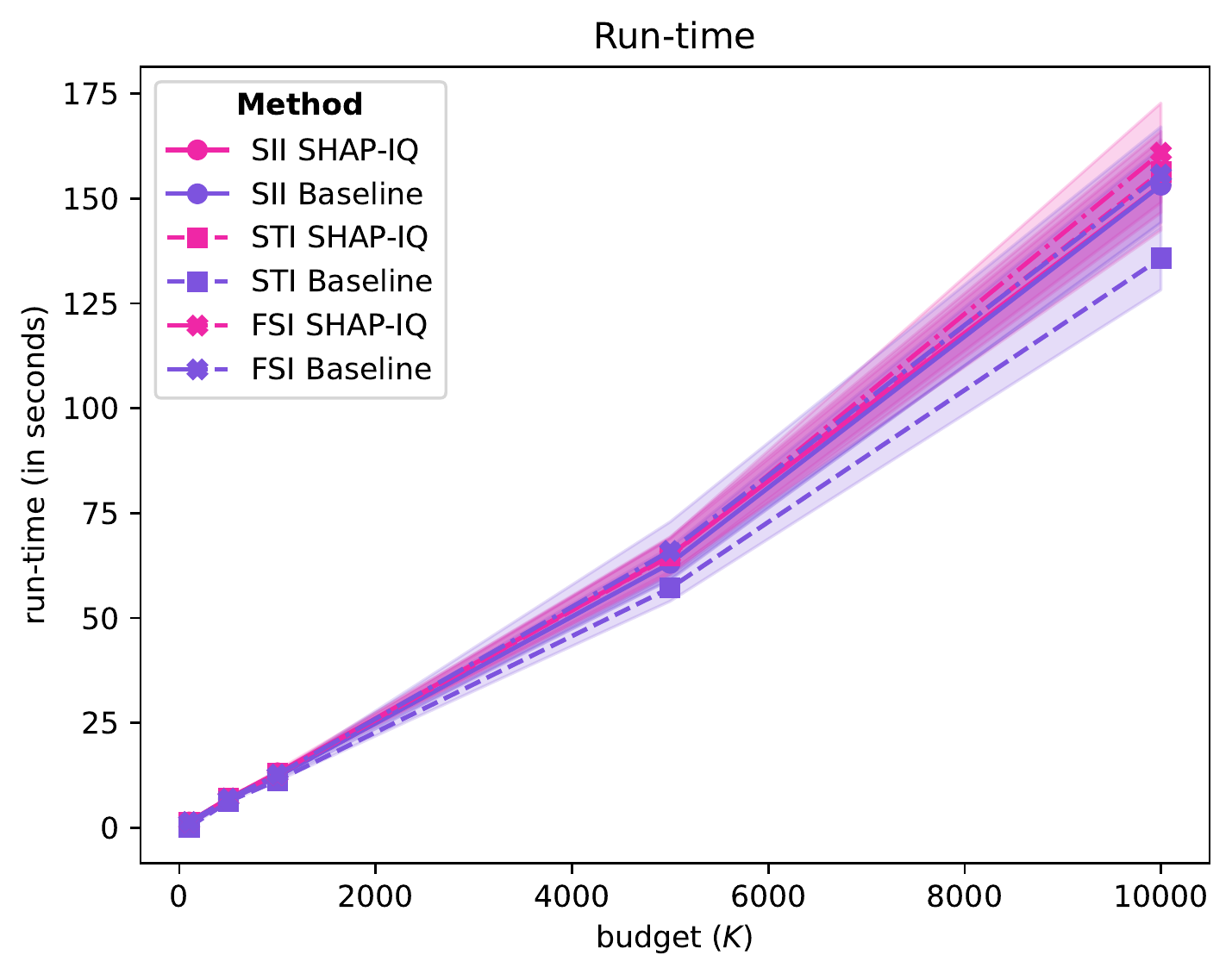}
\end{center}
\caption{For each interaction index (SII: solid, STI: dashed, FSI: dashdotted), we compare the run-time (in seconds) of SHAP-IQ (pink) compared to baseline estimators (violet) of one instantiation of the LM (16 words) over five independent runs at different levels of K.}
\label{exp:runtime}
\end{wrapfigure}

\section{Limitations}
We presented SHAP-IQ, a unified approximation algorithm for any-order CIIs with important theoretical guarantees.
SHAP-IQ relies on the specific structure in terms of discrete derivatives from Definition~\ref{def_SI}.
This representation exists for every interaction indices that fulfills the linearity, symmetry and dummy axiom \cite[Proposition 5]{Grabisch_Roubens_1999}.
However, for FSI, which has been defined as the solution to the weighted least square problem, a closed-form representation for lower-order interactions in terms of discrete derivatives remains difficult to establish \cite[Lemma 70]{Tsai_Yeh_Ravikumar_2022}.
This limits the applicability of SHAP-IQ to top-order interactions of FSI, for which this representation is given in \cite[Theorem 19]{Tsai_Yeh_Ravikumar_2022}.
The FSI baseline performs strongly in ML context, in line with empirical findings for KernelSHAP \cite{Lundberg_Lee_2017} for the SV.
However, the estimator is theoretically not well understood \cite{Covert_Lee_2021} and we have shown that it is not universally best.
Moreover, for other CIIs, such as SII and STI, it is unlikely \cite{DBLP:journals/dam/DingLCC08} that such an explicit form in terms of a weighted least square problem can be found, which limits the applicability of KB approximation to FSI.
SHAP-IQ outperforms the baselines of SII and STI by a large margin, is generally applicable and supported by a solid mathematical foundation.

\section{Conclusion}
How to extend the SV to interactions is an open research question.
In this work, we considered CIIs, a broad class of interaction indices, which covers all currently proposed indices, as well as all indices that fulfill the linearity, symmetry and dummy axiom.
We established a novel representation of the CII, which we used to introduce SHAP-IQ, an efficient sampling-based approximation algorithm that is unbiased and consistent.
For the special case of SV, SHAP-IQ can be seen as a generalization of U-KSH \cite{Covert_Lee_2021} and greatly simplifies its calculation as well as providing a novel representation of the SV.
Furthermore, for n-SII and STI, SHAP-IQ maintains the efficiency condition, which is a direct consequence of a specific property, which we coin s-efficiency for CIIs.
We applied SHAP-IQ in multiple experimental settings to compute any-order interactions of SII and n-SII, where SHAP-IQ consistently outperforms the baseline method and showcased the applicability of feature interaction scores to understand black-box language and image classification models.
SHAP-IQ further benefits from a solid statistical foundation, which can be leveraged to improve the approximation quality.

\paragraph{Future work.} Applying SHAP-IQ to real-world applications, such as NLP tasks \cite{Tsang.2020} and genomics \cite{Wright.2016,DBLP:journals/bmcbi/LiuZG19}, could yield valuable insights.
However, the exponentially increasing number of interactions requires human-centered post-processing to enhance interpretability for practitioners and ML engineers, e.g. through automated dialogue systems \cite{Slack_Krishna_Lakkaraju_Singh_2023}.
Further, it would be beneficial to discover the statistical capabilities of SHAP-IQ to provide confidence bounds or approximate interaction scores sequentially.
Beyond model-agnostic approximation, model-specific variants could substantially reduce computational complexity.
For instance, it is likely that ideas of TreeSHAP \cite{Lundberg.2020} for tree-based models can be extended to Shapley-based interactions.

\section*{Acknowledgements}
We sincerely thank the anonymous reviewers for their work and helpful comments. 
We gratefully acknowledge funding by the Deutsche Forschungsgemeinschaft (DFG, German Research Foundation): TRR 318/1 2021 – 438445824.
Patrick Kolpaczki is supported by the research training group Dataninja (Trustworthy AI for Seamless Problem Solving: Next Generation Intelligence Joins Robust Data Analysis) funded by the German federal state of North Rhine-Westphalia.
\newpage
\bibliographystyle{apalike}
\bibliography{references}

\begin{thebibliography}{}

\bibitem[Achanta et~al., 2012]{SLIC}
Achanta, R., Shaji, A., Smith, K., Lucchi, A., Fua, P., and S{\"{u}}sstrunk, S.
  (2012).
\newblock Slic superpixels compared to state-of-the-art superpixel methods.
\newblock {\em {IEEE} Transactions on Pattern Analysis and Machine
  Intelligence}, 34(11):2274--2282.

\bibitem[Adadi and Berrada, 2018]{Adadi.2022}
Adadi, A. and Berrada, M. (2018).
\newblock Peeking inside the black-box: A survey on explainable artificial
  intelligence (xai).
\newblock {\em {IEEE} Access}, 6:52138--52160.

\bibitem[Bordt and von Luxburg, 2023]{pmlr-v206-bordt23a}
Bordt, S. and von Luxburg, U. (2023).
\newblock From shapley values to generalized additive models and back.
\newblock In {\em Proceedings of The 26th International Conference on
  Artificial Intelligence and Statistics, (AISTATS 2023)}, volume 206 of {\em
  Proceedings of Machine Learning Research}, pages 709--745. PMLR.

\bibitem[Castro et~al., 2017]{Castro.2017}
Castro, J., G{\'{o}}mez, D., Molina, E., and Tejada, J. (2017).
\newblock Improving polynomial estimation of the shapley value by stratified
  random sampling with optimum allocation.
\newblock {\em Computers {\&} Operations Research}, 82:180--188.

\bibitem[Castro et~al., 2009]{Castro.2009}
Castro, J., G{\'{o}}mez, D., and Tejada, J. (2009).
\newblock Polynomial calculation of the shapley value based on sampling.
\newblock {\em Computers {\&} Operations Research}, 36(5):1726--1730.

\bibitem[Charnes et~al., 1988]{Charnes_Golany_Keane_Rousseau_1988}
Charnes, A., Golany, B., Keane, M., and Rousseau, J. (1988).
\newblock {\em Extremal Principle Solutions of Games in Characteristic Function
  Form: Core, Chebychev and Shapley Value Generalizations}, volume~11 of {\em
  Advanced Studies in Theoretical and Applied Econometrics}, page 123–133.
\newblock Springer Netherlands.

\bibitem[Chen et~al., 2023]{chen2023algorithms}
Chen, H., Covert, I.~C., Lundberg, S.~M., and Lee, S.-I. (2023).
\newblock Algorithms to estimate shapley value feature attributions.
\newblock {\em Nature Machine Intelligence}, 5(6):590--601.

\bibitem[Covert and Lee, 2021]{Covert_Lee_2021}
Covert, I. and Lee, S.-I. (2021).
\newblock Improving kernelshap: Practical shapley value estimation using linear
  regression.
\newblock In {\em Proceedings of The 24th International Conference on
  Artificial Intelligence and Statistics, (AISTATS 2021)}, volume 130 of {\em
  Proceedings of Machine Learning Research}, page 3457–3465. PMLR.

\bibitem[Covert et~al., 2021]{Covert_Lundberg_Lee_2021}
Covert, I., Lundberg, S.~M., and Lee, S. (2021).
\newblock Explaining by removing: A unified framework for model explanation.
\newblock {\em Journal of Machine Learning Research}, 22:209:1--209:90.

\bibitem[Cui et~al., 2020]{Cui.2020}
Cui, T., Marttinen, P., and Kaski, S. (2020).
\newblock Learning global pairwise interactions with bayesian neural networks.
\newblock In {\em 24th European Conference on Artificial Intelligence, (PAIS
  2020)}, volume 325 of {\em Frontiers in Artificial Intelligence and
  Applications}, pages 1087--1094. {IOS} Press.

\bibitem[Deng et~al., 2009]{ImageNet}
Deng, J., Dong, W., Socher, R., Li, L., Li, K., and Fei{-}Fei, L. (2009).
\newblock Imagenet: A large-scale hierarchical image database.
\newblock In {\em 2009 IEEE Computer Society Conference on Computer Vision and
  Pattern Recognition, (CVPR 2009)}, pages 248--255. {IEEE} Computer Society.

\bibitem[Deng and Papadimitriou, 1994]{DBLP:journals/mor/DengP94}
Deng, X. and Papadimitriou, C.~H. (1994).
\newblock On the complexity of cooperative solution concepts.
\newblock {\em Mathematics of Operations Research}, 19(2):257--266.

\bibitem[Ding et~al., 2008]{DBLP:journals/dam/DingLCC08}
Ding, G., Lax, R.~F., Chen, J., and Chen, P.~P. (2008).
\newblock Formulas for approximating pseudo-boolean random variables.
\newblock {\em Discrete Applied Mathematics}, 156(10):1581--1597.

\bibitem[Fujimoto et~al., 2006]{DBLP:journals/geb/FujimotoKM06}
Fujimoto, K., Kojadinovic, I., and Marichal, J. (2006).
\newblock Axiomatic characterizations of probabilistic and
  cardinal-probabilistic interaction indices.
\newblock {\em Games Econ. Behav.}, 55(1):72--99.

\bibitem[Ghorbani and Zou, 2019]{ghorbani.2019}
Ghorbani, A. and Zou, J.~Y. (2019).
\newblock Data shapley: Equitable valuation of data for machine learning.
\newblock In {\em Proceedings of the 36th International Conference on Machine
  Learning, (ICML 2019)}, volume~97 of {\em Proceedings of Machine Learning
  Research}, pages 2242--2251. PMLR.

\bibitem[Grabisch et~al., 2000]{Grabisch_Marichal_Roubens_2000}
Grabisch, M., Marichal, J., and Roubens, M. (2000).
\newblock Equivalent representations of set functions.
\newblock {\em Mathematics of Operations Research}, 25(2):157--178.

\bibitem[Grabisch and Roubens, 1999]{Grabisch_Roubens_1999}
Grabisch, M. and Roubens, M. (1999).
\newblock An axiomatic approach to the concept of interaction among players in
  cooperative games.
\newblock {\em International Journal of Game Theory}, 28(4):547--565.

\bibitem[Hammer and Holzman, 1992]{Hammer_Holzman_1992}
Hammer, P.~L. and Holzman, R. (1992).
\newblock Approximations of pseudo-boolean functions; applications to game
  theory.
\newblock {\em {ZOR} Mathematical Methods of Operations Research}, 36(1):3--21.

\bibitem[He et~al., 2016]{resnet18}
He, K., Zhang, X., Ren, S., and Sun, J. (2016).
\newblock Deep residual learning for image recognition.
\newblock In {\em 2016 IEEE Conference on Computer Vision and Pattern
  Recognition, (CVPR 2016)}, pages 770--778. {IEEE} Computer Society.

\bibitem[Janizek et~al., 2021]{Janizek.2021}
Janizek, J.~D., Sturmfels, P., and Lee, S. (2021).
\newblock Explaining explanations: Axiomatic feature interactions for deep
  networks.
\newblock {\em Journal of Machine Learning Research}, 22:104:1--104:54.

\bibitem[Jia et~al., 2019]{Jia.2019}
Jia, R., Dao, D., Wang, B., Hubis, F.~A., Hynes, N., G{\"{u}}rel, N.~M., Li,
  B., Zhang, C., Song, D., and Spanos, C.~J. (2019).
\newblock Towards efficient data valuation based on the shapley value.
\newblock In {\em The 22nd International Conference on Artificial Intelligence
  and Statistics, (AISTATS 2019)}, volume~89 of {\em Proceedings of Machine
  Learning Research}, pages 1167--1176. PMLR.

\bibitem[Lhoest et~al., 2021]{Lhoest_Datasets_A_Community_2021}
Lhoest, Q., Villanova~del Moral, A., von Platen, P., Wolf, T., Šaško, M.,
  Jernite, Y., Thakur, A., Tunstall, L., Patil, S., Drame, M., Chaumond, J.,
  Plu, J., Davison, J., Brandeis, S., Sanh, V., Le~Scao, T., Canwen~Xu, K.,
  Patry, N., Liu, S., McMillan-Major, A., Schmid, P., Gugger, S., Raw, N.,
  Lesage, S., Lozhkov, A., Carrigan, M., Matussière, T., von Werra, L., Debut,
  L., Bekman, S., and Delangue, C. (2021).
\newblock Datasets: A community library for natural language processing.
\newblock In {\em Proceedings of the 2021 Conference on Empirical Methods in
  Natural Language Processing: System Demonstrations, (EMNLP 2021)}, pages
  175--184. Association for Computational Linguistics.

\bibitem[Liu et~al., 2019]{DBLP:journals/bmcbi/LiuZG19}
Liu, G., Zeng, H., and Gifford, D.~K. (2019).
\newblock Visualizing complex feature interactions and feature sharing in
  genomic deep neural networks.
\newblock {\em {BMC} Bioinformatics}, 20(1):401:1--401:14.

\bibitem[Lou et~al., 2013]{Lou.2013}
Lou, Y., Caruana, R., Gehrke, J., and Hooker, G. (2013).
\newblock Accurate intelligible models with pairwise interactions.
\newblock In {\em The 19th ACM SIGKDD International Conference on Knowledge
  Discovery and Data Mining, (KDD 2013)}, pages 623--631. ACM.

\bibitem[Lundberg et~al., 2020]{Lundberg.2020}
Lundberg, S.~M., Erion, G.~G., Chen, H., DeGrave, A.~J., Prutkin, J.~M., Nair,
  B., Katz, R., Himmelfarb, J., Bansal, N., and Lee, S. (2020).
\newblock From local explanations to global understanding with explainable ai
  for trees.
\newblock {\em Nature Machine Intelligence}, 2(1):56--67.

\bibitem[Lundberg and Lee, 2017]{Lundberg_Lee_2017}
Lundberg, S.~M. and Lee, S. (2017).
\newblock A unified approach to interpreting model predictions.
\newblock In {\em Advances in Neural Information Processing Systems 30: Annual
  Conference on Neural Information Processing Systems 2017, (NeurIPS 2017)},
  pages 4765--4774. Curran Associates, Inc.

\bibitem[Maas et~al., 2011]{Maas.2011}
Maas, A.~L., Daly, R.~E., Pham, P.~T., Huang, D., Ng, A.~Y., and Potts, C.
  (2011).
\newblock Learning word vectors for sentiment analysis.
\newblock In {\em Proceedings of the 49th Annual Meeting of the Association for
  Computational Linguistics: Human Language Technologies, (HLT 2011)}, pages
  142--150. Association for Computational Linguistics.

\bibitem[Molnar et~al., 2019]{Molnar.2019}
Molnar, C., Casalicchio, G., and Bischl, B. (2019).
\newblock Quantifying model complexity via functional decomposition for better
  post-hoc interpretability.
\newblock In {\em Machine Learning and Knowledge Discovery in Databases -
  International Workshops of ECML PKDD 2019}, pages 193--204.

\bibitem[Murdoch et~al., 2018]{Murdoch.2018}
Murdoch, W.~J., Liu, P.~J., and Yu, B. (2018).
\newblock Beyond word importance: Contextual decomposition to extract
  interactions from lstms.
\newblock In {\em 6th International Conference on Learning Representations,
  (ICLR 2018)}.

\bibitem[Paszke et~al., 2019]{torch.2017}
Paszke, A., Gross, S., Massa, F., Lerer, A., Bradbury, J., Chanan, G., Killeen,
  T., Lin, Z., Gimelshein, N., Antiga, L., Desmaison, A., Kopf, A., Yang, E.,
  DeVito, Z., Raison, M., Tejani, A., Chilamkurthy, S., Steiner, B., Fang, L.,
  Bai, J., and Chintala, S. (2019).
\newblock Pytorch: An imperative style, high-performance deep learning library.
\newblock In {\em Advances in Neural Information Processing Systems 32: Annual
  Conference on Neural Information Processing Systems 2019, (NeurIPS 2019)},
  pages 8024--8035. Curran Associates, Inc.

\bibitem[Rabitti and Borgonovo, 2019]{Rabitti_Borgonovo_2019}
Rabitti, G. and Borgonovo, E. (2019).
\newblock A shapley-owen index for interaction quantification.
\newblock {\em SIAM/ASA Journal on Uncertainty Quantification},
  7(3):1060–1075.

\bibitem[Ruiz et~al., 1996]{Ruiz_Valenciano_Zarzuelo_1996}
Ruiz, L.~M., Valenciano, F., and Zarzuelo, J.~M. (1996).
\newblock The least square prenucleolus and the least square nucleolus. two
  values for tu games based on the excess vector.
\newblock {\em International Journal of Game Theory}, 25(1):113–134.

\bibitem[Ruiz et~al., 1998]{Ruiz_Valenciano_Zarzuelo_1998}
Ruiz, L.~M., Valenciano, F., and Zarzuelo, J.~M. (1998).
\newblock The family of least square values for transferable utility games.
\newblock {\em Games and Economic Behavior}, 24(1):109–130.

\bibitem[Sanh et~al., 2019]{Sanh.2019}
Sanh, V., Debut, L., Chaumond, J., and Wolf, T. (2019).
\newblock Distilbert, a distilled version of bert: smaller, faster, cheaper and
  lighter.
\newblock {\em CoRR}, abs/1910.01108.

\bibitem[Shapley, 1953]{Shapley.1953}
Shapley, L.~S. (1953).
\newblock A value for n-person games.
\newblock In {\em Contributions to the Theory of Games (AM-28), Volume II},
  pages 307--318. Princeton University Press.

\bibitem[Singh et~al., 2019]{Singh.2019}
Singh, C., Murdoch, W.~J., and Yu, B. (2019).
\newblock Hierarchical interpretations for neural network predictions.
\newblock In {\em 7th International Conference on Learning Representations,
  (ICLR 2019)}.

\bibitem[Slack et~al., 2023]{Slack_Krishna_Lakkaraju_Singh_2023}
Slack, D., Krishna, S., Lakkaraju, H., and Singh, S. (2023).
\newblock Explaining machine learning models with interactive natural language
  conversations using talktomodel.
\newblock {\em Nature Machine Intelligence}, 5(8):873–883.

\bibitem[Spivey, 2007]{SPIVEY20073130}
Spivey, M.~Z. (2007).
\newblock Combinatorial sums and finite differences.
\newblock {\em Discrete Mathematics}, 307(24):3130--3146.

\bibitem[Sundararajan et~al., 2020]{Sundararajan_Dhamdhere_Agarwal_2020}
Sundararajan, M., Dhamdhere, K., and Agarwal, A. (2020).
\newblock The shapley taylor interaction index.
\newblock In {\em Proceedings of the 37th International Conference on Machine
  Learning, (ICML 2020)}, volume 119 of {\em Proceedings of Machine Learning
  Research}, pages 9259--9268. PMLR.

\bibitem[Tsai et~al., 2023]{Tsai_Yeh_Ravikumar_2022}
Tsai, C., Yeh, C., and Ravikumar, P. (2023).
\newblock Faith-shap: The faithful shapley interaction index.
\newblock {\em Journal of Machine Learning Research}, 24(94):1--42.

\bibitem[Tsang et~al., 2020a]{Tsang.2020b}
Tsang, M., Cheng, D., Liu, H., Feng, X., Zhou, E., and Liu, Y. (2020a).
\newblock Feature interaction interpretability: A case for explaining
  ad-recommendation systems via neural interaction detection.
\newblock In {\em 8th International Conference on Learning Representations,
  (ICLR 2020)}.

\bibitem[Tsang et~al., 2018]{Tsang.2018}
Tsang, M., Cheng, D., and Liu, Y. (2018).
\newblock Detecting statistical interactions from neural network weights.
\newblock In {\em 6th International Conference on Learning Representations,
  (ICLR 2018)}.

\bibitem[Tsang et~al., 2020b]{Tsang.2020}
Tsang, M., Rambhatla, S., and Liu, Y. (2020b).
\newblock How does this interaction affect me? interpretable attribution for
  feature interactions.
\newblock In {\em Advances in Neural Information Processing Systems 31: Annual
  Conference on Neural Information Processing Systems, (NeurIPS 2020)}, pages
  6147--6159. Curran Associates, Inc.

\bibitem[van~der Walt et~al., 2014]{scikit-image}
van~der Walt, S., {S}ch\"onberger, J.~L., {Nunez-Iglesias}, J., {B}oulogne, F.,
  {W}arner, J.~D., {Y}ager, N., {G}ouillart, E., {Y}u, T., and the scikit-image
  contributors (2014).
\newblock scikit-image: image processing in python.
\newblock {\em PeerJ}, 2(6):e453.

\bibitem[Welford, 1962]{Welford_1962}
Welford, B.~P. (1962).
\newblock Note on a method for calculating corrected sums of squares and
  products.
\newblock {\em Technometrics}, 4(3):419–420.

\bibitem[Williamson and Feng, 2020]{Williamson_Feng_2020}
Williamson, B.~D. and Feng, J. (2020).
\newblock Efficient nonparametric statistical inference on population feature
  importance using shapley values.
\newblock In {\em Proceedings of the 37th International Conference on Machine
  Learning, (ICML 2020)}, volume 119 of {\em Proceedings of Machine Learning
  Research}, pages 10282--10291. PMLR.

\bibitem[Wolf et~al., 2020]{Wolf_Transformers_State-of-the-Art_Natural_2020}
Wolf, T., Debut, L., Sanh, V., Chaumond, J., Delangue, C., Moi, A., Cistac, P.,
  Rault, T., Louf, R., Funtowicz, M., Davison, J., Shleifer, S., von Platen,
  P., Ma, C., Jernite, Y., Plu, J., Xu, C., Scao, T.~L., Gugger, S., Drame, M.,
  Lhoest, Q., and Rush, A.~M. (2020).
\newblock Transformers: State-of-the-art natural language processing.
\newblock In {\em Proceedings of the 2020 Conference on Empirical Methods in
  Natural Language Processing: System Demonstrations, (EMNLP 2020)}, pages
  38--45. Association for Computational Linguistics.

\bibitem[Wright et~al., 2016]{Wright.2016}
Wright, M.~N., Ziegler, A., and K{\"{o}}nig, I.~R. (2016).
\newblock Do little interactions get lost in dark random forests?
\newblock {\em {BMC} Bioinformatics}, 17(145).

\bibitem[Yeh et~al., 2018]{yeh2018representer}
Yeh, C., Kim, J.~S., Yen, I.~E., and Ravikumar, P. (2018).
\newblock Representer point selection for explaining deep neural networks.
\newblock In {\em Advances in Neural Information Processing Systems 31: Annual
  Conference on Neural Information Processing Systems, (NeurIPS 2018)}, pages
  9311--9321. Curran Associates, Inc.

\end{thebibliography}

\newpage
\appendix
\onecolumn

\section*{Appendix of ``SHAP-IQ: Unified Approximation of any-order Shapley Interactions''}

\subsection*{Organisation of the Appendix}

We provide further theoretical and experimental results for SHAP-IQ.
The appendix is organized as follows: In Appendix~\ref{appx::si_theory}, we formally introduce specific CIIs, such as SII, n-SII, STI and FSI and their axiomatic foundation and theoretical results and provide a novel theoretical result for the sum of SII scores.
In Appendix~\ref{appx::proofs}, we provide all proofs of theoretical results from the main paper.
In Appendix~\ref{appx::algo-shap-iq}, we give further insights into the implementation of SHAP-IQ.
In Appendix~\ref{appx::experiments}, we provide further experimental results, information about the used models, explicit formulas for the SOUM interaction scores, and pseudo-code and information on the computational complexity of baseline implementations.
In Appendix~\ref{appx::sv_theory}, we provide further theoretical results for the special case of SV, in particular the explicit form of the covariance matrix from \cite{Covert_Lee_2021} and a simplified representation, similar to Theorem \ref{thm::SV_representation}, of SHAP-IQ in this case.
In Appendix~\ref{appx::approx_sv}, we describe the approximation methods ApproShapley \cite{Castro.2009} and KernelSHAP \cite{Lundberg_Lee_2017} for the SV on which our baseline methods are built.

\startcontents[sections]
\printcontents[sections]{l}{1}{\setcounter{tocdepth}{2}}

\clearpage
\section{Shapley-based Interaction Indices and further theoretical Results}\label{appx::si_theory}
In this section, we review the axiomatic structures of n-SII, SII, STI and FSI.
We further provide an additional result for the sum of SII interaction scores.

\subsection{Shapley-based Interaction Indices}
We consider $\mathcal G$ as the set of all games $\nu: \mathcal P(\fset) \to \mathbb{R}$.
In the following, we formalize the axioms for the Shapley interaction indices $I_\nu: \mathcal P(\fset) \to \mathbb{R}$

\begin{definition}[Linearity Axiom \cite{Grabisch_Roubens_1999}]
$I_\nu$ is \emph{linear}, if for any two games $\nu_1,\nu_2 \in \mathcal G$ and any $S \subseteq \fset$, it holds $I_{\nu_1+\nu_2}(S) = I_{\nu_1}(S) + I_{\nu_2}(S)$.
\end{definition}

\begin{definition}[Dummy Axiom \cite{Grabisch_Roubens_1999}]
For a dummy player $i \in \fset$ for a game $\nu \in \mathcal G$, i.e. constant contribution $c(i)$ added to any coalition $\nu(T \cup \{i\} = \nu(T) + c(i)$, then for every $T \subseteq \fset \setminus \{i\}$, the \emph{dummy axiom} requires
$I_\nu(S \cup \{ i\})=0$ for any $S \subseteq \fset \setminus \{i\}$.
That is, a dummy player has no interaction with any coalition.
\end{definition}

\begin{definition}[Symmetry Axiom \cite{Grabisch_Roubens_1999}]
$I_\nu$ is said to fulfill the \emph{symmetry axiom}, if for any permutation $\pi$ on $\fset$ it holds $I_\nu(S) = I_{\pi\nu}(\pi S)$, where $\pi\nu(\pi S) :=\nu(S)$ and $\pi S := \{\pi(i): i \in S\}$ changes the ordering of the players.
\end{definition}

\begin{definition}[Recursive Axiom \cite{Grabisch_Roubens_1999}]
$I_\nu$ fulfills the \emph{recursive axiom}, if for any $S\subseteq \fset$ with $\vert S \vert > 1$ and any game $\nu \in \mathcal G$
\begin{equation*}
I(S) = I_{\nu_{[S]}}([S]) - \sum_{K \subsetneq S, K \neq \emptyset} I_{\nu^{\fset \setminus K}}(S \setminus K),
\end{equation*}
where $\nu_{[S]}$ is the game, where all players in $S$ is considered as one player, and $\nu^{\fset \setminus K}$ is a game defined on the subset of players $\fset \setminus K$.
\end{definition}
The recursive axiom defines higher order interactions using lower order interaction.
For pairwise interactions it can be stated as $I_\nu(ij) = I_{\nu_{[ij]}}([ij]) - I_{\nu^{\fset \setminus \{j\}}}(i) - I_{\nu^{\fset \setminus \{i\}}}(j)$, i.e. the pairwise interaction is the difference of the value for the reduced player $[ij]$ and the individual player values for the reduced game.

\begin{definition}[Shapley Interaction Index \cite{Grabisch_Roubens_1999}]
The Shapley interaction index (SII) is the unique interaction index that satisfies the linearity, dummy, symmetry and recursive axiom, where the values for $\vert S \vert = 1$ correspond to the Shapley value.
It can be represented as a CII as
\begin{equation*}
    I^{\text{SII}}(S) := \sum_{T \subseteq \fset \setminus S} m^{\text{SII}}_{s}(t) \delta^\nu_S(T) \text{ and } m^{\text{SII}}_{s}(t) := \frac{(\fnum-t-s)!t!}{(\fnum-s+1)!}.\end{equation*}
\end{definition}

In contrast to the SV, the SII does not yield an efficiency property, which is desirable in ML context.
The efficiency axiom was therefore introduced for interactions.
The following axioms rely on a maximum interaction order $s_0$, where the values of the interaction indices change for different maximum interaction orders.

\begin{definition}[Efficiency Axiom \cite{Sundararajan_Dhamdhere_Agarwal_2020}]
For all $\nu \in \mathcal G$, it holds $\sum_{S \in \mathcal S_{s_0}} I_\nu(S) = \nu(\fset) - \nu(\emptyset)$.
\end{definition}

The efficiency axiom is an extension of the SV efficiency axiom and requires that all interaction scores up to the maximum order $s_0$ sum up to $\nu(\fset) -\nu(\emptyset)$.
For SII there exists a unique recursive aggregation, such that the efficiency axiom is fulfilled.
This novel interaction index is referred to as n-Shapley Values (n-SII) \cite{pmlr-v206-bordt23a}.

\begin{definition}[n-Shapley Values (n-SII) \cite{pmlr-v206-bordt23a}]
Given $s_0$ the n-Shapley Values (n-SII) are defined as
\begin{align*}
    I_{s_0}^{\text{n-SII}}(S) := \begin{cases}
        I^{\text{SII}}(S), &\text{ for } \vert S \vert = s_0
        \\
        I^{\text{n-SII}}_{s_0-1}(S) + B_{\fnum-\vert S \vert} \sum_{\substack{K \subseteq \fset \setminus S \\ k + s = \fnum}} I^{\text{SII}}(S \cup K), &\text{ for } \vert S \vert < s_0.
    \end{cases}
\end{align*}
\end{definition}

Besides n-SII, there have been two axiomatic extensions to CIIs that directly require the efficiency axiom together with the linearity, symmetry and dummy axiom.
However, unlike for the SV, it is not sufficient to require efficiency for a unique interaction index.
The Shapley Taylor Interaction Index (STI) further specifies the interaction distribution axiom, which then yields a unique index \cite{Sundararajan_Dhamdhere_Agarwal_2020}.

\begin{definition}[Interaction Distribution Axiom \cite{Sundararajan_Dhamdhere_Agarwal_2020}]
For an \emph{interaction function} $\nu_T$ parametrized by $T \subseteq \fset$ it holds $\nu_T(S) = 0$, if $T \subsetneq S$ and a constant value $\nu_T(S) = c$ for $T \subseteq S$ and $c \in \mathbb{R}$.
The \emph{interaction distribution axiom} requires for all $S \subseteq \fset$ with $S \subsetneq T$ and $s < s_0$ that $I_{\nu_T}(S) = 0$.
\end{definition}

\begin{definition}[The Shapley Taylor Interaction Index \cite{Sundararajan_Dhamdhere_Agarwal_2020}]
The Shapley Taylor interaction index (STI) is the unique interaction index that satisfies the linearity, dummy, symmetry, efficiency and interaction distribution axiom.
\end{definition}

The STI yields a unique interaction index by introducing the interaction distribution axiom, which favors the maximum order interactions as discussed in \cite{Tsai_Yeh_Ravikumar_2022}.
It was thus argued that instead the representation of the interaction index as a solution to a weighted least square problem is preferable \cite{Tsai_Yeh_Ravikumar_2022}, which yields the Faith-Interaction Index.

\begin{definition}[Faith-Interaction Index \cite{Tsai_Yeh_Ravikumar_2022}]
$I$ is called a \emph{Faith-interaction index} if it can be expressed as
\begin{align}
\begin{aligned}
I = \argmin_{\beta \in \mathbb{R}^{\fnum_{s_0}}} \sum_{T \subseteq \fset: \mu(T)<\infty}\mu(T)\left(\nu(T)-\sum_{\substack{T \subseteq S\\ t \leq s_0}}{\beta(S)}\right)^2
\\
\text{s.t. } \nu(T) = \sum_{S \subseteq T, t \leq s_0} \beta(S), \forall T: \mu(T) = \infty.
\end{aligned}
\end{align}
\end{definition}

\begin{definition}[Faithful Shapley Interaction Index \cite{Tsai_Yeh_Ravikumar_2022}]
The Faithful Shapley interaction index (FSI) is the unique faith-interaction index that satisfies the linearity, dummy, symmetry and efficiency axiom.
\end{definition}

While n-SII, SII, STI and FSI offer differnt ways of characterizing an interaction index, it was shown in \cite{Grabisch_Roubens_1999} that every interaction index satisfying the linearity, symmetry and dummy axiom admits a CII representation.

\begin{proposition}\label{appx::SI_weights}
Every interaction index satisfying the linearity, symmetry and dummy axiom can be represented as a CII
\begin{align*}
I^{m}(S) := \sum_{T \subseteq \fset\setminus S} m_{s}(t) \delta^\nu_S(T).
\end{align*}
Furthermore, the weights $m_{s_0}$ for the top-order interaction indices, i.e. $s=s_0$ are defined as
\begin{align*}
    m_{s_0}^{\text{SII}}(t) &:= \frac{(\fnum-t-s_0)!t!}{(\fnum-s_0+1)!},
    \\
    m_{s_0}^{\text{STI}}(t) &:= s_0 \frac{(\fnum-t-1)!t!}{\fnum!}
    \\
    m_{s_0}^{\text{FII}}(t) &:= \frac{(2s_0-1)!}{((s_0-1)!)^2} \frac{(\fnum-t-1)!(t+s_0-1)!}{(\fnum+s_0-1)!}.
\end{align*}
The definitions for lower order interactions can be found in \cite{Grabisch_Roubens_1999} for SII, in \cite{Sundararajan_Dhamdhere_Agarwal_2020} for STI and in \cite{Tsai_Yeh_Ravikumar_2022} for FSI.
Note that for FSI a closed-form solution for lower-order interactions in terms of discrete derivatives remains unknown \cite[Lemma 70]{Tsai_Yeh_Ravikumar_2022}.
\end{proposition}

\begin{proof}
For the proofs, we refer to the corresponding paper of each index. 
The general statement is proven in \cite[Proposition 5]{Grabisch_Roubens_1999}.
\end{proof}

\subsection{Explicit Formula for Sum of SII Scores}\label{appx::explicit-sii-sum}

Using s-efficiency, it is easy to calculate the sum of SII scores, which provide an explicit representation of a formula considered in \cite[Theorem 4.2]{Rabitti_Borgonovo_2019}.
\begin{proposition}[Sum of SII Scores]\label{prop:sii_efficiency}
For SII it holds
\begin{equation*}
    \sum_{\substack{ S \subseteq \fset \\ \vert S \vert = s_0}} I^{\text{SII}}(S)  = \sum_{\substack{T\subseteq \fset \\ t < s_0}} r(t)\left[(-1)^{s_0}\nu(T)+\nu(\fset \setminus T) \right] ,
\end{equation*}
\end{proposition}
where $r(t) :=\frac{1}{s_0} \binom{\fnum-t}{s_0-t-1}$ and $m_s^{\text{SII}}(t) :=  \frac{(\fnum-t-s_0)!t!}{(\fnum-s_0+1)!}$.
\begin{proof}
For SII, we let $m(t) := m_s^{\text{SII}}(t) = \frac{1}{\fnum-s+1} \binom{\fnum-s}{t}^{-1}$ and have by Theorem \ref{thm::s-efficiency} and the definition of s-efficiency
\begin{align*}
\sum_{\substack{ S \subseteq \fset \\ \vert S \vert = s_0}} I^{m}(S) =\sum_{\substack{S \subseteq \fset \\ \vert S \vert = s_0}} c_{s_0}(S) = \sum_{\substack{T\subseteq \fset \\ t < s_0}} \nu(T) \sum_{\substack{ S \subseteq \fset \\ \vert S \vert = s_0}} \gamma^m_s(t,\vert T \cap S \vert) +  \sum_{\substack{T\subseteq \fset \\ t > \fnum-s_0}} \nu(T) \sum_{\substack{ S \subseteq \fset \\ \vert S \vert = s_0}} \gamma^m_s(t,\vert T \cap S \vert)
\end{align*}
For $t<s_0$ we have
\begin{align*}
\rho(t) := \sum_{\substack{ S \subseteq \fset \\ \vert S \vert = s_0}} \gamma^m_s(t,\vert T \cap S \vert) &= \sum_{k=0}^{t} \binom{t}{k}\binom{t-k}{s_0-k}\gamma^m_s(t,\vert T \cap S \vert) 
\\
&= \frac{1}{\fnum-s_0+1}\sum_{k=0}^{t} (-1)^{s_0-k} \binom{t}{k}\binom{\fnum-t}{s_0-k} \binom{\fnum-s_0}{t-k}^{-1}.
\end{align*}
For $t>\fnum-s$ there are at least $k_{\min}= t-(\fnum-s)$ elements in the intersection of $\vert T \cap S \vert$ and thus with $\bar t := d-t <s_0$ and $k_{\min}=s_0-\bar t$
\begin{align*}
\sum_{\substack{ S \subseteq \fset \\ \vert S \vert = s_0}} \gamma^m_s(t,\vert T \cap S \vert) &= \sum_{k=k_{\min}}^{s_0} \binom{t}{k}\binom{t-k}{s_0-k}\gamma^m_s(t,\vert T \cap S \vert) 
\\
&= \frac{1}{\fnum-s_0+1}\sum_{k=k_{\min}}^{s_0} (-1)^{s_0-k} \binom{t}{k}\binom{\fnum-t}{s_0-k} \binom{\fnum-s_0}{t-k}^{-1}
\\
&= \frac{1}{\fnum-s_0+1}\sum_{k=s_0-\bar t}^{s_0} (-1)^{s_0-k} \binom{\fnum - \bar t}{k}\binom{\bar t}{s_0-k} \binom{\fnum-s_0}{\fnum-\bar t-k}^{-1}
\\
&= \frac{1}{\fnum-s_0+1}\sum_{k=0}^{\bar t} (-1)^{k} \binom{\fnum - \bar t}{s_0-k}\binom{\bar t}{k} \binom{\fnum-s_0}{\fnum-\bar t-s_0+k}^{-1}
\\
&= \frac{1}{\fnum-s_0+1}\sum_{k=0}^{\bar t} (-1)^{k} \binom{\fnum - \bar t}{s_0-k}\binom{\bar t}{k} \binom{\fnum-s_0}{\bar t-k}^{-1} 
\\
&= (-1)^{s_0}\rho(\bar t).
\end{align*}

We can explicitly compute $\rho(t)$ as
\begin{align*}
\rho(t) &= \frac{1}{\fnum-s_0+1}\sum_{k=0}^{t} (-1)^{s_0-k} \binom{t}{k}\binom{\fnum-t}{s_0-k} \binom{\fnum-s_0}{t-k}^{-1} 
\\
&=(-1)^{s_0}\frac{t!(\fnum-t)!}{(\fnum-s_0+1)!}\sum_{k=0}^{t} (-1)^k \frac{1}{(s_0-k)!k!}
\\
&= (-1)^{s_0}\frac{t!(\fnum-t)!}{s_0!(\fnum-s_0+1)!}\sum_{k=0}^{t} (-1)^k \binom{s_0}{k} 
\\
&= (-1)^{s_0}\frac{t!(\fnum-t)!}{s_0!(\fnum-s_0+1)!} \binom{s_0-1}{t} = \frac {(-1)^{s_0}}{s_0}\binom{\fnum-t}{s_0-t-1},
\end{align*}
where we have used that $\sum_{k=0}^t (-1)^k \binom{s_0}{k} = \binom{s_0-1}{t}$ for $t<s_0$.
Hence,
\begin{align*}
\sum_{\substack{ S \subseteq \fset \\ \vert S \vert = s_0}} I^{m}(S) =  \frac{1}{s_0}\sum_{\substack{T\subseteq \fset \\ t < s_0}} \left[(-1)^{s_0}\nu(T)+\nu(\fset \setminus T) \right] \binom{\fnum-t}{s_0-t-1}.
\end{align*}
\end{proof}

\clearpage
\section{Proofs}\label{appx::proofs}
This section contains the proofs of the claims made in the main paper.

\subsection{Proof of Theorem \ref{thm::si}}
\begin{proof}
By definition, the sum $I^m(S) := \sum_{T \subseteq \fset\setminus S} m(t) \delta^\nu_S(T)$ ranges over all subsets $T \subseteq \fset$, where every subset is exactly once evaluated.
On the one hand, it is easy to see that every evaluated subset in $I^m(S)$ is different, as $T \cup L$ is unique.
Furthermore, given any subset $T \subseteq \fset$, we decompose $T = \tilde T \cup L$, where $\tilde T \subseteq \fset \setminus S$ and $L := T\cap S \subseteq S$.
The corresponding weight is $m(\tilde t) = m(t-l) = m(t-\vert T \cap S\vert)$ and the sign from $\delta^\nu_S(\tilde T)$ is $(-1)^{s-l} = (-1)^{s-\vert T \cap S\vert}$.
This yields with $\nu_0(T) := \nu(T) - \nu(\emptyset)$
\begin{align*}
I^m(S) &= \sum_{T\subseteq \fset} \nu(T) \gamma^m_s(t,\vert T \cap S \vert) =  \sum_{T\subseteq \fset} \nu_0(T) \gamma^m_s(t,\vert T \cap S \vert)  + \nu(\emptyset) \sum_{T\subseteq \fset} \gamma^m_s(t,\vert T \cap S \vert) 
\\
&= \sum_{T\subseteq \fset} \nu_0(T) \gamma^m_s(t,\vert T \cap S \vert),
\end{align*}
as the sum over all $\gamma^m$ is zero, if the dummy axiom is fulfilled.
\end{proof}

\begin{remark}\label{appx::rem_nu_zero}
It is important to note, that the sum $\sum_{T \subseteq \fset}\gamma^m_s(t,\vert T \cap S \vert)$ is not zero, if not all subsets are considered, which makes it crucial to use $\nu_0$ instead of $\nu$.
In fact, the estimates of $I^m$ would be heavily skewed by $\nu(\emptyset)$.
While the estimator would still be unbiased, its variance would scale with $\nu(\emptyset)^2$.
\end{remark}

\subsection{Proof of Theorem \ref{thm::unbiased_consistent}}
\begin{proof}
We aim to show that $\hat I^m(S)$ is unbiased and consistent, i.e. $\mathbb{E}\left[\hat I^m_{k_0}(S)\right] = I^m(S)$ and $\lim_{K\to \infty}\hat I^m_{k_0}(S) = I^m(S)$.
Given 
\begin{align*}
    \hat I^m_{k_0}(S) :=  c_{k_0}(S) + \frac 1 K \cdot \sum_{k=1}^K \nu_0(T_k) \frac{\gamma^m_s(t_k,\vert T_k \cap S \vert)}{p_{k_0}(T_k)},
\end{align*}
it is clear that due to the linearity of the expectation
\begin{align*}
\mathbb{E}_{T \sim p_{k_0}(T)}\left[\hat I^m_{k_0}(S)\right] = c_{k_0}(S) + \frac 1 K \sum_{k=1}^K\mathbb{E}_{T \sim p_{k_0}(T)}\left[\nu_0(T) \frac{\gamma^m_s(\vert T \vert,\vert T \cap S \vert)}{p_{k_0}(T)}\right] = I^m(S).
\end{align*}

Furthermore, let $\sigma^2(S) := \mathbb{V}_{T \sim p_{k_0}(T)}\left[\nu_0(T)\frac{\gamma^m_s(\vert T \vert,\vert T \cap S\vert)}{p_{k_0}(T)} \right]$ be the variance of each estimate, then, by the law of large numbers
\begin{equation*}
\frac 1 K \cdot \sum_{k=1}^K \nu_0(T_k) \frac{\gamma^m_s(t_k,\vert T_k \cap S \vert)}{p_{k_0}(T_k)} \overset{K \to \infty}{\longrightarrow}  \mathbb{E}_{T \sim p_{k_0}(T)}\left[\nu_0(T) \frac{\gamma^m_s(\vert T \vert,\vert T \cap S \vert)}{p_{k_0}(T)}\right],
\end{equation*}
and thus $\lim_{K\to \infty} \hat I^m(S) = I^m(S)$.
Lastly, as $\hat I^m(S)$ is unbiased, we have for $\epsilon > 0$ by Chebyshev's inequality
\begin{align*}
\mathbb{P}(\vert\hat I^m_{k_0}(S) - I^m(S)\vert > \epsilon) \leq \frac{\mathbb{V}\left[\hat I^m_{k_0}(S)\right]}{\epsilon^2} = \frac 1 {K^2} \frac{K \sigma^2(S)}{\epsilon^2} = \frac{1}{K} \frac{\sigma^2(S)}{\epsilon^2}.
\end{align*}
\end{proof}

\subsection{Proof of Theorem \ref{thm::SV_representation}}
\begin{proof}
We let $m(t) := \frac{(\fnum-t-1)!t!}{\fnum!}$ and apply Theorem \ref{thm::si}. 
With $\gamma^m_s(0,0)=-m(0)$, $\gamma^m_s(\fnum,1)=m(\fnum-1)$ and $m(0) = m(\fnum-1) = \frac 1 \fnum$, we have
\begin{align*}
I^{\text{SV}}(i) &\overset{\text{Theorem \ref{thm::si}}}{=} \sum_{T \subseteq \fset} \nu(T)\gamma^m_s(t,\mathbf{1}(i \in T))  = \frac{\nu(\fset)-\nu(\emptyset)}{\fnum} +  \sum_{T \in \mathcal T_1} \nu(T) \gamma^m_s(t,\mathbf{1}(i \in T))
\\
&= c_1(i)  +  \sum_{T \in \mathcal T_1} \nu(T) \left[\mathbf{1}(i\in T)\gamma^m_s(t,1)+\mathbf{1}(i \notin T)\gamma^m_s(t,0)\right]
\\
&= c_1(i) + \sum_{T \in \mathcal T_1} \nu(T) \left[\mathbf{1}(i\in T)\left(\gamma^m_s(t,1)-\gamma^m_s(t,0)\right) + \gamma^m_s(t,0) \right]
\\
&= c_1(i) + \sum_{T \in \mathcal T_1} \nu(T) \left[\mathbf{1}(i\in T)\left(m(t-1)+m(t)\right) - m(t) \right]
\\
&= c_1(i) + \sum_{T \in \mathcal T_1} \nu(T) \left[\mathbf{1}(i\in T)\left(\frac{(\fnum-t)!(t-1)!}{\fnum!}+\frac{(\fnum-t-1)!t!}{\fnum!}\right) -\frac{(\fnum-t-1)!t!}{\fnum!} \right]
\\
&= c_1(i) + \sum_{T \in \mathcal T_1} \nu(T) \frac{(\fnum-t-1)!(t-1)!}{(\fnum-1)!}\left[\mathbf{1}(i\in T) - \frac{t}{\fnum}\right]
\\
&= c_1(i) + \sum_{T \in \mathcal T_1} \nu(T) \mu(t)\left[ \mathbf{1}(i \in T) - \frac{t}{\fnum}\right].
\end{align*}
\end{proof}

\subsection{Proof of Theorem \ref{thm::u_ksh}}
\begin{proof}
According to Proposition \ref{prop::shap_x_for_SV}, our goal is to show that
\begin{equation*}
    \hat I^{\text{SV}}_U(i) = c_1(i) + \frac{2 h_{\fnum-1}}{K} \sum_{k=1}^K \nu_0(T_k)\left[ \mathbf{1}(i \in T_k) - \frac{t_k}{\fnum}\right]
\end{equation*}
with $T_k \overset{\text{iid}}{\sim} p(T) := \mu(t)/(2 h_{\fnum-1})$, where $p$ is a probability distribution over $\mathcal T_1$ and $h_n := \sum_{t=1}^n t^{-1}$.
The proof is structured in the following steps:
\begin{enumerate}
\item Exact computation of $A^{-1}$ using the exact structure of $A$ with diagonal entries $\mu_1$ and off-diagonal entries $\mu_2$, cf. \cite[Appendix A]{Covert_Lee_2021}.
\item Exact computation of $\hat I^{\text{SV}}_U$, which yields with Proposition \ref{prop::shap_x_for_SV} $\hat I^{\text{SV}}_U(i) = \hat I^m_1(i)$, if $(\mu_1-\mu_2)2h_{\fnum-1} = 1$.
\item We show that $(\mu_1-\mu_2)2h_{\fnum-1} = 1$.
\end{enumerate}
\paragraph{Calculation of $A^{-1}$.}
It has been shown in \cite[Appendix A]{Covert_Lee_2021} that all off-diagonal entries are equal and all diagonal entries are equal, i.e. $A$ may be written as $A = \mu_2 \mathbf{J}+ (\mu_1-\mu_2) \mathbf{I}$ with off-diagonal entries $\mu_2 := p(Z_i=Z_j=1)$ and diagonal entries $\mu_1 := p(Z_i=1)$, where $Z_i$ refers to the i-th component of the binary vector $Z$ and $\mathbf{J}$ is a matrix of ones and $\mathbf{I}$ is the identity matrix.
The simple structure of A allows to compute the inverse exactly by using the following Lemma.

\begin{lemma}\label{lem::exact_inverse}
Let $\mu_1,\mu_2 > 0$ with $\mu_1\neq \mu_2$, then
\begin{align*}
(\mu_2 \mathbf{J} + (\mu_1-\mu_2)\mathbf{I})^{-1} = \tilde \mu_2 \mathbf{J} + (\tilde\mu_1-\tilde\mu_2)\mathbf{I}
\end{align*}
with
\begin{align*}
    \tilde \mu_2  &= \frac{-\mu_2}{(\mu_1-\mu_2)(\mu_1+(d-1)\mu_2)} 
    \\
    \tilde\mu_1 &= \frac{\mu_1+ (\fnum-2) \mu_2}{(\mu_1-\mu_2)(\mu_1+(\fnum-1)\mu_2)}.
\end{align*}
\end{lemma}
\begin{proof}
We compute
\begin{align*}
\mathbf{I} = &(\mu_2 \mathbf{J} + (\mu_1-\mu_2)\mathbf{I}) \cdot (\tilde\mu_2 \mathbf{J} + (\tilde\mu_1-\tilde\mu_2)\mathbf{I})
\\
= &((\mu_1+(\fnum-1)\mu_2)\tilde\mu_2+(\tilde\mu_1-\tilde\mu_2)\mu_2) \mathbf{J} 
\\
&+ (\mu_1-\mu_2)(\tilde\mu_1-\tilde\mu_2)\mathbf{I},
\end{align*}
which yields $(\mu_1-\mu_2)(\tilde\mu_1-\tilde\mu_2)=1$ and $\mu_2 \tilde\mu_2 \fnum+(\mu_1-\mu_2)\tilde\mu_1+(\tilde\mu_1-\tilde\mu_2)\mu_2=0$.
From the first equation we have $\tilde\mu_1-\tilde\mu_2 = 1/(\mu_1-\mu_2)$ and thus by the second equation
\begin{align*}
    \tilde\mu_2 = \frac{-\mu_2}{(\mu_1-\mu_2)(\mu_1+(\fnum-1)\mu_2)}
\end{align*}
and hence
\begin{equation*}
    \tilde\mu_1 = \frac{\mu_1 + (\fnum-2)\mu_2}{(\mu_1-\mu_2)(\mu_1+(\fnum-1)\mu_2)}
\end{equation*}
\end{proof}

\paragraph{Calculation of $\hat I^{\text{SV}}_U$.}
By Lemma \ref{lem::exact_inverse}, we proceed to compute the different components of
\begin{equation*}
    \hat I^{\text{SV}}_{U} := A^{-1}\left(\hat b_L - \mathbf{1}\frac{\mathbf{1}^T A^{-1} \hat b_L - \nu_0(\mathbf{1})}{\mathbf{1}^T A^{-1} \mathbf{1}} \right).
\end{equation*}
First, 
\begin{equation*}
\mathbf{1}^T A^{-1} = \left((\fnum-1)\tilde\mu_2 + \tilde\mu_1\right)\mathbf{1}^T=\frac{1}{\mu_1+(\fnum-1)\mu_2}\mathbf{1}^T.
\end{equation*}
Then the denominator yields $\mathbf{1}^T A^{-1}\mathbf{1} = \frac{\fnum}{\mu_1+(\fnum-1)\mu_2}$.
We then obtain
\begin{equation*}
\mathbf{1}\frac{\mathbf{1}^T A^{-1} \hat b_L - \nu_0(\mathbf{1})}{\mathbf{1}^T A^{-1} \mathbf{1}} = \frac 1 d \mathbf{1}\mathbf{1}^T \hat b_L- \frac{\mu_1+(\fnum-1)\mu_2}{\fnum} \nu_0(\mathbf{1})\mathbf{1} = \frac 1 \fnum \mathbf{J} \cdot \hat b_L + \frac{\mu_1+(\fnum-1)\mu_2}{\fnum} \nu_0(\mathbf{1})\cdot \mathbf{1},
\end{equation*}
which, with $A^{-1} \mathbf{1} = (\tilde\mu_1 + (\fnum-1)\tilde\mu_2)\mathbf{1} = \frac {1}{\mu_1+(\fnum-1)\mu_2} \mathbf{1}$, yields
\begin{equation*}
\hat I^{\text{SV}}_U = A^{-1}(\hat b_L - \frac 1 \fnum \mathbf{J} \cdot \hat b_L) + \frac {\nu_0(\mathbf{1})}{\fnum} \cdot \mathbf{1} = c_1 + A^{-1}(\hat b_L - \frac 1 \fnum \mathbf{J} \cdot \hat b_L).
\end{equation*}
It remains to show that 
\begin{equation*}
\left(A^{-1}(\hat b_L - \frac 1 \fnum \mathbf{J} \cdot \hat b_L)\right)_i = \frac{1}K\sum_{k=1}^K \nu_0(T_k) \frac{\gamma^m_s(t_k,\mathbf{1}(i\in T_k))}{p(T_k)}.
\end{equation*}
With $\hat b_L = \frac{1}{K}\sum_{k=1}^K z_k\nu_0(z_k)$ it follows
\begin{equation*}
A^{-1}(\hat b_L - \frac 1 \fnum \mathbf{J} \cdot \hat b_L) = \frac 1 K \sum_{k=1}^K\left(A^{-1}z_k - \frac 1 \fnum A^{-1}\mathbf{J}z_k\right) \nu_0(z_k).
\end{equation*}
Then $A^{-1} z_k = t_k \tilde\mu_2 \mathbf{1} + (\tilde\mu_1 - \tilde\mu_2)z_k$ where $t_k$ is the subset size, i.e. $t_k$ is the sum of all entries in $z_k$.
It follows with $A^{-1} J= \frac{1}{\mu_1+(\fnum-1)\mu_2} J$
\begin{equation*}
A^{-1}z_k - \frac 1 \fnum A^{-1}\mathbf{J}z_k = t_k \left(\tilde\mu_2 - \frac{1}{d\left(\mu_1+(\fnum-1)\mu_2\right)} \right)\mathbf{1}  + (\tilde\mu_1-\tilde\mu_2)z_k  = \frac{1}{\mu_1-\mu_2} \left(z_k- \frac{t_k} \fnum \mathbf{1}\right).
\end{equation*}
For the i-th component, we have with set notation $T_k$ and for the SV weights $m(t) := \frac{(\fnum-t-1)!t!}{\fnum!}$ then
\begin{align*}
\hat I^{\text{SV}}_U(i) - c_1(i) = \left(A^{-1}(\hat b_L - \frac 1 \fnum \mathbf{J} \cdot \hat b_L)\right)_i &= \frac 1 K \sum_{k=1}^K \nu_0(T_k) \frac{1}{\mu_1-\mu_2} \left( \mathbf{1}(i \in T_k)- \frac{t_k}{\fnum} \right)
\\
&= \frac{1}{(\mu_1-\mu_2)2h_{\fnum-1}} \left(\hat I^m_1(i) - c_1(i)\right),
\end{align*}
where we have used Proposition \ref{prop::shap_x_for_SV} for $\hat I^m_1(i)$.
It remains to show that $(\mu_1-\mu_2)2h_{\fnum-1} = 1$.

\paragraph{Show that $(\mu_1-\mu_2)2h_{\fnum-1} = 1$.}
We let $p(T_k) := \mu(t)/R$ be a probability distribution over $\mathcal T_1$.
By definition, and as subsets of size $t$ have equal probability, we have
\begin{align*}
\mu_1 = p(Z_i = 1) &= \sum_{t=1}^{\fnum-1} p(Z_i = 1 \vert \mathbf{1}^T Z = t)p(\mathbf{1}^T Z = t) 
\\
&= \sum_{t=1}^{\fnum-1}\frac{\binom{\fnum-1}{t-1}}{\binom{\fnum}{t}}\frac{\mu(t)} {R} \binom{\fnum}{t} = \sum_{t=1}^{\fnum-1}\binom{\fnum-1}{t-1}\frac{\mu(t) } {R}
\end{align*}
and
\begin{align*}
\mu_2 = p(Z_i = Z_j = 1) &= \sum_{t=1}^{\fnum-1} p(Z_i = Z_j = 1 \vert \mathbf{1}^T Z = t)p(\mathbf{1}^T Z = t) 
\\
&= \sum_{t=2}^{\fnum-1}\frac{\binom{\fnum-2}{t-2}}{\binom{\fnum}{t}}\frac{\mu(t)} {R} \binom{\fnum}{t} = \sum_{t=2}^{\fnum-1}\binom{\fnum-2}{t-2}\frac{\mu(t) } {R}.
\end{align*}
Hence,
\begin{align*}
\mu_1 - \mu_2 = \frac{\mu(1)}{R} + \sum_{t=2}^{\fnum-1} \frac{\mu(t)}{R}\left(\binom{\fnum-1}{t-1}-\binom{\fnum-2}{t-2}\right)= \sum_{t=1}^{\fnum-1} \frac{\mu(t)}{R}\binom{\fnum-2}{t-1} = \frac 1 R,
\end{align*}
where we have used the recursion for the binomial coefficient and $\mu(t) = \frac{1}{\fnum-1}\binom{\fnum-2}{t-1}^{-1}$.
Lastly, we have seen in the proof of Proposition \ref{prop::shap_x_for_SV} that $R=2 h_{\fnum-1}$, which finishes the proof.
\end{proof}

\subsection{Proof of Theorem \ref{thm::s-efficiency}}
\begin{proof}
We prove the statements in separate subsections.
We consider interactions of maximum order $s_0 \geq 1$, summarized in $\mathcal I := \{S \subset \fset \mid s = s_0\}$.
To show that SII and STI are s-efficient, by Theorem \ref{thm::si} it suffices to show that $\sum_{S \in \mathcal I}\gamma^m_s(t,\vert T \cap S\vert) = 0$ for all $T \in \mathcal T_{s_0}$.
Given a subset $T \in \mathcal T_{s_0}$ with $\vert T \cap S \vert = k$ and $k \in \{0,\dots,s_0\}$, we have
\begin{align*}
\sum_{S \in \mathcal I} \gamma^m_s(t,\vert T \cap S \vert) &= \sum_{k=0}^{s_0} \binom{t}{k}\binom{t-k}{s_0-k}\gamma^m_s(t,k) = \sum_{k=0}^{s_0} (-1)^{s_0-k} \binom{t}{k}\binom{\fnum-t}{s_0-k} m(t-k).
\end{align*}

\subsubsection{Proof of s-efficiency for SII.}
We let $m(t) := m_{s_0}^{\text{SII}}(t)= \frac{(\fnum-t-s_0)!t!}{(\fnum-s_0+1)!}=\frac{1}{\fnum-s_0+1}\binom{\fnum-s_0}{t}^{-1}$ and obtain
\begin{equation*}
\sum_{S \in \mathcal I} \gamma^m_s(t,\vert T \cap S \vert) = \frac{1}{\fnum-s_0+1}\sum_{k=0}^{s_0} (-1)^{s_0-k} \binom{t}{k}\binom{\fnum-t}{s_0-k} \binom{\fnum-s_0}{t-k}^{-1}.
\end{equation*}
We now consider
\begin{align*}
\sum_{k=0}^{s_0} (-1)^k \binom{t}{k} \binom{\fnum-t}{s_0-k} \binom{\fnum-s_0}{t-k}^{-1} &= \sum_{k=0}^{s_0} (-1)^k \frac{t!}{k!(t-k)!} \frac{(\fnum-t)!}{(s_0-k)!(\fnum-t-s_0+k)!} \frac{(t-k)!(\fnum-s_0-t+k)!}{(\fnum-s_0)!}
\\
&= \frac{t!(\fnum-t)!}{(\fnum-s_0)!}\sum_{k=0}^{s_0} (-1)^k \frac{1}{(s_0-k)!k!}
\\
&= \frac{t!(\fnum-t)!}{s_0!(\fnum-s_0)!} \sum_{k=0}^{s_0} (-1)^k \binom{s_0}{k}
\\
&= \frac{t!(\fnum-t)!}{s_0!(\fnum-s_0)!} (1-1)^{s_0}
\\
&= 0,
\end{align*}
where we have used the binomial expansion for $(1-1)^{s_0}$.
Hence,
\begin{align*}
\sum_{S \in \mathcal I} \gamma^m_s(t,\vert T \cap S \vert) &= 0,
\end{align*}
which finishes the proof of SII s-efficiency.

\subsubsection{Proof of s-efficiency of STI.}
For STI, i.e. $m_{s_0}^{\text{STI}} := s_0\frac{(\fnum-t-1)!t!}{\fnum!}$, we have
\begin{align*}
\sum_{S \in \mathcal I} \gamma^m_s(t,\vert T \cap S \vert) &= \sum_{k=0}^{s_0} (-1)^{s_0-k} \binom{t}{k}\binom{\fnum-t}{s_0-k} s_0\frac{(\fnum-t+k-1)!(t-k)!}{\fnum!}
\\
&= \frac{s_0(-1)^{s_0} t!(\fnum-t)!}{s_0!\fnum!} \sum_{k=0}^{s_0} (-1)^k \binom{s_0}{k} \frac{(\fnum-t+k-1)!}{(\fnum-t-s_0+k)!}
\end{align*}
As $\frac{(\fnum-t+k-1)!}{(\fnum-t-s_0+k)!}$ is a polynomial with orders less than $s_0$, we can use $\sum_{k=0}^{s_0} \binom{s_0}{k} k^m = 0$ for $m<s_0$ \cite{SPIVEY20073130} to obtain
\begin{equation*}
\frac{s_0(-1)^{s_0} t!(\fnum-t)!}{s_0!\fnum!} \sum_{k=0}^{s_0} (-1)^k \binom{s_0}{k} \frac{(\fnum-t+k-1)!}{(\fnum-t-s_0+k)!} = 0,
\end{equation*}
which finishes the proof of STI s-efficiency.

\subsubsection{Proof of efficiency for STI.}
It is clear that s-efficiency of a CII implies by Theorem~\ref{thm::si} with sampling order $k\geq s_0$ that the sum of top-order interaction estimates is
\begin{align*}
    \sum_{S \in \mathcal I} \hat I^{\text{STI}}_{s_0}(t) = \sum_{S \in \mathcal I} c^{\text{STI}}_{s_0}(S).
\end{align*}
On the other hand, it also implies that the sum of STI scores for the top-order interactions are
\begin{align*}
    \sum_{S \in \mathcal I} I^{\text{STI}}_{s_0}(t)=\sum_{S \in \mathcal I} c^{\text{STI}}_{s_0}(S).
\end{align*}

While lower-order estimates are computed exactly for STI, it follows that the sum of STI scores and sum of SHAP-IQ estimates over all interaction sets $\mathcal S_{s_0}$ are equal.
Furthermore, due to the definition of STI, they must fulfill the efficiency axiom, which finishes the proof.

\subsubsection{Proof of efficiency for n-SII.}
This result follows from the aggregation suggested by n-SII, which is independent of the index.
The efficiency condition follows directly from the SV efficiency and the Bernoulli numbers, independent of higher-order interaction values, cf. proof in \cite[Proposition 12]{pmlr-v206-bordt23a}.
The proof is based on induction, starting from the SV, where the efficiency condition holds.
The specific aggregation based on the Bernoulli numbers then ensures that this efficiency condition is maintained.
To apply this observation to SHAP-IQ estimates of n-SII, we observe that due to Proposition~\ref{prop::shap_x_for_SV}, SHAP-IQ maintains the efficiency condition for order $s=1$, i.e. the SV estimates.
The aggregation of n-SII for the SHAP-IQ estimates of higher orders then immediately implies that this efficiency condition is always preserved, as the arguments presented therein hold independent of the interaction index, cf. \cite[Proof of Proposition 12]{pmlr-v206-bordt23a}.
\end{proof}

\clearpage
\section{Algorithmic Details of SHAP-IQ}\label{appx::algo-shap-iq}
In this section, we decribe further algorithmic details that are used in the SHAP-IQ implementation.

\subsection{Algorithm of SHAP-IQ}
The pseudo-code for SHAP-IQ is outlined in Algorithm~\ref{alg::approximation}.
\begin{algorithm}[!h]
    \caption{SHAP-IQ for any-order interactions $\mathcal S_{s_0}$ up to order $s_0$}
    \label{alg::approximation}
    \begin{algorithmic}[1]
    \REQUIRE Budget $K>0$, weights $q(t) \geq 0$ with $t \in [\fnum]$, precomputed weights $\gamma^m_s(t,\ell)$ for $s=1,\dots,s_0$, $t=0\dots,d$ and $\ell=0,\dots,s$
    \vspace{0.5em}
    \STATE $k_0 \gets {\textsc{\texttt{getSamplingOrder}}}(q,K)$
    \FOR[Deterministic]{$T \notin \mathcal T_{k_0}$}
        \STATE $\eta \gets \nu_0(T)$
        \FOR{$S \in \mathcal S_{s_0}$}
            \STATE $c_{k_0}(S) \gets c_{k_0}(S) + \eta \cdot \gamma^m_s(t,\vert T \cap S \vert)$ \COMMENT{Update deterministic part for all $S$}
        \ENDFOR
        \STATE $K \gets K - 1$
    \ENDFOR
    \FOR{$t=k_0,\dots,d-k_0$}
        \STATE $p(t) \gets q(t)/\left(\sum_{k=k_0}^{\fnum-k_0}q(k)\binom{\fnum}{k}\right)$ \COMMENT{compute probabilities $\mathbb{P}_{k_0}(\vert T \vert = t)$}
    \ENDFOR
    \FOR[Sampling]{$k=1,\dots,K$}
        \STATE $T \gets \textsc{\texttt{Sample}}(p,k_0)$
        \STATE $\eta \gets \nu_0(T)$
        \FOR{$S \in \mathcal S_{s_0}$}
            \STATE $\Delta(S) \gets \eta \cdot \gamma^m_s(t,\vert T \cap S\vert) \binom{\fnum}{t}/p[t]$ \COMMENT{Use probabilities $\mathbb{P}_{k_0}(T) = \mathbb{P}_{k_0}(\vert T \vert = t)/\binom{\fnum}{t} \propto q(t)$} 
        \ENDFOR
        \STATE $\hat \mu,\hat s^2 \gets {\textsc{\texttt{WelfordUpdate}}}(\hat \mu,s2,k,\Delta)$
    \ENDFOR
    \STATE mean $\hat I^m_{k_0} \gets c_{k_0} + \hat \mu^m$ and variance $\hat \sigma^2 \gets  s2 / (n-1)$
    \STATE \textbf{return} $\hat I^m_{k_0}$ and $\hat \sigma^2$
    \end{algorithmic}
\end{algorithm}

\subsection{Sampling Weights $q(t)$}
In our implementation, we rely on $q(t) \propto \mu(t)$ for $1\leq t\leq \fnum-1$, where the weights $q(0)=q(\fnum)=q_0 \gg 0$ are set to a high positive constant, which favors these subsets before weighting the remaining subsets in $\mathcal T_1$.
These weights ensure that SHAP-IQ is equal to U-KSH for the SV.
Another choice of weights is $q(t) = \frac{(\fnum-t-s_0)!(t-s_0)!}{(\fnum-s_0+1)!}$ for $s_0\leq t \leq \fnum-s_0$ and $q(t)=q_0$ otherwise, which prefers all orders up to $s_0$ and from $\fnum$ to $\fnum-s_0$.
This choice of subset weights may be beneficial for very low budgets, as it is important to ensure that $k_0 \geq s_0$ for SHAP-IQ to maintain the efficiency condition.
The algorithm to find $k_0$ given weight $q$ and budget $K$ is outlined in Algorithm~\ref{alg::find_k_0}.
The sampling procedure to generate a subset according to $p(T)$ is outlined in Algorithm~\ref{alg::sample_T}.

\begin{algorithm}
    \caption{Determine the the sampling order $k_0$ for the deterministic part}
    \begin{algorithmic}[1]
    \label{alg::find_k_0}
        \REQUIRE weights $q$ over $0,\dots,\fnum$, budget $K>0$ \vspace{0.5em}
        \STATE initialize $k_0 = 0$
        \FOR{$t = 0,\dots,\text{\textsc{\texttt{Floor}}}(\fnum/2)$}
            \STATE $R = \left(\sum_{k=k_0}^{\fnum-k_0}q[k]\binom{\fnum}{k}\right)$ \COMMENT{Normalization}
            \STATE $p[t] \gets q[t]\binom{\fnum}{t}/R$ 
            \STATE $p[\fnum-t] \gets q[\fnum-t]\binom{\fnum}{t}/R$
            \IF{$K \cdot q[t] > R$ and $K \cdot q[\fnum-t] > R$}
                \STATE $k_0 \gets k_0 + 1$
                \STATE $K \gets K - 2\binom{\fnum}{t}$
            \ENDIF
        \ENDFOR
        \STATE \textbf{return} $k_0$
    \end{algorithmic}
\end{algorithm}

\begin{algorithm}
    \caption{Sample a subset $T\sim p(T)$}
    \begin{algorithmic}[1]
    \label{alg::sample_T}
    \REQUIRE $p$ with $\sum_{k=k_0}^{\fnum-k_0} p[k] = 1$, sampling order $k_0$ \vspace{0.5em}
    \FOR{$t=k_0,\dots,d-k_0$}
        \STATE $p(\vert T \vert = t) \gets p[t]\binom{\fnum}{t}$
    \ENDFOR
    \STATE choose subset size $t_0$ with probability $p(\vert T \vert = t)$
    \STATE choose subset $T$ of size $t_0$ with probability $\binom{\fnum}{t_0}^{-1}$
    \STATE \textbf{return} $T$
    \end{algorithmic}
\end{algorithm}

\subsection{Welford's Algorithm}
Welford's algorithm \cite{Welford_1962} allows to iteratively update the mean and variance using a single pass. 
The algorithm is outlined in Algorithm~\ref{alg:welford_update}.

\begin{algorithm}
  \caption{Welford Algorithm for Mean and Variance \cite{Welford_1962}}
    \begin{algorithmic}[1]\label{alg:welford_update}
    \REQUIRE $\mu,s2,n,\Delta$
    \STATE $n \gets n +1$
    \STATE $\Delta_1 \gets \Delta - \mu$
    \STATE $\mu \gets \mu + \Delta/n$
    \STATE $\Delta_{2} \gets \Delta - \mu$
    \STATE $s2 \gets s2 + \Delta_1\Delta_2$
    \STATE \textbf{return} $\mu,s2$
    \end{algorithmic}
\end{algorithm}

\clearpage
\section{Experiments}
\label{appx::experiments}

For the interested reader, we provide a more detailed view on our empirical evaluation of Section~\ref{sec:Experiments}.
We give descriptions and pseudocode of the baseline algorithms approximating the three considered interaction indices SII, STI, and FSI, formal definitions of our synthetic games, and finally further obtained results that we omitted in the main part due to space constraints.

\subsection{Baseline Algorithms for SII, STI, and FSI}
\label{appx::algos-sti-sii}
In this section, we describe our baseline algorithms for SII, STI and FSI.
We distinguish between permutation-based approximation (SII and STI) and kernel-based approximation (FSI).

\subsubsection{Permutation-based (PB) Approximation}\label{appx::algos_PB}
The algorithm for SII is outlined in Algorithm~\ref{alg::sii_baseline}.
Note that with each permutation only $\fnum-s+1$ interaction estimates of order $s$ are updated.

\begin{algorithm}[ht!]
  \caption{Permutation-based sampling for SII for all orders up to $s_0$ \cite{Tsai_Yeh_Ravikumar_2022}}
    \begin{algorithmic}[1]
    \label{alg::sii_baseline}
	\REQUIRE maximum interaction order $s_0$, interaction set $\mathcal S_{s_0}$, budget $K$
	\STATE sum[S] $\gets$ 0 for all $S \in \mathcal S_{s_0}$
	\STATE count[S] $\gets$ 0 for all $S \in \mathcal S_{s_0}$
    \STATE permutationCost $\gets 0$
    \FOR{$s=1,\dots,s_0$}
        \STATE permutationCost $\gets$ permutationCost $+2^{s} \cdot (\fnum-s+1)$ \COMMENT{Evaluate costs per permutation}
    \ENDFOR
	\WHILE{$K\geq$ permutationCost}
		\STATE $\pi \gets \{i_1,\dots,i_\fnum\}$ random permutation of $\fset$
        \FOR{$s=1,\dots,s_0$}
		\FOR{$m=1,\dots,\fnum-s+1$}
			\STATE $S \gets \{i_m,\dots,i_{m+s-1}\}$
			\STATE $T \gets \{i_1,\dots,i_{m-1}\}$ the set of predecessors of $i_m$ in $\pi$
			\STATE sum[S] $\gets$ sum[S] + $\delta_S^\nu(T)$ \COMMENT{$\delta_S^\nu(T)$ costs $2^{s}$ evaluations}
			\STATE count[S] = count[S]+1
		\ENDFOR
        \ENDFOR
        \STATE $K \gets K - $ permutationCost \COMMENT{Update budget}
	\ENDWHILE
	\STATE SII[S] $\gets$ sum[S]/count[S] for all $S \in \mathcal S_{s_0}$.
	\STATE \textbf{return} SII
    \end{algorithmic}
\end{algorithm}

The sampling-based algorithm for top-order interactions of STI is outlined in Algorithm~\ref{alg::sti_baseline}.
Note that with each permutation all top-order interaction estimates can be updated.
However, the update requires a significant amount (permutationCost) of model evaluations.

\begin{algorithm}[ht!]
  \caption{Permutation-based sampling for STI for all orders up to $s_0$ \cite{Sundararajan_Dhamdhere_Agarwal_2020,Tsai_Yeh_Ravikumar_2022}}
    \begin{algorithmic}[1]\label{alg::sti_baseline}
	\REQUIRE maximum interaction order $s_0$, interaction set $\mathcal S_{s_0}$, budget $K$
	\STATE sum[S] $\gets$ 0 for all $S \in \mathcal S_{s_0}$
	\STATE count[S] $\gets$ 0 for all $S \in \mathcal S_{s_0}$
     \STATE \textbf{Compute exact (trivial) lower-order interactions}
    \STATE eval[S] $\gets$ 0 for all $S \in \mathcal S_{s_0-1}$ \COMMENT{Model evaluations for lower-order STI values}
    \FOR[Precompute model evaluations for lower-order STI]{$S \in \mathcal S_{s_0-1}$}
        \STATE eval[S] $\gets \nu(S)$
        \STATE $K \gets K-1$
    \ENDFOR
    \FOR[Lower-order interactions]{$S \in \mathcal S_{s_0-1}$}
        \FOR{$L \in \mathcal P(S)$}
            \STATE SII[S] $\gets$ SII[S] + $(-1)^{s-l} \cdot$ eval[L] \COMMENT{Exact lower-order STI values}
        \ENDFOR
    \ENDFOR
    \STATE \textbf{Compute sampling-based top-order interaction estimates}
    \STATE permutationCost $\gets 2^{s_0} \cdot \binom{\fnum}{s_0}$ \COMMENT{Every $S$ requires to compute $\delta^\nu_S$ with $2^{s_0}$ evaluations}
	\WHILE[Evaluate one permutation]{$K\geq$ permutationCost}
		\STATE $\pi \gets \{i_1,\dots,i_\fnum\}$ random permutation of $\fset$
		\FOR{all top-order interactions $S$}
			\STATE $i_m \gets$ the leftmost element of $S$  in $\pi$
			\STATE $T \gets \{i_1,\dots,i_{m-1}\}$ the set of predecessors of $i_m$ in $\pi$
			\STATE sum[S] $\gets$ sum[S] + $\delta_S^\nu(T)$ \COMMENT{$\delta^\nu_S(T)$ costs $2^{s_0}$ evaluations}
			\STATE count[S] = count[S]+1
		\ENDFOR
        \STATE $K \gets K-$ permutationCost \COMMENT{Update budget}
	\ENDWHILE
	\STATE STI[S] $\gets$ sum[S]$/$count[S] for all top-order interactions.
	\STATE \textbf{return} STI
    \end{algorithmic}
\end{algorithm}

\subsubsection{Kernel-based (KB) Approximation}\label{appx::algos_KB}
Given a budget of $K$, we first find the sampling budget by identifying $k_0$ according to Algorithm~\ref{alg::find_k_0} with weights $q(t) := \mu(t)$ and subtracting the number of subsets used for the deterministic part.
We then sample the remaining subsets according to $p(T) \propto \mu(t)$ according to Algorithm~\ref{alg::sample_T}.

Given the collection of $K$ subsets (deterministic and sampled), we solve an approximated weighted least square objective as
\begin{align*}
    &\mathbb{E}_{T \sim p(T)}\left[\left(\nu(T)-\sum_{\substack{S \in \mathcal S_{s_0}\\S \subseteq T}}{\beta(S)}\right)^2\right]
    \\
    \approx &\sum_{T \in \mathcal T_{k_0}}p(T)\left(\nu(T)-\sum_{\substack{S \in \mathcal S_{s_0}\\S \subseteq T}}{\beta(S)}\right)^2 
    + p(T \in \mathcal T_{k_0})\mathbb{E}_{T \sim p_{k_0}(T)}\left[\left(\nu(T)-\sum_{\substack{S \in \mathcal S_{s_0}\\S \subseteq T}}{\beta(S)}\right)^2\right],
\end{align*}
where $k_0$ is found similar to SHAP-IQ and $p_{k_0}$ is a probability distribution over $\mathcal T_{k_0}$ with $p_{k_0}(T) \propto \mu(t)$, which is related to $p$ with $p_{k_0}(T) = p(T)/p(T \in \mathcal T_{k_0})$.
The expectation over $p_{k_0}$ is then found by Monte Carlo integration.
Approximating this objective yields a weighted sum of sampled subsets that approximates the weighted least square problem.
This approximated least-square problem is then computed explicitly using

\begin{equation}
    \hat I^{\text{FSI}} = (\mathbf{Z}^T \mathbf{W} \mathbf{Z})^{-1} \mathbf{Z}^T \mathbf{W} \mathbf{y},
\end{equation}
where $Z \in \{0,1\}^{K\times d_{s_0}}$ is a matrix that represents a binary encoding for each sampled subset where an entry in column $S \in \mathcal S_{s_0}$ is equal to one, if the subset contains $S$ and zero otherwise.
The matrix $W \in \mathbb{R}^{K \times K}$ contains the weights for each subset on the diagonal, e.g. $p(T)$ for subsets of the deterministic part or $p(T\in \mathcal T_{k_0})/m$ for subsets of the sampled part, where $m$ refers to the number of sampled subsets for Monte Carlo integration.
The vector $\mathbf{y}$ consists of all model evaluations $\nu_0(T)$, where $T$ is in the collection of subsets.
To include the optimization constraint, we add $\fset$ to the collection with weight set to a high positive constant (mimicking infinite).
The algorithm is outlined in Algorithm~\ref{alg::fsi_baseline}.

\begin{algorithm}[ht!]
  \caption{Kernel-based approximation of FSI \cite{Lundberg_Lee_2017,Tsai_Yeh_Ravikumar_2022}}
    \begin{algorithmic}[1]\label{alg::fsi_baseline}
	\REQUIRE maximum interaction order $s_0$, budget $K$, high constant $c_0 >> 0$.
     \STATE Weight vector $w[T]$ with one row and column per distinct subset $T$.
     \STATE Binary subset matrix $Z[T,S]$ with one row per distinct subset $T$ and one column per interaction subset $S$.
     \STATE Model evaluation vector $y[T]$ with model evaluations $\nu(T)$ per distinct subset $T$.
     \STATE Index array: $I$ per distinct subset $T$
     \STATE \textbf{Initialize constraints}
    \STATE $w \gets \textsc{\texttt{Append}}(w,c_0)$
    \STATE $Z \gets \textsc{\texttt{AppendRow}}(Z,\mathbf{1}^T)$ with $\mathbf{1}^T$ of length $d_{s_0}$.
    \STATE $y \gets \textsc{\texttt{Append}}(y,\nu_0(\fset))$.
    \STATE $K \gets K-1$.
    \STATE \textbf{Deterministic Part}
    \FOR[Initialize subset size probabilities as $\mathbb{P}_1(\vert T \vert = t)$]{$t=1,\dots,d-1$}
        \STATE $p(t) \gets q(t)/\left(\sum_{k=1}^{\fnum-1}q(k)\binom{\fnum}{k}\right)$
    \ENDFOR   
    \STATE $k_0 \gets {\textsc{\texttt{getSamplingOrder}}}(q,K)$
    \FOR[Deterministic]{$T \in \mathcal T_1$ and $T \notin \mathcal T_{k_0}$}
        \STATE $Z \gets \textsc{AppendRow}(Z,\textsc{\texttt{Binary}}(T))$
        \STATE $y \gets \textsc{\texttt{Append}}(y,\nu_0(T))$
        \STATE $w \gets \textsc{\texttt{Append}}(w,p(t)/\binom{\fnum}{t})$ \COMMENT{Weight with probability $\mathbb P_1(T) = \mathbb{P}_1(\vert T \vert = t)/\binom{\fnum}{t}$}
        \STATE $K \gets K -1$
    \ENDFOR
    \STATE \textbf{Sampling Part}
    \STATE $w_0 \gets sum(p(t))$ for $k_0\leq t \leq \fnum-k_0$. \COMMENT{Remaining probability weight}   
    \FOR[Sampling]{$k=1,\dots,K$}
        \STATE $T \gets \textsc{\texttt{Sample}}(p,k_0)$
        \IF{$T \in I$}
        \STATE $Z \gets \textsc{\texttt{AppendRow}}(Z,\textsc{\texttt{Binary}(T)})$.
        \STATE $y \gets \textsc{\texttt{Append}}(y,\nu_0(T))$
        \STATE $w \gets \textsc{\texttt{Append}}(w,1)$
        \STATE $I.\textsc{\texttt{AddIndex}}(T)$
        \ELSE
        \STATE $w[I[T]] \gets w[I[T]] + 1$
        \ENDIF
    \ENDFOR
    \STATE $w[I[T]] \gets w[I[T]] \cdot w_0 / K$ for all $T \in I$ \COMMENT{Rescaling}
    \STATE $W \gets \textsc{\texttt{diag}}(w)$ \COMMENT{Diagonal matrix with diagonal $w$}
    \STATE FSI $\gets \textsc{\texttt{SolveWLS}}(Z,W,y)$.
	\STATE \textbf{return} FSI
    \end{algorithmic}
\end{algorithm}

\subsubsection{Computational Complexity of Baseline Methods}\label{appx::computational-complexity}
To evaluate one permutation for STI, the PB algorithm requires $2^s$ model evaluations per interaction, i.e. in total $\binom{\fnum}{s_0} \cdot 2^{s_0}$ for all top-order interactions.
With each evaluated permutation all interaction estimates can be updated.
For lower order interactions, STI requires to compute model evaluations for all subsets with $\vert S \vert \leq s_0$.
For SII, the complexity is $(\fnum-s+1)\cdot 2^s$ per permutation and only interaction estimates with $S \in \pi$ can be updated per permutation, i.e. $\fnum-s+1$ interaction estimates with one permutation.
This constitutes a significant drawback over the PB approximations for SV, which iterates only once through the permutation requiring $\fnum-1$ evaluations to update all estimates of the SV.
In contrast, as for SV, the KB approach of FSI allows to update all interaction estimates using \emph{one single} model evaluation.
However, the KB approach of FSI always requires to estimate \emph{all} interactions with order $s \leq s_0$ and its computational effort increases non-linear with the number of subsets used, as solving the weighted least square problem requires inverting a $K \times d_{s_0}$ matrix.

\subsection{Further Information about the Models}
\label{appx::games-information}

This section contains a detailed information about the models (SOUM, LM, and ICM) used in our experiments.

\subsubsection{Sum of Unanimity Model (SOUM)}
\begin{definition}[Sum of Unanimity Model (SOUM)]
For $N$ subsets $Q_1,\dots,Q_N \subseteq \fset$ and coefficients $a_1,\dots,a_N \in \mathbb{R}$ the \emph{sum of unanimity model (SOUM)} is defined as
\begin{equation*}
\nu(T) := \sum_{n=1}^N a_n \mathbf{1}(Q_n \subseteq T).
\end{equation*}
\end{definition}

For SOUMs, it is possible to efficiently compute the ground-truth values for CII.
\begin{proposition}[Ground-truth values for SOUM]\label{appx::ground_truth_unanimity}
For a SOUM, it holds
\begin{equation*}
I^m_\nu (S) = \sum_{n=1}^N a_n \omega(q_n,\vert S \cap Q_n \vert),
\end{equation*}
with 
\begin{align*}
 \omega(q,r) = \sum_{t=q}^\fnum \sum_{k=0}^{k_{max}(r)} \binom{\fnum-q-(s-r)}{t-q-k} \binom{s-r}{k} \gamma^m_s(t,k + r)
\end{align*}
and $k_{max}(r) := \min(t-q,s-r)$
\end{proposition}
\begin{proof}
Due to the linearity of the CII, it suffices to compute the CII for $\nu_Q(T) := \mathbf{1}(Q \subseteq T)$.
By Theorem \ref{thm::si}, we have
\begin{align*}
I_{\nu_Q}^m(S) &= \sum_{T \subseteq \fset} \mathbf{1}(Q \subseteq T) \gamma^m_s(t,\vert T \cap S \vert) 
\\
&= \sum_{t=q}^\fnum \sum_{k=0}^{k_{max}} \binom{\fnum-q-(s-\vert S \cap Q \vert)}{t-q-k} \binom{s-\vert S \cap Q\vert}{k} \gamma^m_s(t,k + \vert S \cap Q \vert) 
\\
&=: \omega(q,\vert S \cap Q \vert),
\end{align*}
where we used that $Q\cap S \subseteq T \cap S$ due to $\mathbf{1}(Q \subseteq T)$ and $\vert T \cap S \vert = \vert S \cap Q \vert + (\vert T \cap S \vert ) \setminus (\vert S \cap Q \vert)$, where $k := \vert (T \cap S)\setminus (S \cap Q) \vert$ ranges from $0$ to $k_{max} := \min(t-q,s-\vert S \cap Q \vert)$.
Since $S \cap Q$ is fixed, we need to count the number of subsets $T$ of size $t$, given $k$, such that $\vert T \cap S\vert = \vert S \cap Q \vert +k$.
We count $\binom{s-\vert S \cap Q \vert}{k}$ ways to choose subsets of elements that are not in $S \cap Q$ but are in $S$.
Then $q-(s-\vert S \cap Q \vert)$ elements of $T$ are fixed.
We thus select from $\fnum-q-(s-\vert S \cap Q \vert)$ elements exactly $t-q-k$ elements, as $q$ and $k$ elements are already contained in $T$.

Finally, the CII value is given as
\begin{equation*}
I^m_\nu (S) = \sum_{n=1}^N a_n \omega(q_n,\vert S \cap Q_n \vert),
\end{equation*}
where the weights $\omega$ can be precomputed with $\vert S \cap Q_n \vert \in \{0,\dots,s\}$.
\end{proof}

\subsubsection{Language Model (LM)}
For a language model (LM), we use a fine-tuned version of the DistilBERT\footnote{The model can be found at \url{https://huggingface.co/dhlee347/distilbert-imdb}.} transformer architecture \cite{Sanh.2019} on movie review sentences from the original \emph{IMDB} dataset \cite{Maas.2011} for sentiment analysis, i.e. $\nu$ has values in $[-1,1]$.
The IMDB data stems from the \emph{dataset} library \cite{Lhoest_Datasets_A_Community_2021}.
In the LM, for a given sentence, different feature coalitions are computed by masking absent features in a tokenized sentence.
The implementation is based on the \emph{transformers} API \cite{Wolf_Transformers_State-of-the-Art_Natural_2020}.

\subsubsection{Image Classifier (ICM)}
The image classification model (ICM) is a ResNet18 \cite{resnet18} pre-trained on ImageNet \cite{ImageNet} as provided by \emph{torch} \cite{torch.2017}.
We randomly sample $50$ images from ImageNet \cite{ImageNet} and explain the prediction of the highest probability class given the original image.
To obtain the prediction of different coalitions, we pre-compute super-pixels with SLIC \cite{SLIC,scikit-image} to obtain a function on $\fnum=14$ features and apply mean imputation on absent features.

\subsection{Hardware Details and Environmental Impact of the Experiments}

Running the experiments required approximately $2\,000$ CPU hours in total.
The experiments concerning the approximation quality of SHAP-IQ compared to the baselines were run on an computation cluster on hyperthreaded Intel Xeon E5-2697 v3 CPUs clocking at with 2.6Ghz.
To further increase the efficiency of the experiments, the outputs of the LM and ICM were pre-computed given the powerset of all features. 
Around $1\,500$ CPU hours were consumed for these experiments on the cluster.
Before running the experiments on the cluster, the implementations were validated on a Dell XPS 15 9510 containing an Intel i7-11800H at 2.30GHz. 
For this and further small-scale experiments like the n-SII values approximately $500$ CPU hours were consumed.

\subsection{Further Experimental Results}

This section describes the further results and experiments omitted in the main body of the work.

\subsubsection{Approximation Quality of top Order Interactions.}
We further compute interaction scores for $s_0=1$, $s_0=2$, $s_0=3$, and $s_0=4$ of all three interaction indices SII, STI, and FSI on the LM.
We plot the MSE and Prec@10 based on $g=50$ independent iterations for these settings.
All results are summarized in Figure~\ref{fig:app_language}.
Moreover, we compute interaction scores for $s_0=1$, $s_0=2$, $s_0=3$, and $s_0=4$ of all three interaction indices SII, STI, and FSI on the ICM.
The MSE and Prec@10 based on $g=50$ independent iterations for these settings are shown in Figure~\ref{fig:app_icm}.
Further results to the plots show in Section~\ref{sec_experiments_different_ciis} for the SOUM are presented in Figure~\ref{fig:app_soum}.

\begin{figure}
    \centering
    \begin{minipage}[c]{0.32\textwidth}
         \includegraphics[width=\textwidth]{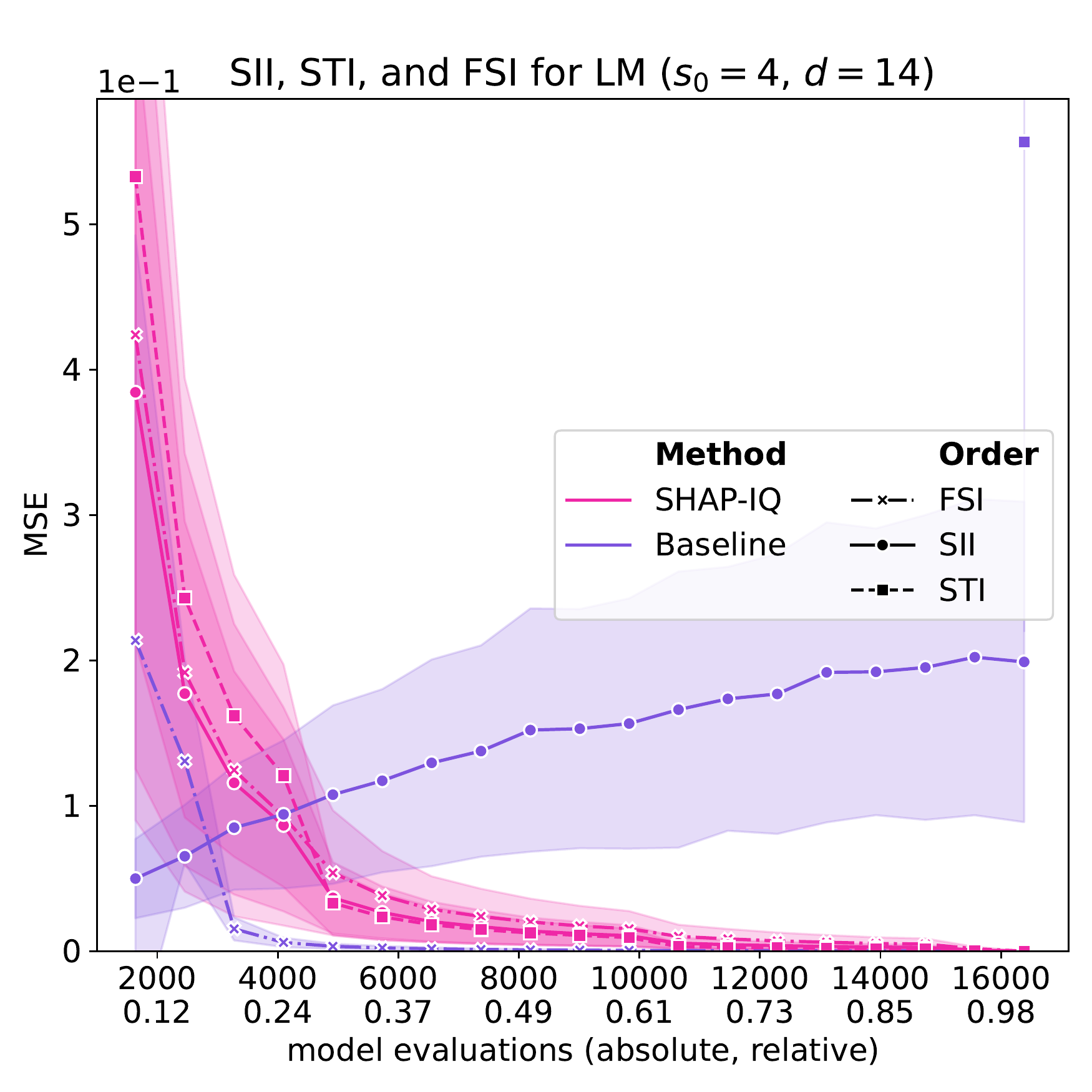}
     \end{minipage}
     \begin{minipage}[c]{0.32\textwidth}
         \includegraphics[width=\textwidth]{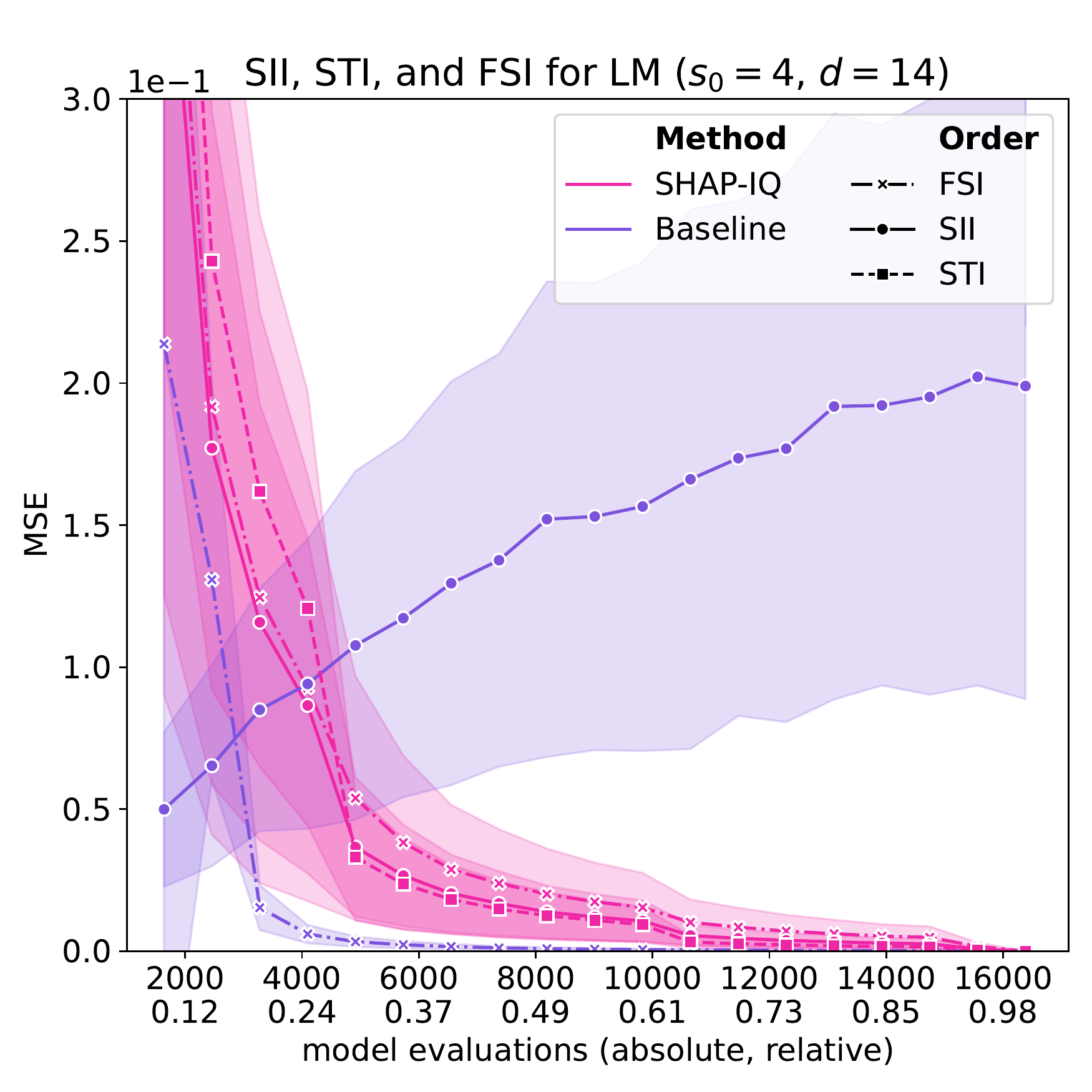}
     \end{minipage}
     \begin{minipage}[c]{0.32\textwidth}
         \includegraphics[width=\textwidth]{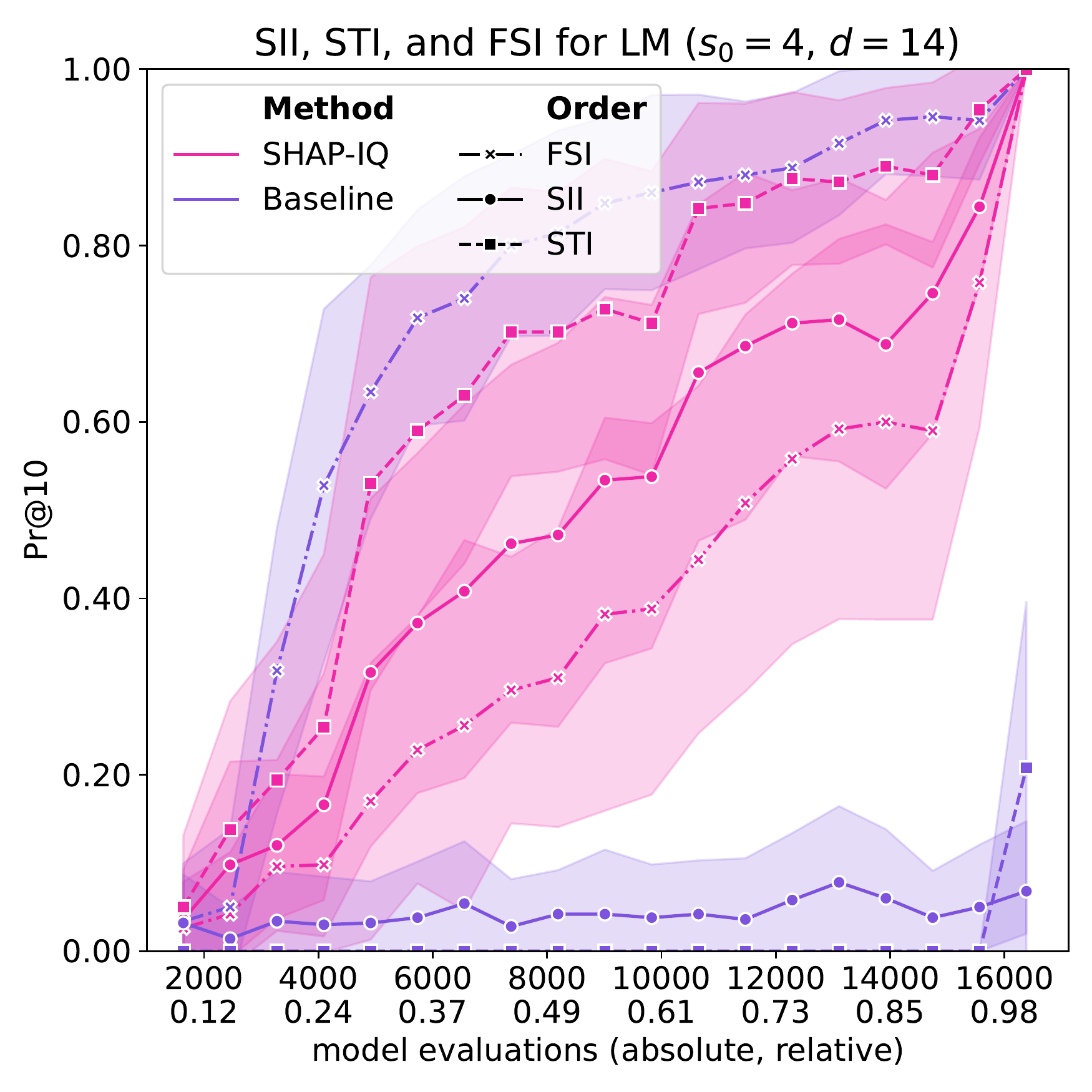}
     \end{minipage}
     \\
     \begin{minipage}[c]{0.32\textwidth}
         \includegraphics[width=\textwidth]{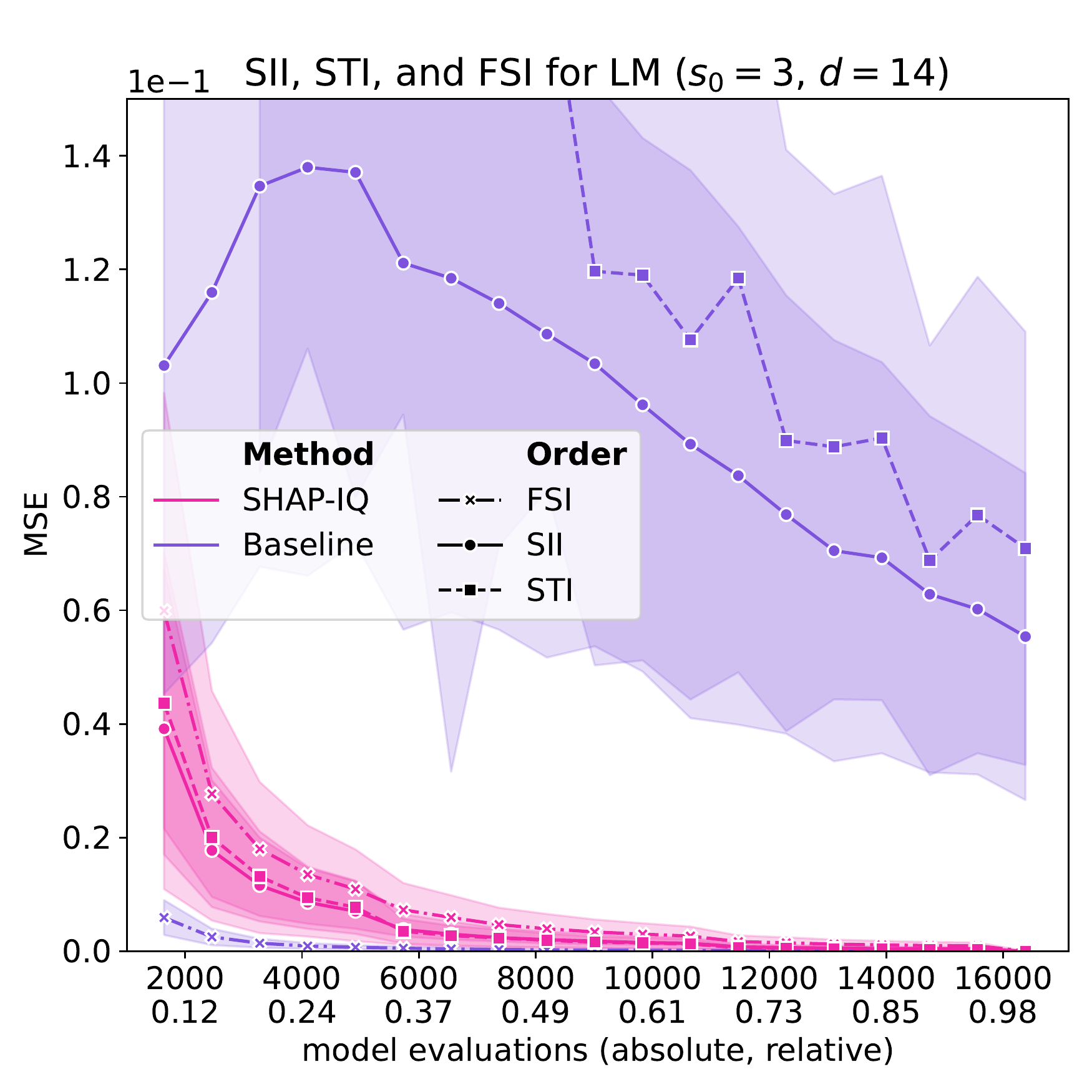}
     \end{minipage}
     \begin{minipage}[c]{0.32\textwidth}
         \includegraphics[width=\textwidth]{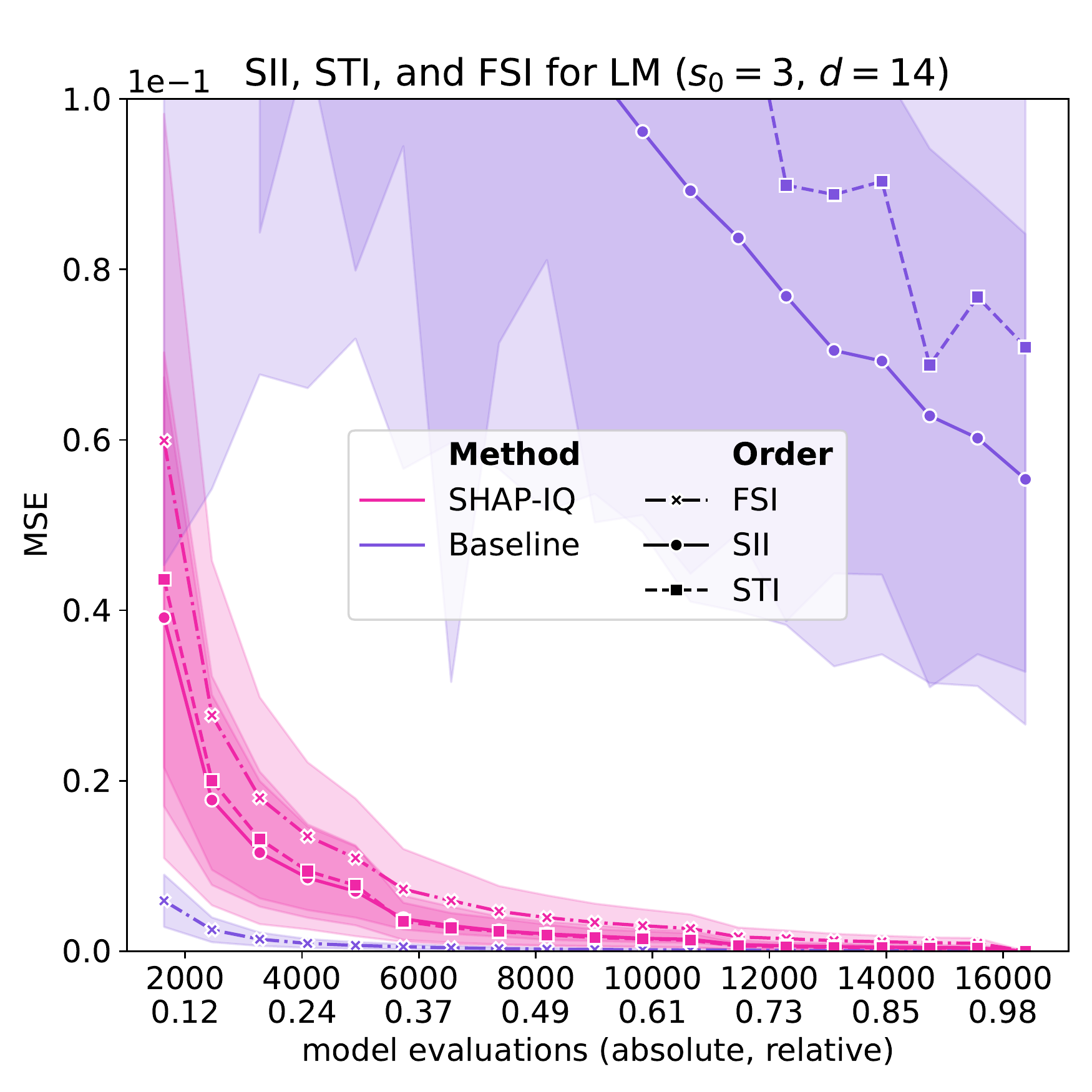}
     \end{minipage}
     \begin{minipage}[c]{0.32\textwidth}
         \includegraphics[width=\textwidth]{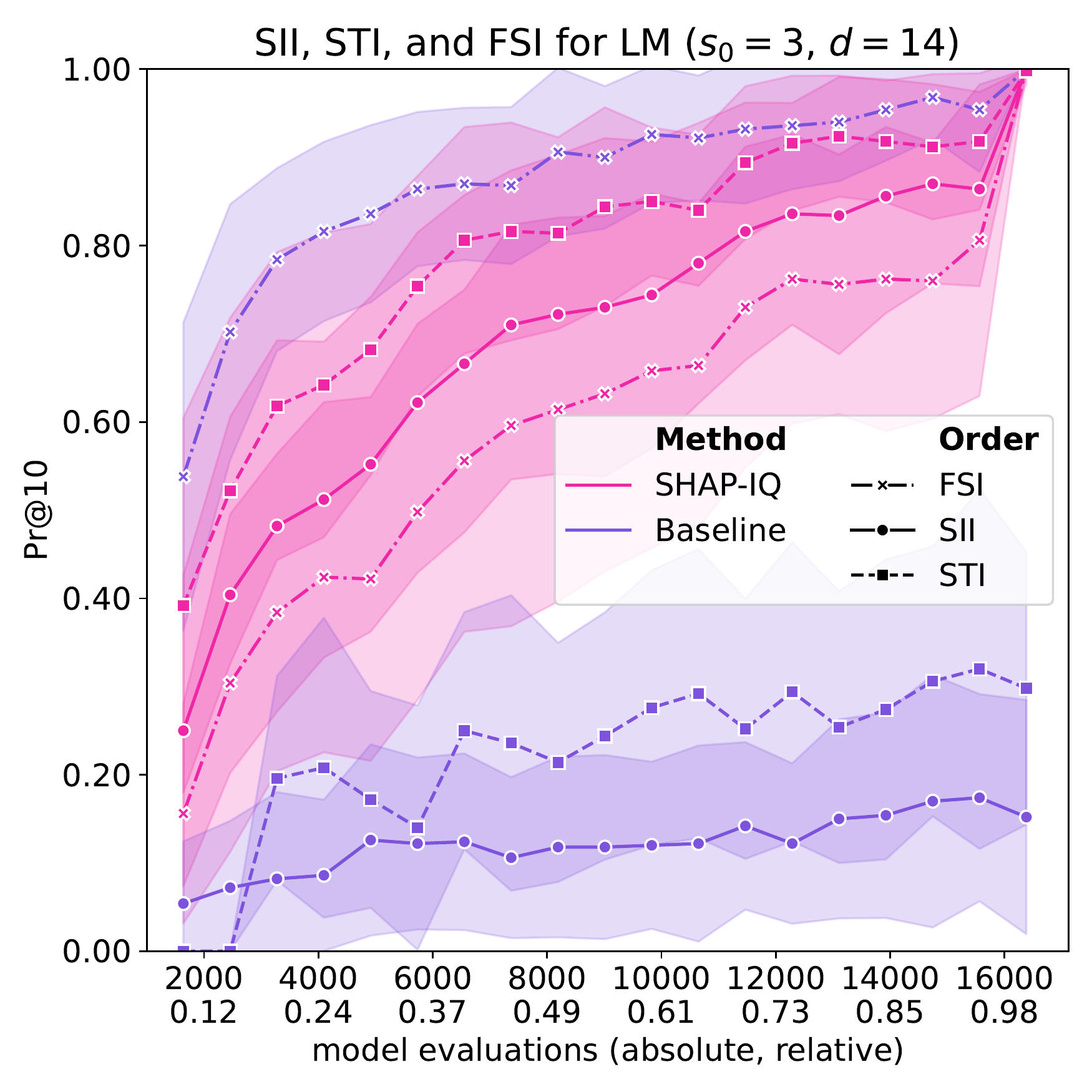}
     \end{minipage}
     \\
     \begin{minipage}[c]{0.32\textwidth}
         \includegraphics[width=\textwidth]{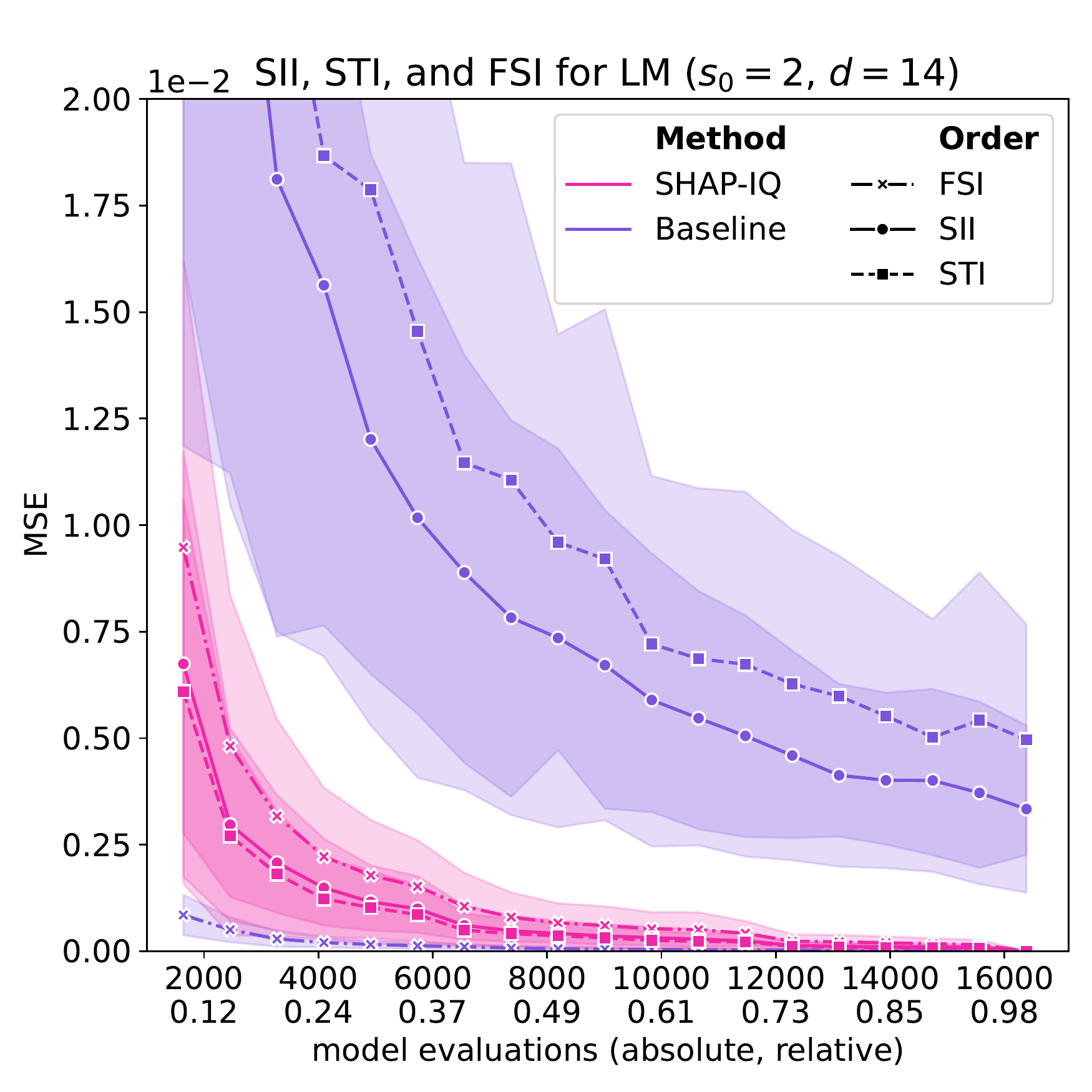}
     \end{minipage}
     \begin{minipage}[c]{0.32\textwidth}
         \includegraphics[width=\textwidth]{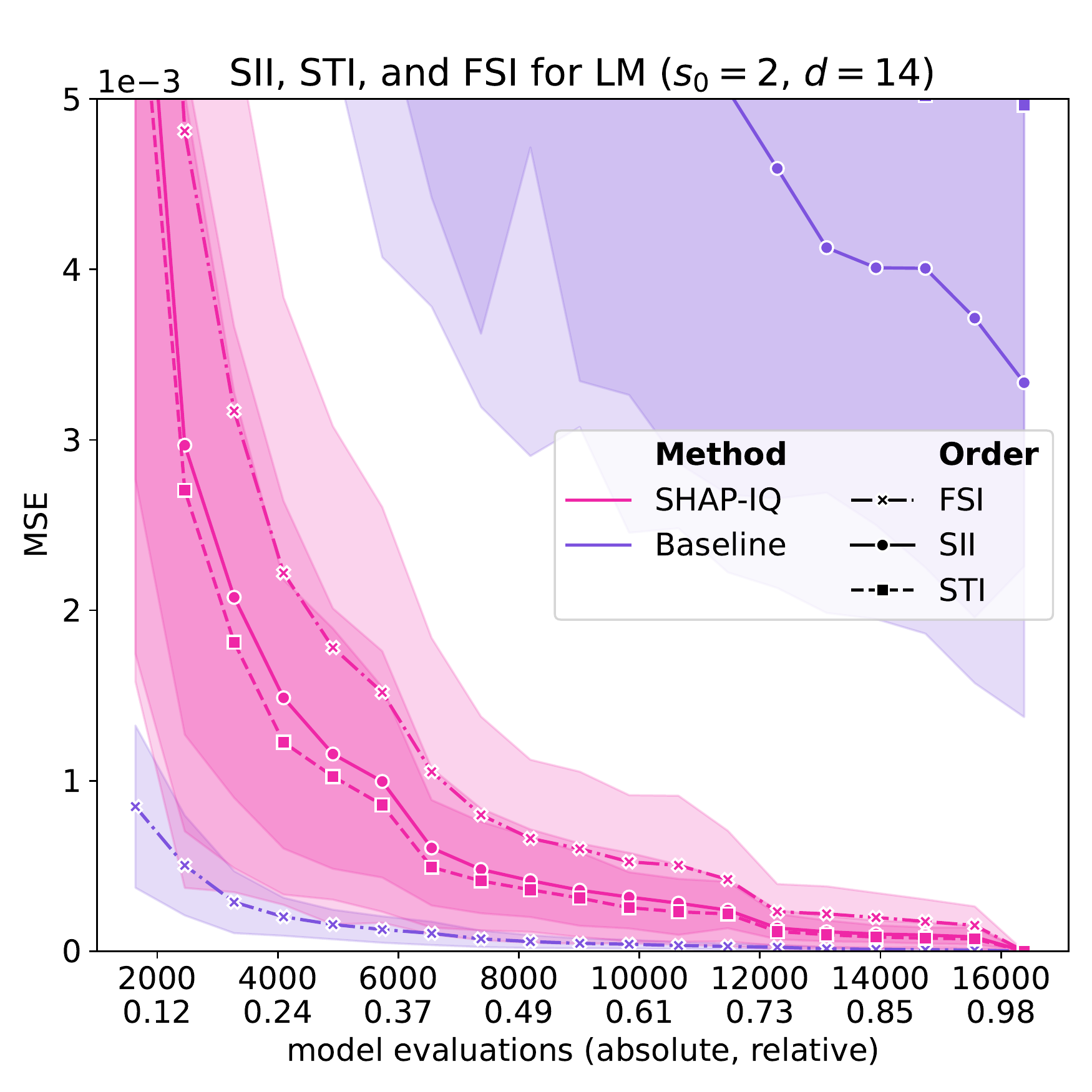}
     \end{minipage}
     \begin{minipage}[c]{0.32\textwidth}
         \includegraphics[width=\textwidth]{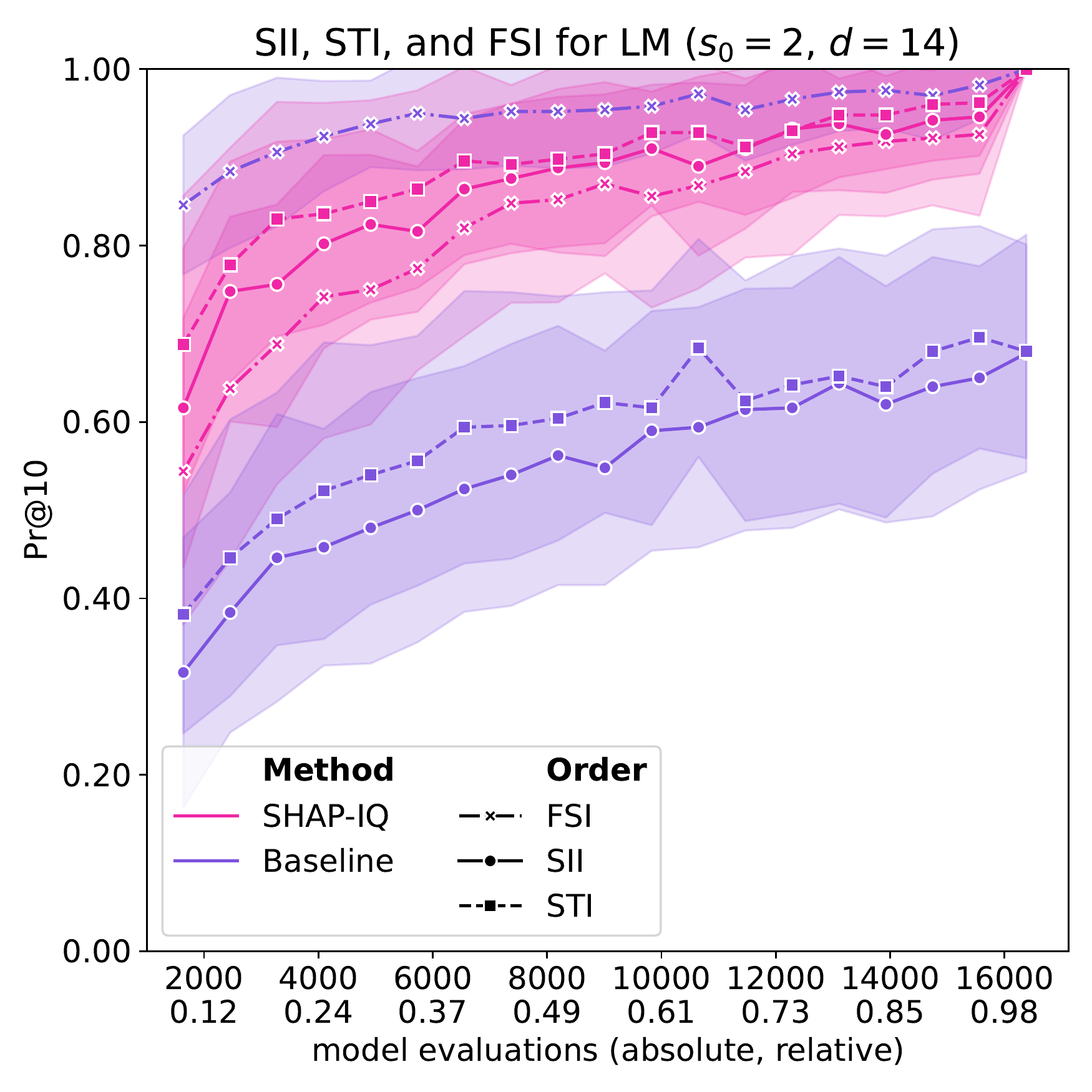}
     \end{minipage}
     \\
     \begin{minipage}[c]{0.32\textwidth}
         \includegraphics[width=\textwidth]{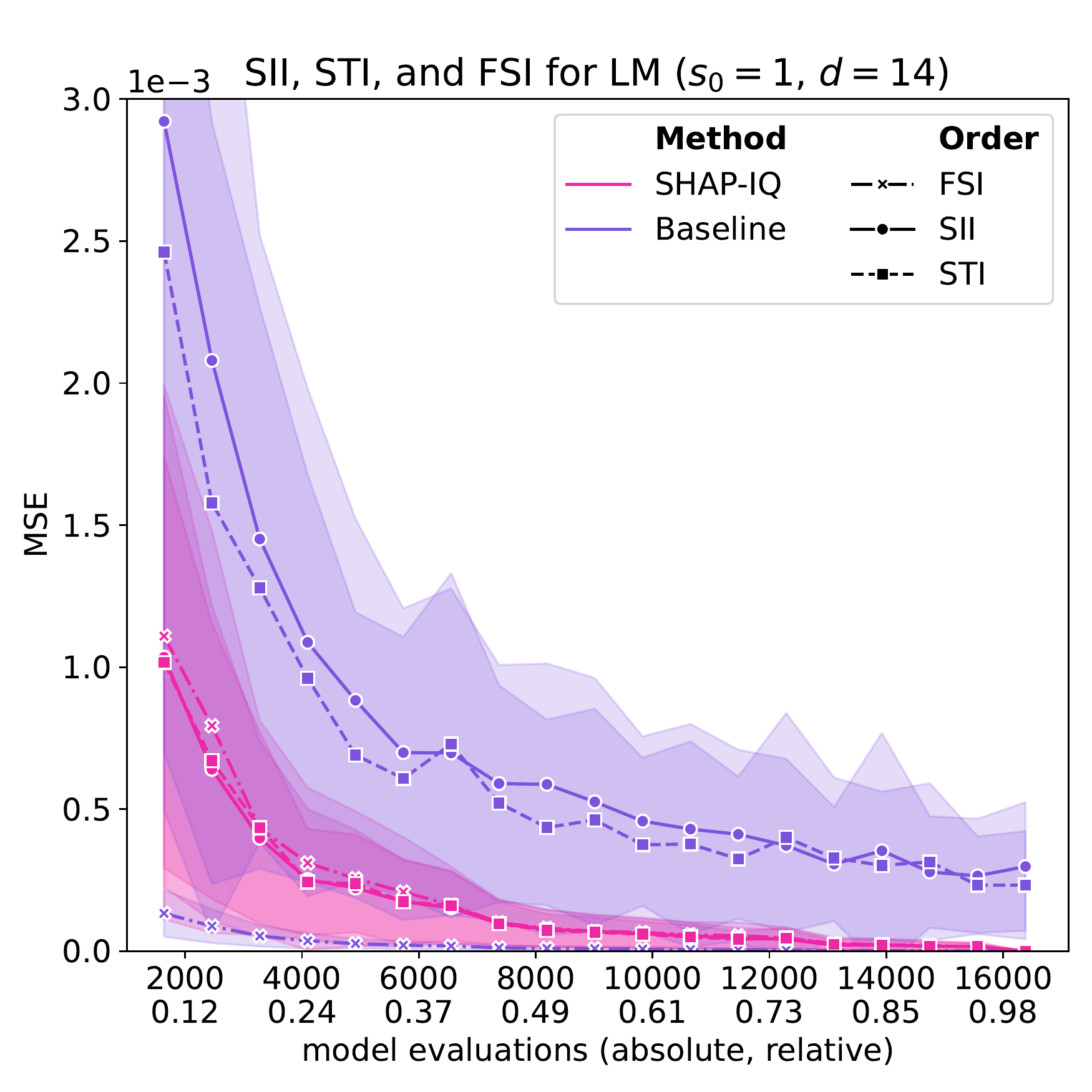}
     \end{minipage}
    \caption{Approximation Quality for LM with interaction order $s_0 = 4$ for $g = 50$ iterations (first row), with interaction order $s_0 = 3$ for $g = 50$ iterations (second row), with interaction order $s_0 = 2$ for $g = 50$ iterations (third row), and with interaction order $s_0 = 1$ (Shapley Value) for $g = 50$ iterations (fourth row).}
    \label{fig:app_language}
\end{figure}

\begin{figure}
    \centering
    \begin{minipage}[c]{0.32\textwidth}
         \includegraphics[width=\textwidth]{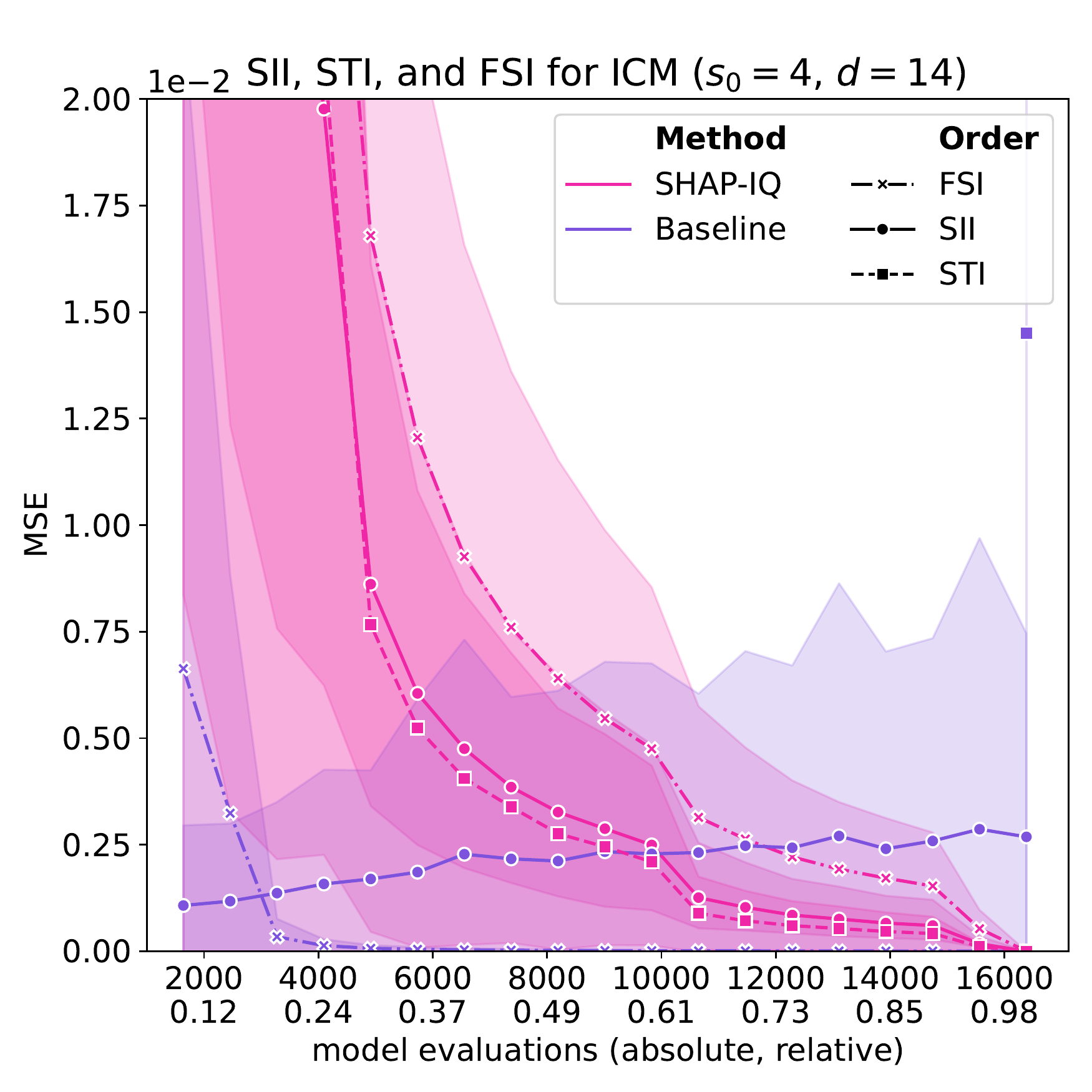}
     \end{minipage}
     \begin{minipage}[c]{0.32\textwidth}
         \includegraphics[width=\textwidth]{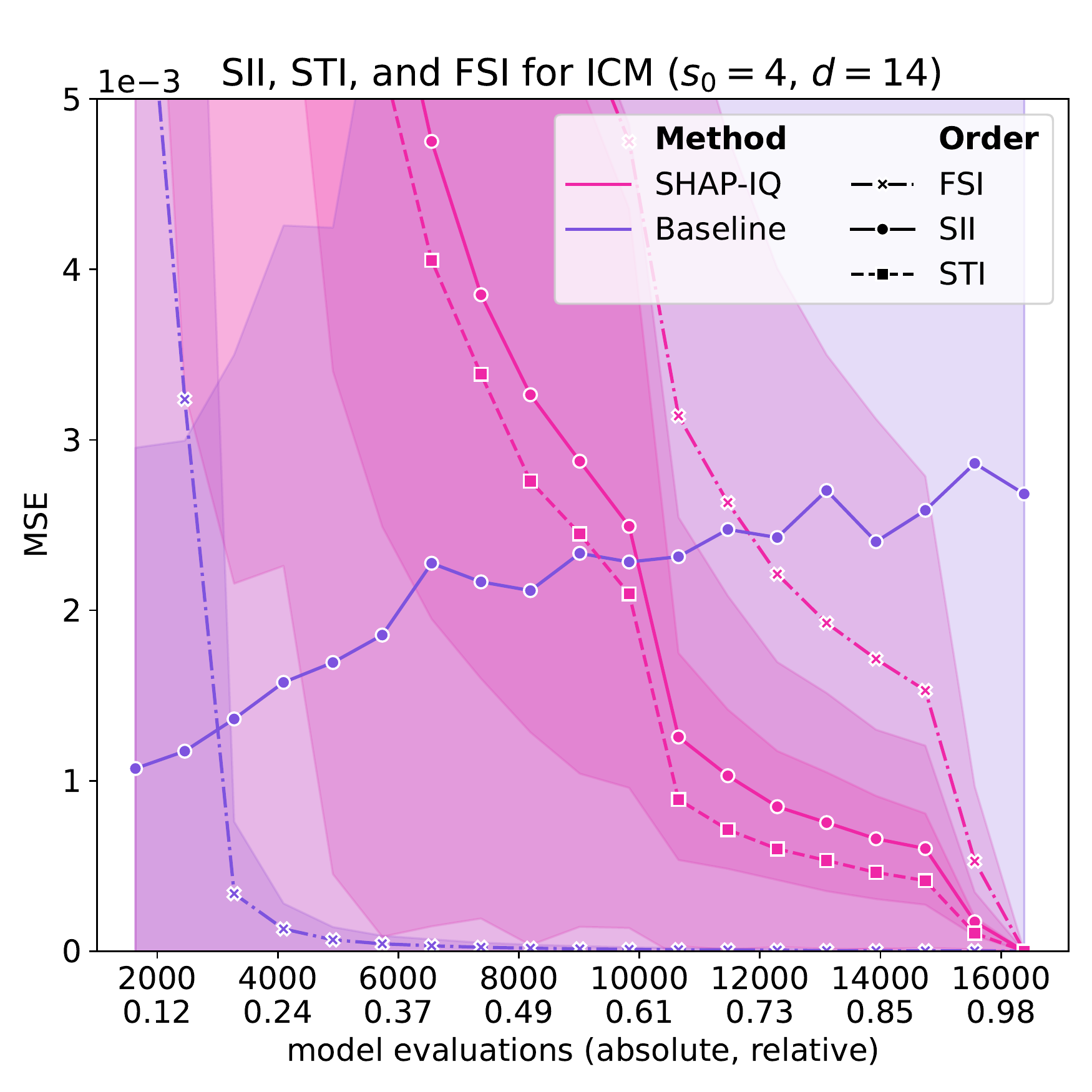}
     \end{minipage}
     \begin{minipage}[c]{0.32\textwidth}
         \includegraphics[width=\textwidth]{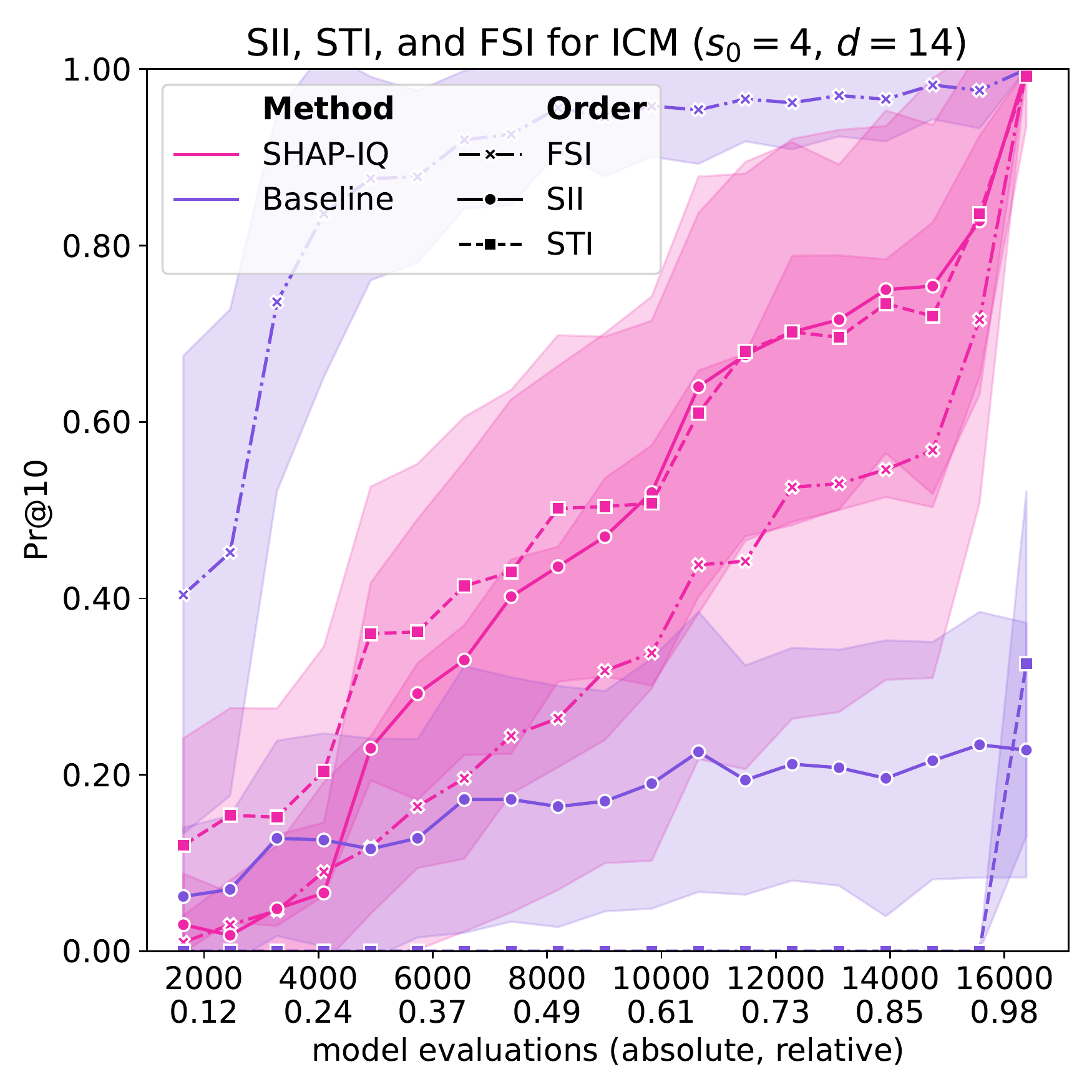}
     \end{minipage}
     \\
     \begin{minipage}[c]{0.32\textwidth}
         \includegraphics[width=\textwidth]{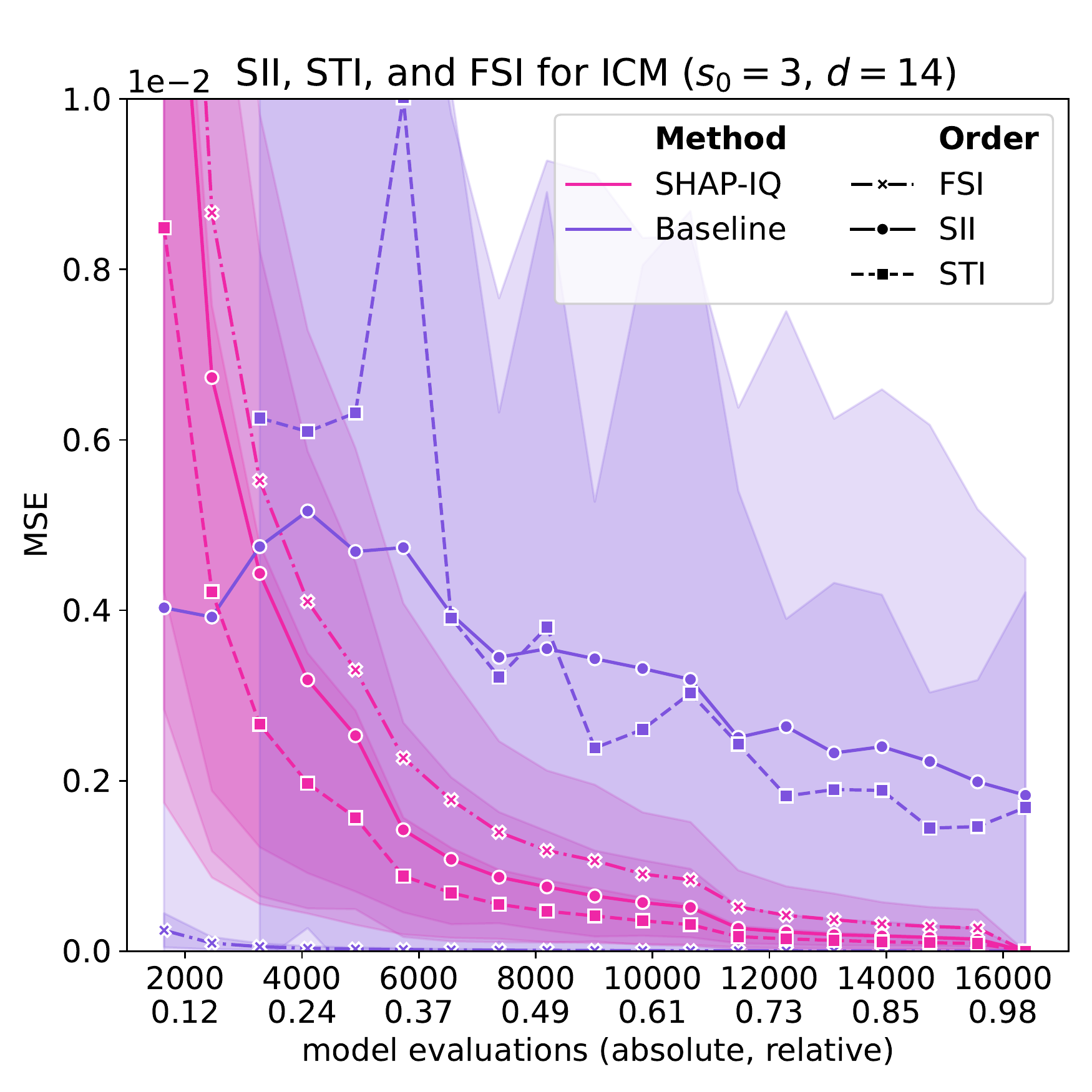}
     \end{minipage}
     \begin{minipage}[c]{0.32\textwidth}
         \includegraphics[width=\textwidth]{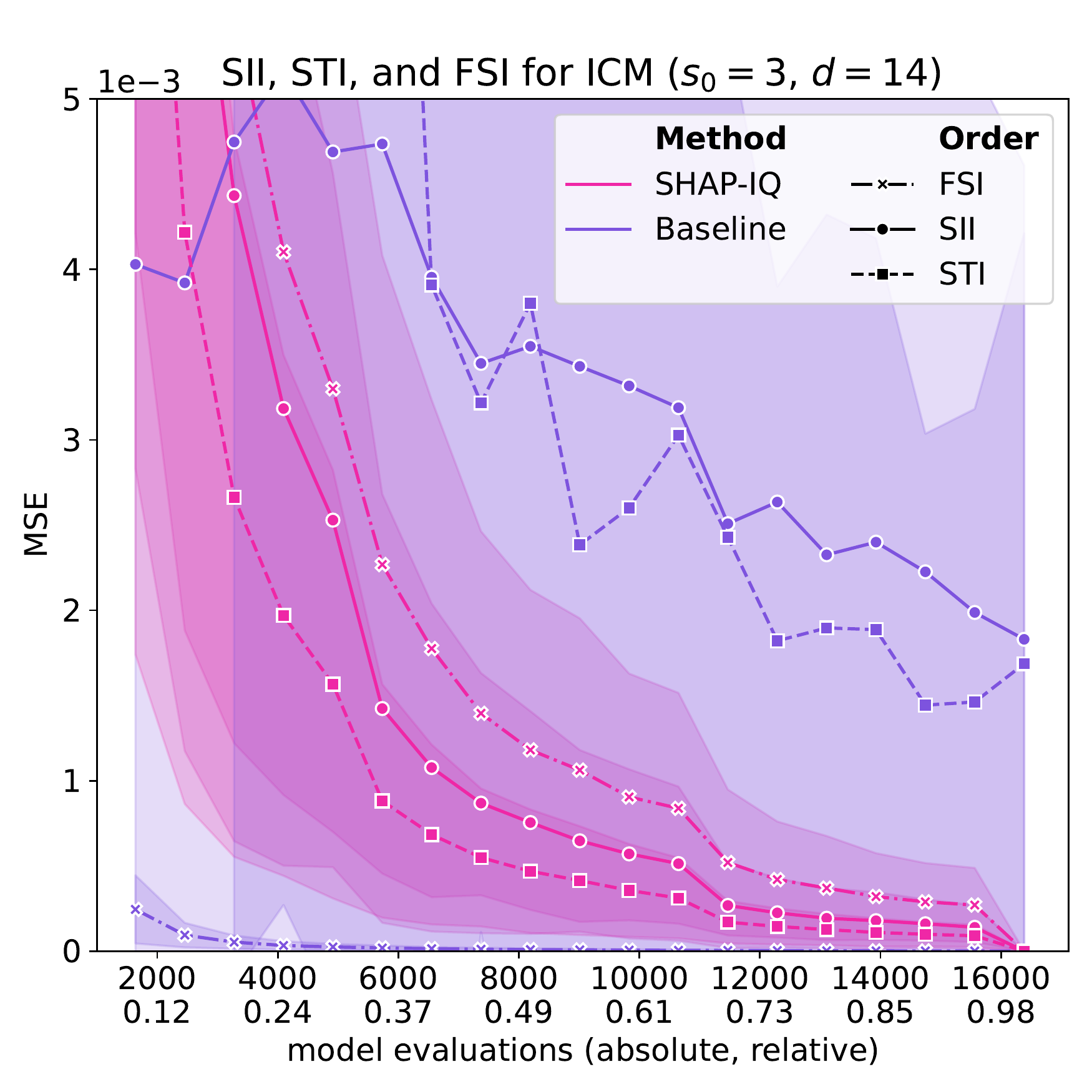}
     \end{minipage}
     \begin{minipage}[c]{0.32\textwidth}
         \includegraphics[width=\textwidth]{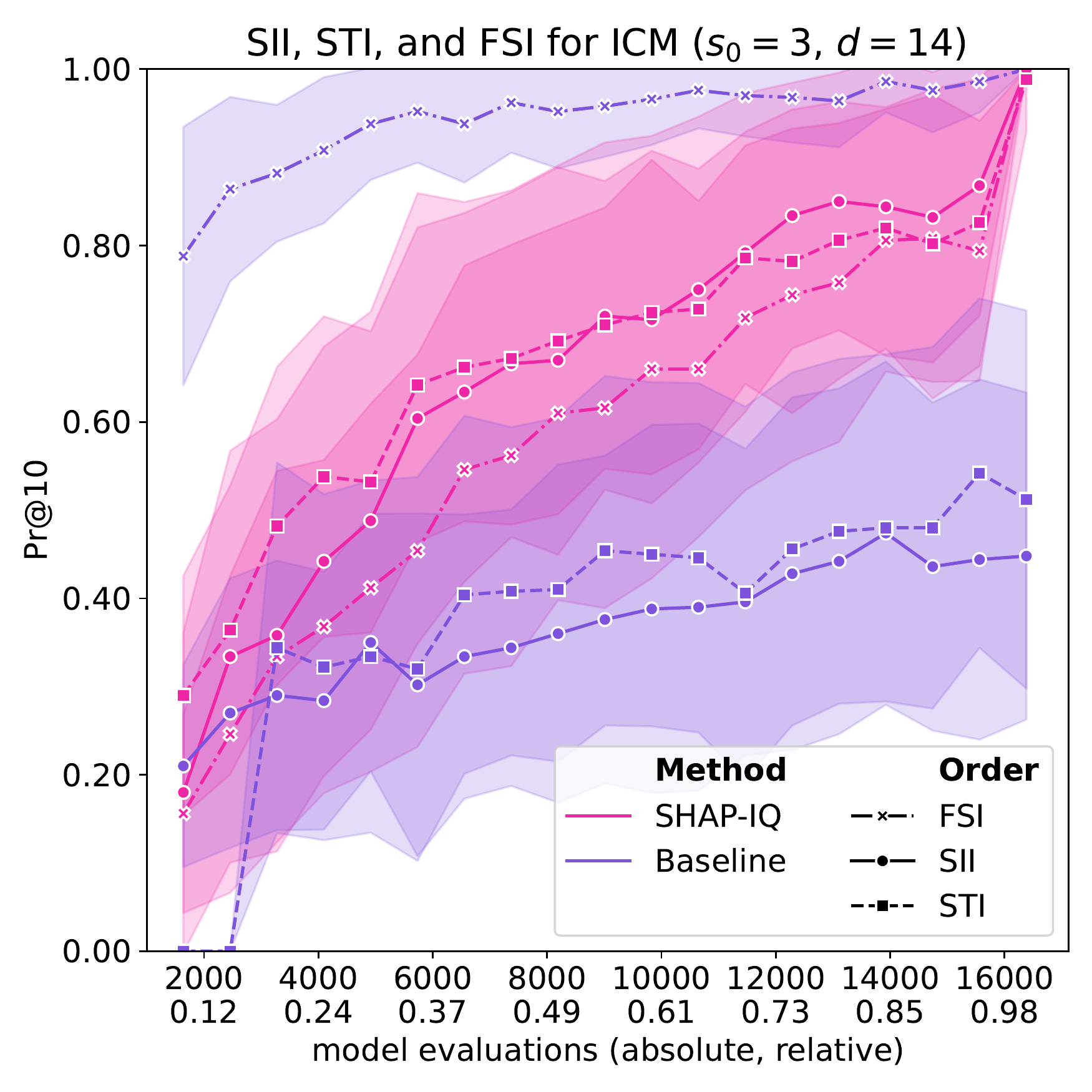}
     \end{minipage}
     \\
     \begin{minipage}[c]{0.32\textwidth}
         \includegraphics[width=\textwidth]{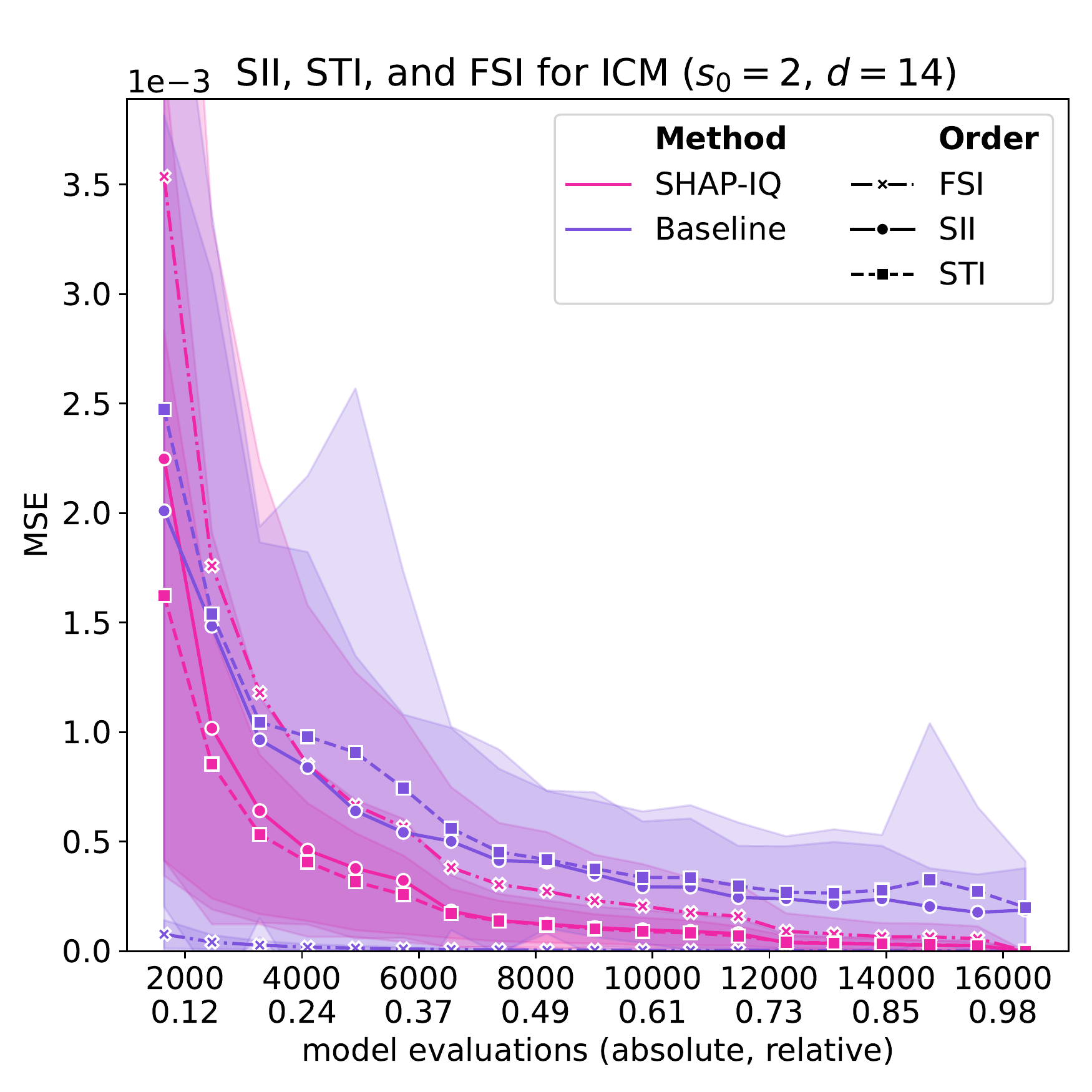}
     \end{minipage}
     \begin{minipage}[c]{0.32\textwidth}
         \includegraphics[width=\textwidth]{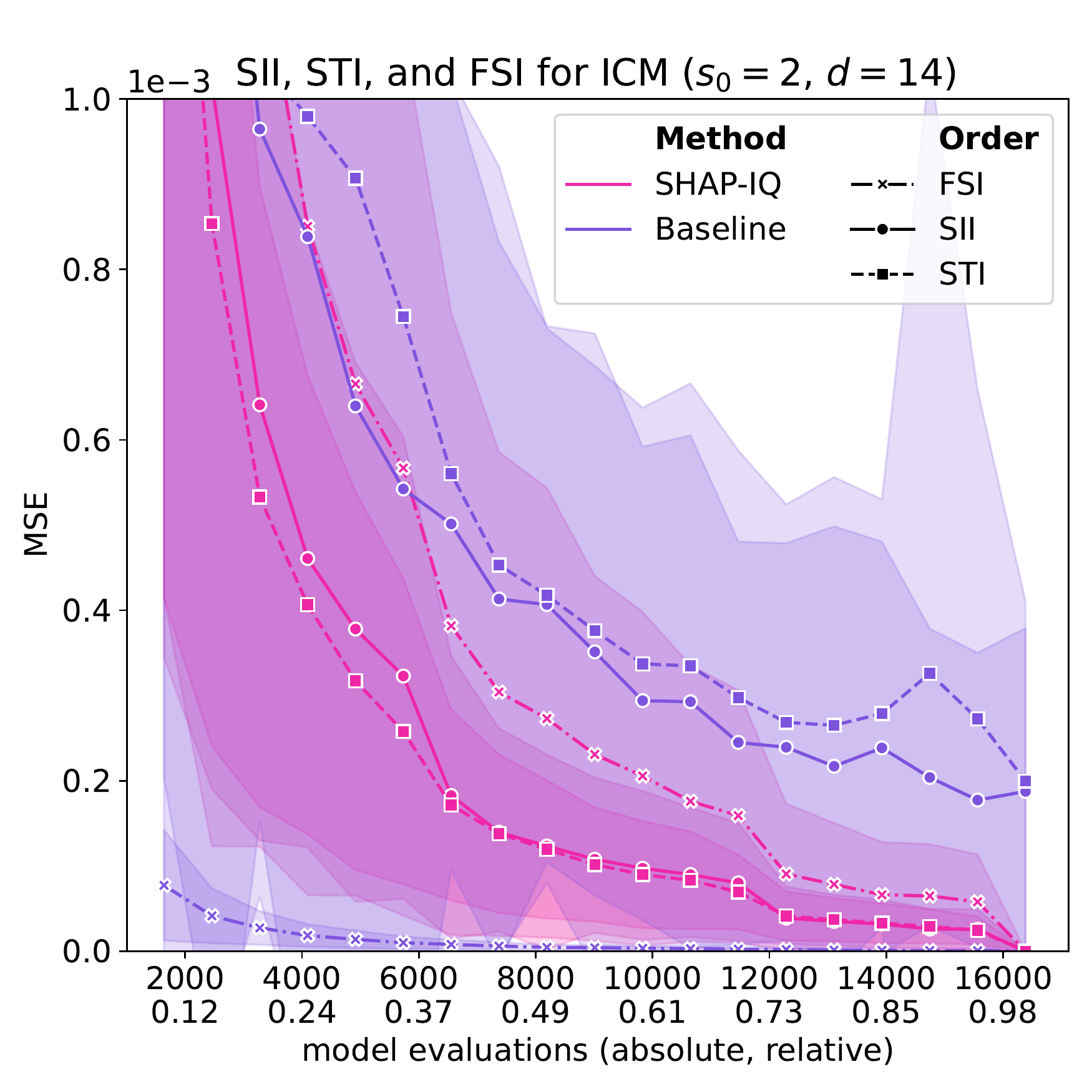}
     \end{minipage}
     \begin{minipage}[c]{0.32\textwidth}
         \includegraphics[width=\textwidth]{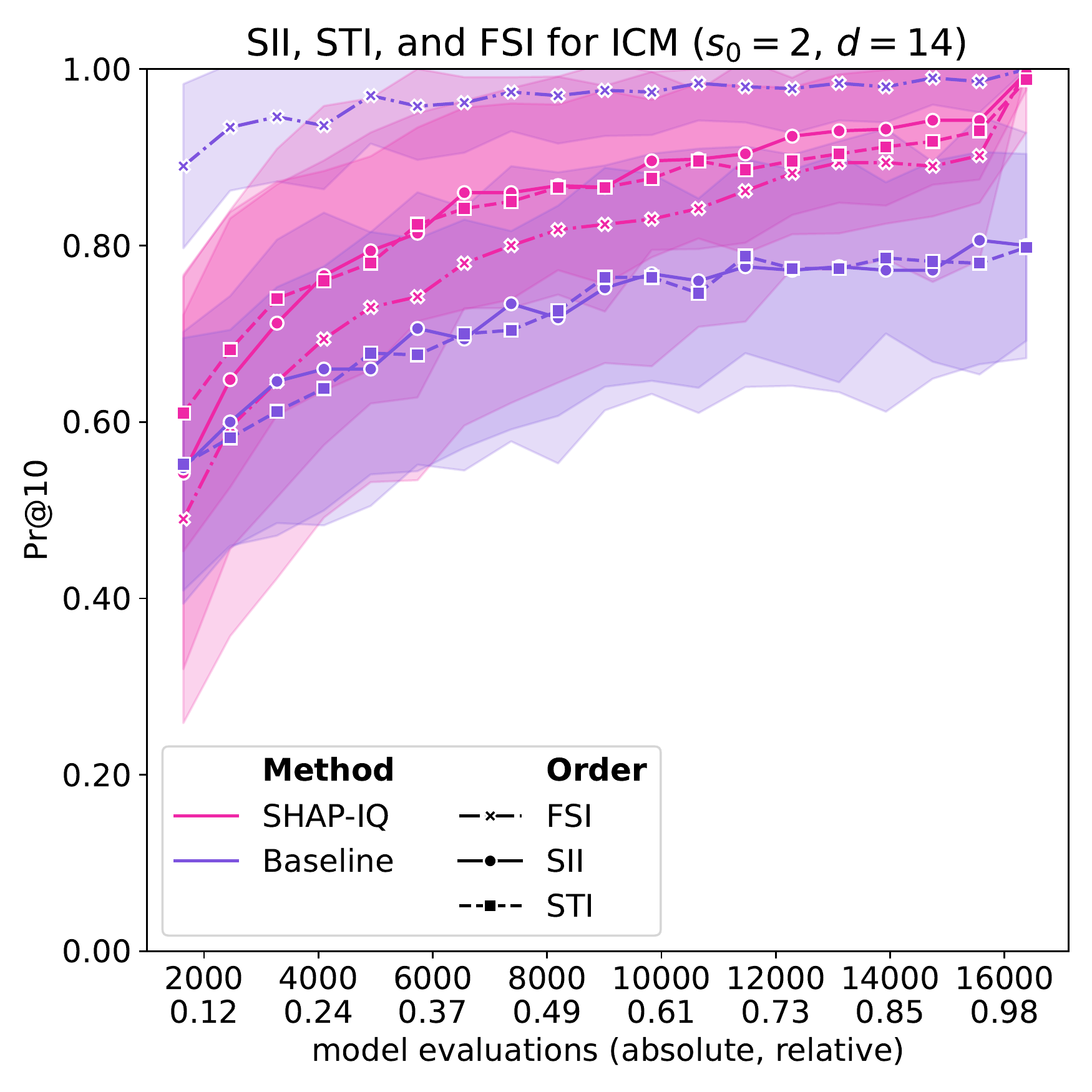}
     \end{minipage}
     \\
     \begin{minipage}[c]{0.32\textwidth}
         \includegraphics[width=\textwidth]{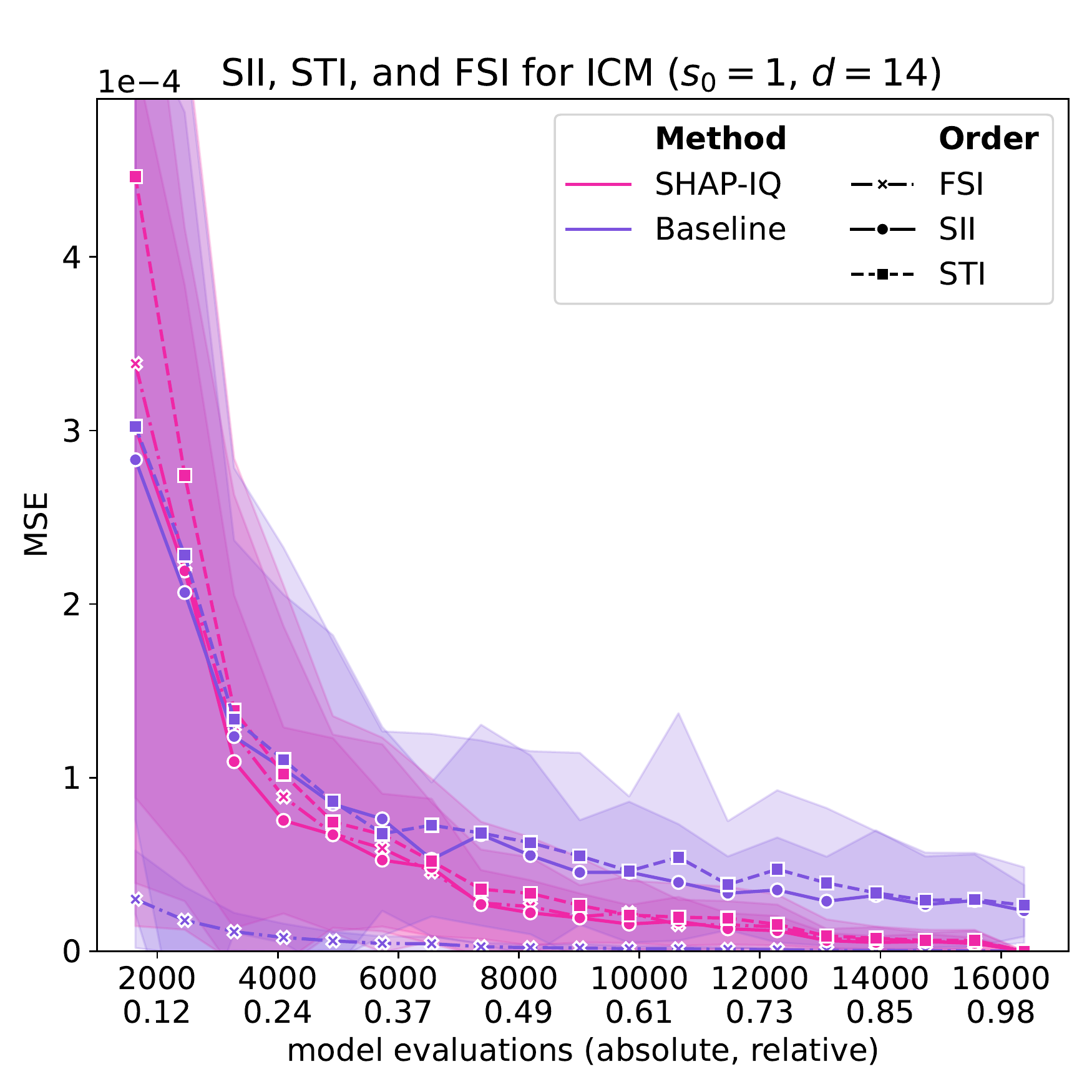}
     \end{minipage}
    \caption{Approximation Quality for ICM with interaction order $s_0 = 4$ for $g = 50$ iterations (first row), with interaction order $s_0 = 3$ for $g = 50$ iterations (second row), with interaction order $s_0 = 2$ for $g = 50$ iterations (third row), and with interaction order $s_0 = 1$ (Shapley Value) for $g = 50$ iterations (fourth row).}
    \label{fig:app_icm}
\end{figure}

\begin{figure}
    \centering
    \begin{minipage}[c]{0.32\textwidth}
         \includegraphics[width=\textwidth]{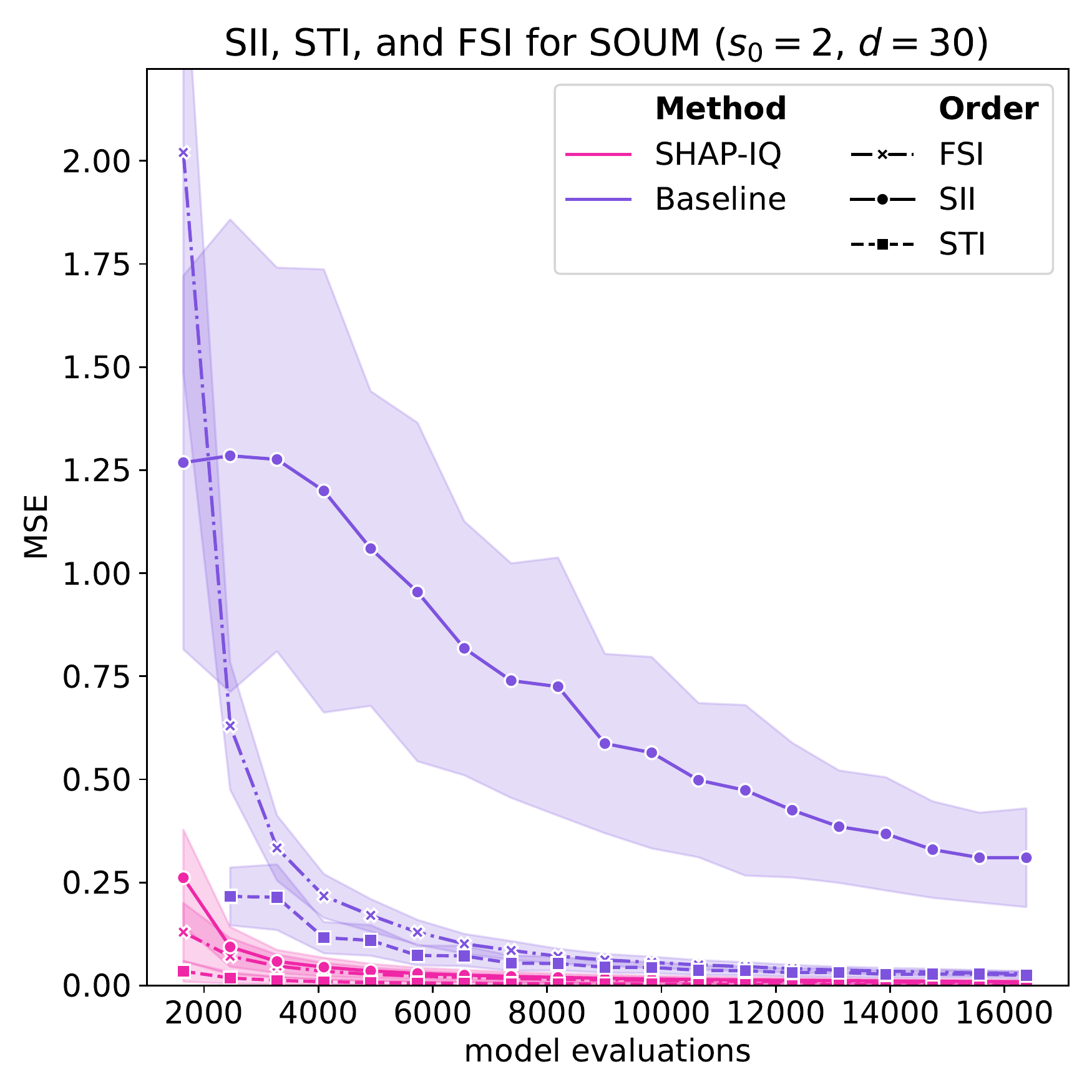}
     \end{minipage}
     \begin{minipage}[c]{0.32\textwidth}
         \includegraphics[width=\textwidth]{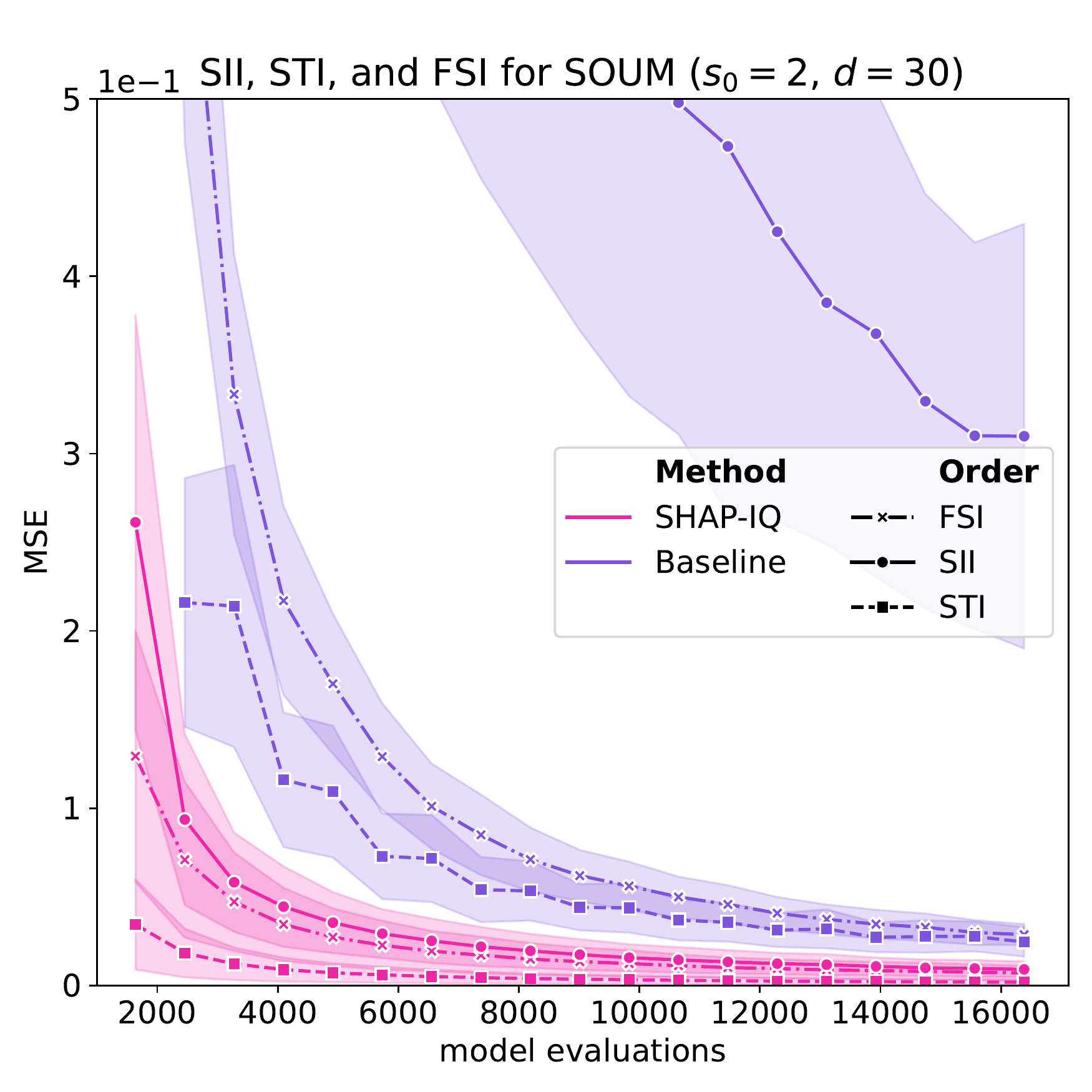}
     \end{minipage}
     \begin{minipage}[c]{0.32\textwidth}
         \includegraphics[width=\textwidth]{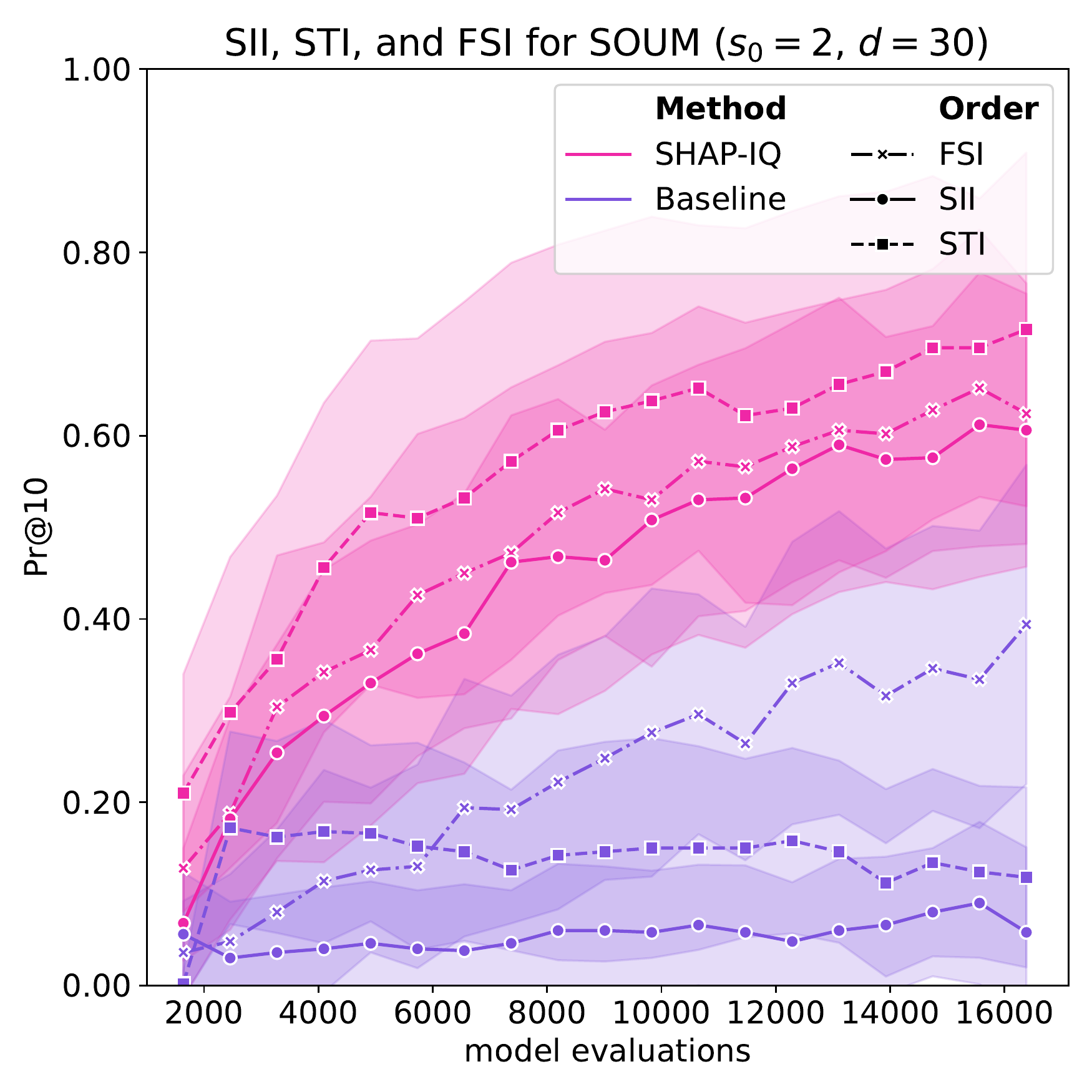}
     \end{minipage}
     \\
     \begin{minipage}[c]{0.32\textwidth}
         \includegraphics[width=\textwidth]{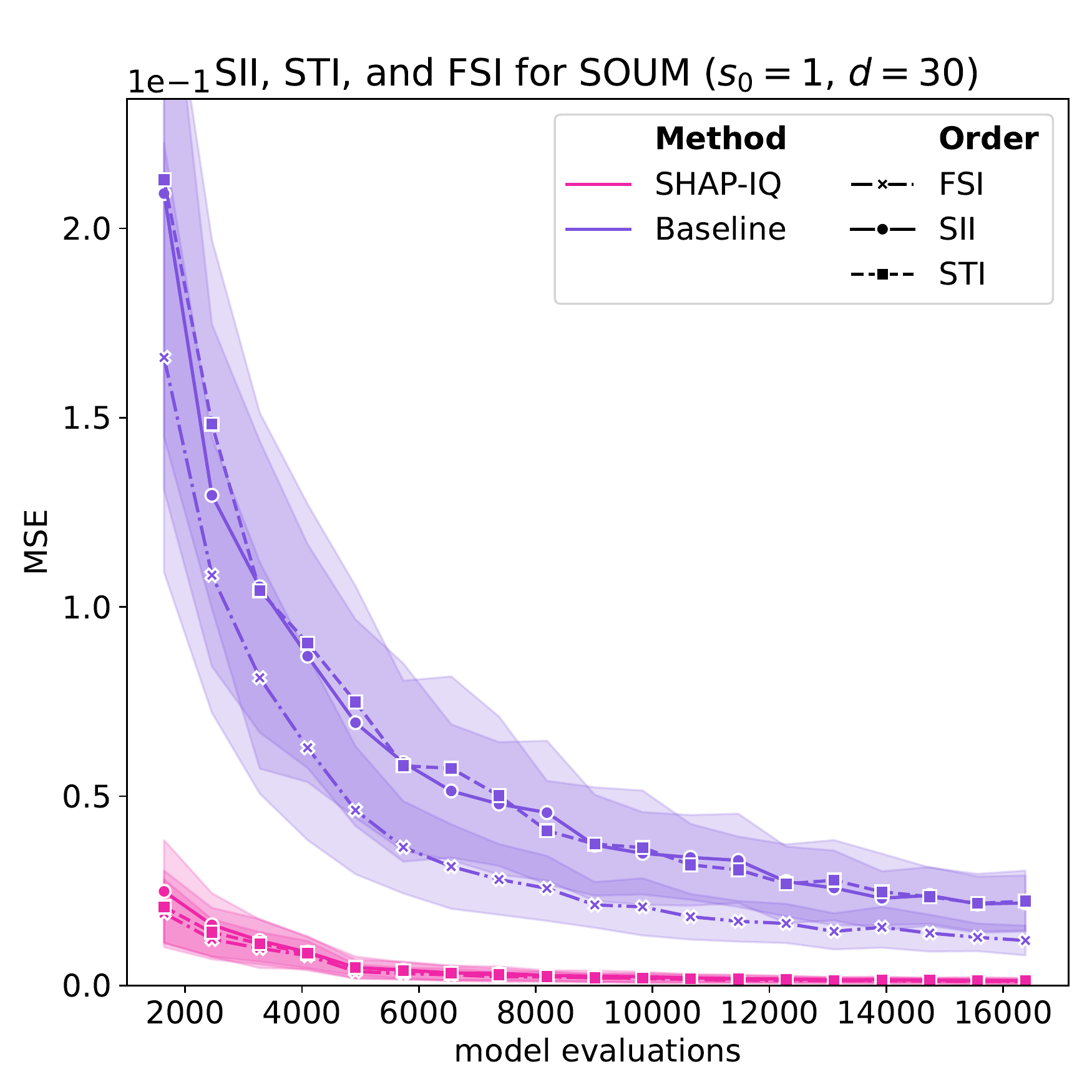}
     \end{minipage}
    \caption{Approximation Quality for SOUM order $s_0 = 2$ (first row) and $s_0 = 1$ (Shapley value, second row) for $g = 50$ iterations on the SOUM with $N = 100$ interactions, $d = 30$ features, and $\ell_{\max} = 30$.}
    \label{fig:app_soum}
\end{figure}

\subsubsection{n-SII Estimation on Example Sentences.}

We further probe the LM with randomly selected sentences from the \emph{IMDB} dataset and estimate the SII scores up to order $4$.
For the example sentence ``\textit{It is a gruesome cannibal movie. But it's not bad. If you like Hannibal, you'll love this.}'' in Section~\ref{sec_experiments_different_ciis} all orders of n-SII with $s_0=1,2,3,4$ are illustrated in Figure~\ref{fig:app_n_SII_1}.
Another example sentence ``\textit{I have never forgot this movie. All these years and it has remained in my life.}'' with all orders of n-SII up to the maximum interaction order $s_0=1,2,3,4$ is shown in Figure~\ref{fig:app_n_SII_2}.
Lastly, for four sentences n-SII for all orders with $s_0=4$ are visualized in Figure~\ref{fig:app_n_SII_3}.

\begin{figure}
    \centering
    \begin{minipage}[c]{0.47\textwidth}
         \includegraphics[width=\textwidth]{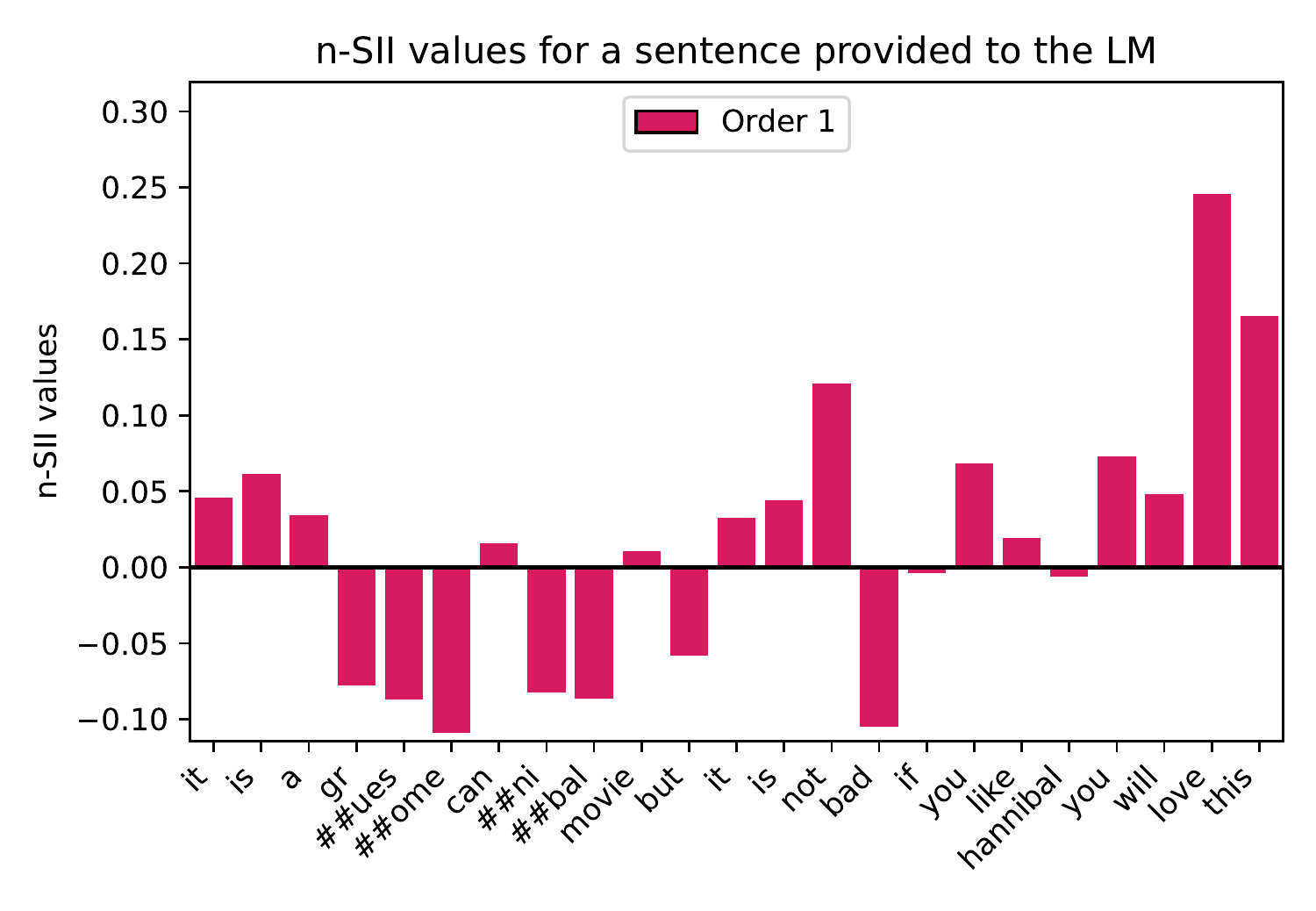}
     \end{minipage}
     \begin{minipage}[c]{0.47\textwidth}
         \includegraphics[width=\textwidth]{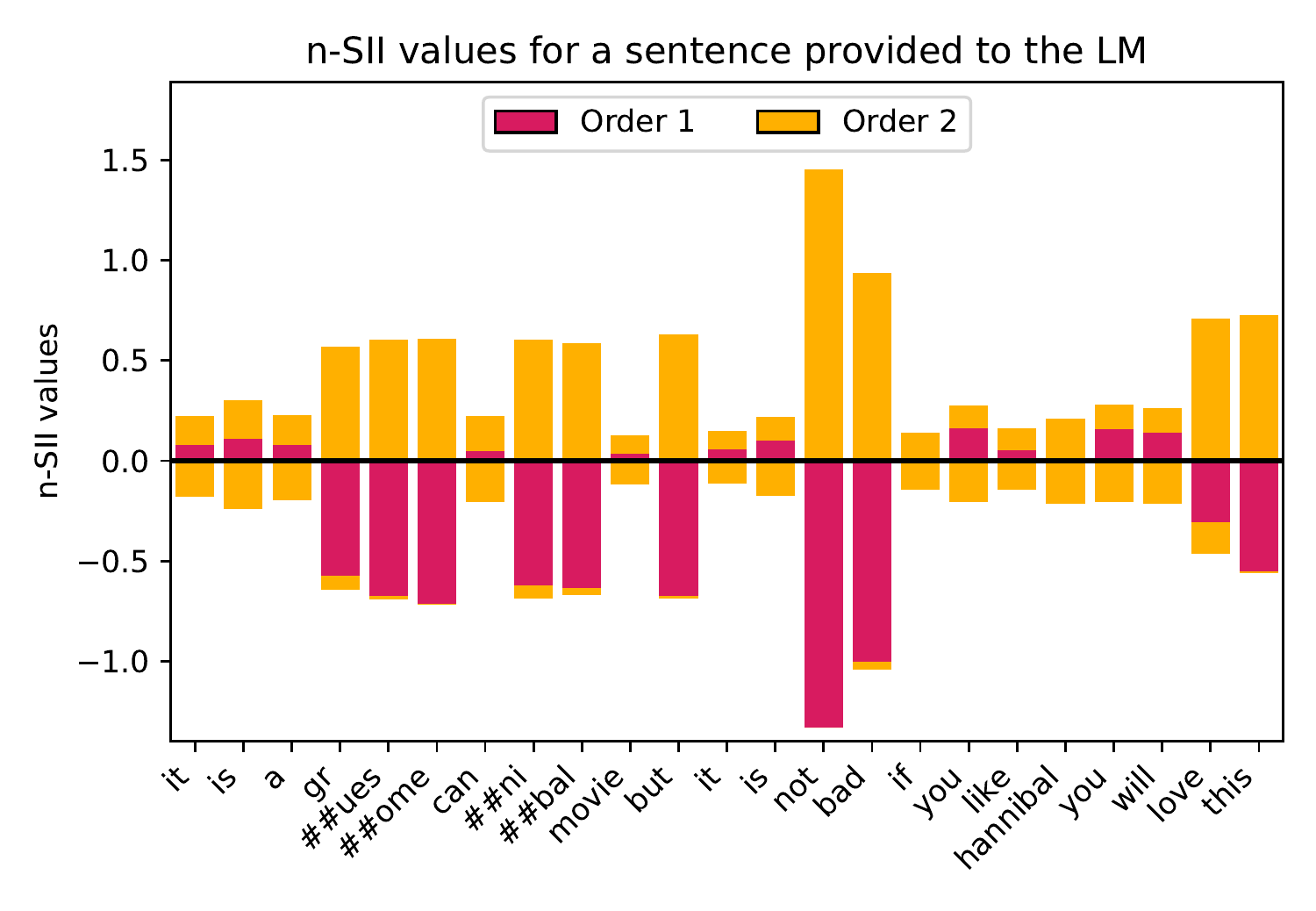}
     \end{minipage}
     \\
     \begin{minipage}[c]{0.47\textwidth}
         \includegraphics[width=\textwidth]{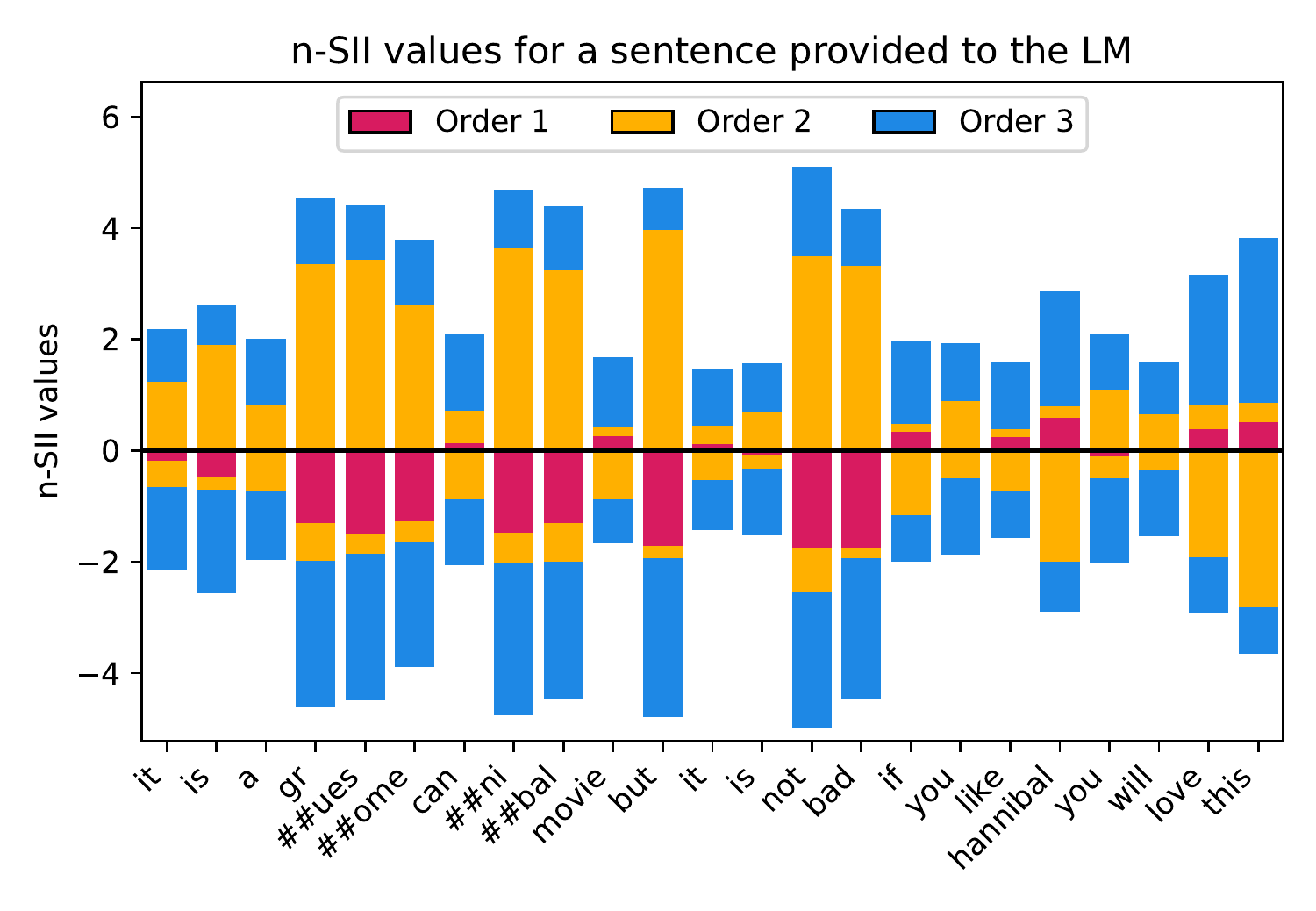}
     \end{minipage}
     \begin{minipage}[c]{0.47\textwidth}
         \includegraphics[width=\textwidth]{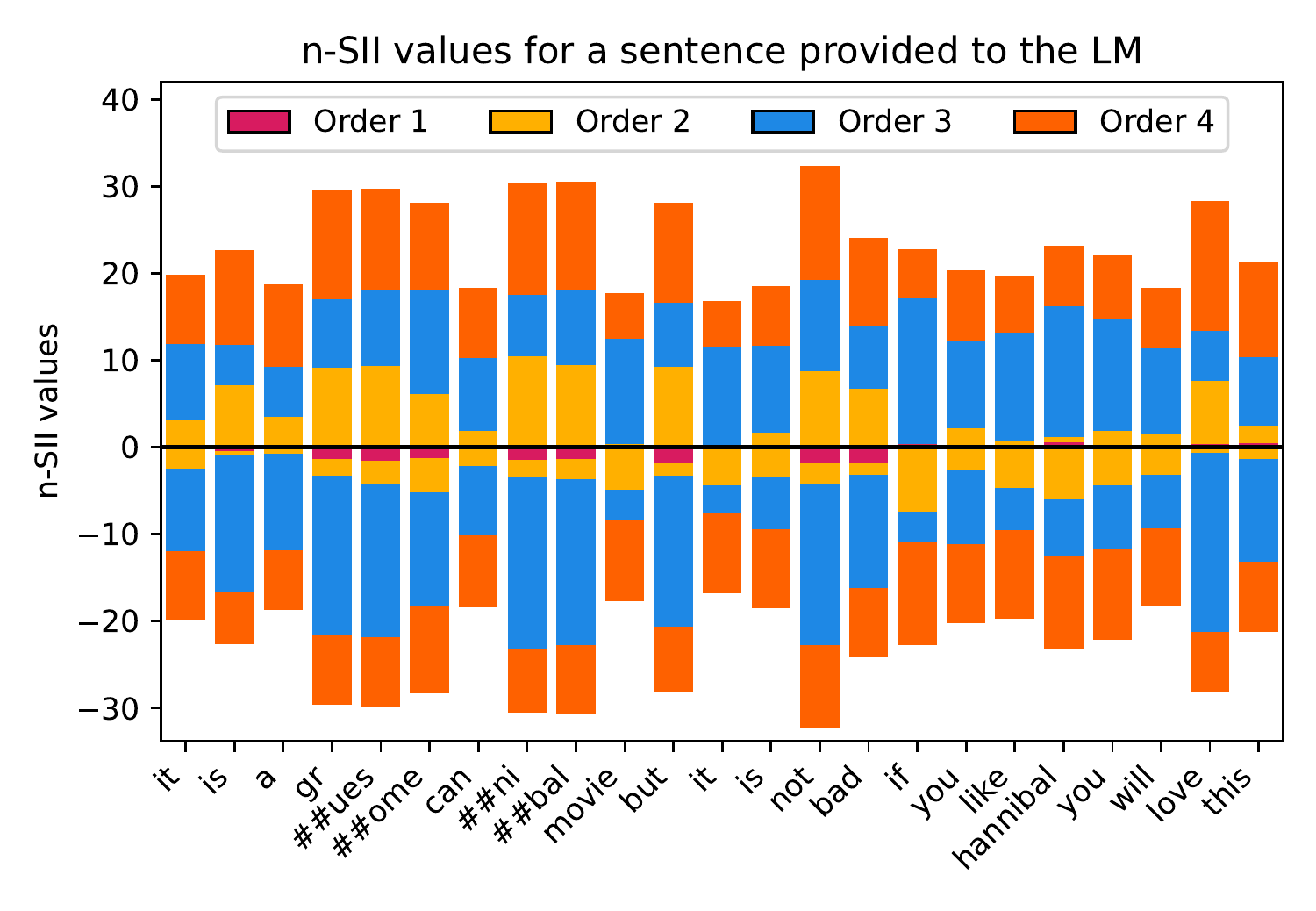}
     \end{minipage}
    \caption{Estimated SII values with orders $s=1,2,3,4$ for the sentence ``\textit{It is a gruesome cannibal movie. But it's not bad. If you like Hannibal, you'll love this.}'' ($d = 23$) provided to the LM. The plots show all orders of n-SII with maximum interaction order $s_0 = 1$ (top left), $s_0 = 2$ (top right), $s_0 = 3$ (bottom left), and $s_0 = 4$ (bottom right).}
    \label{fig:app_n_SII_1}
\end{figure}

\begin{figure}
    \centering
    \begin{minipage}[c]{0.47\textwidth}
         \includegraphics[width=\textwidth]{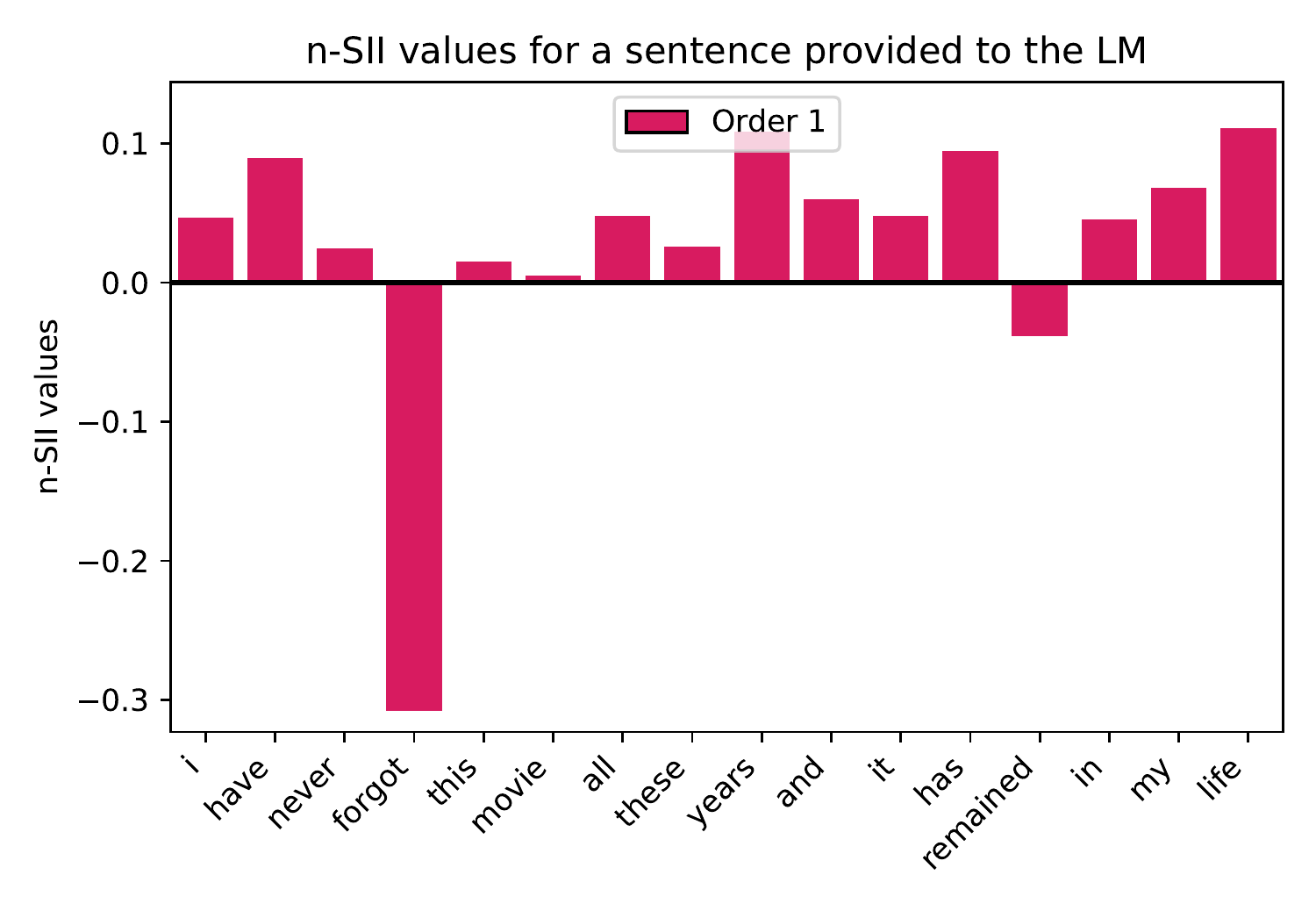}
     \end{minipage}
     \begin{minipage}[c]{0.47\textwidth}
         \includegraphics[width=\textwidth]{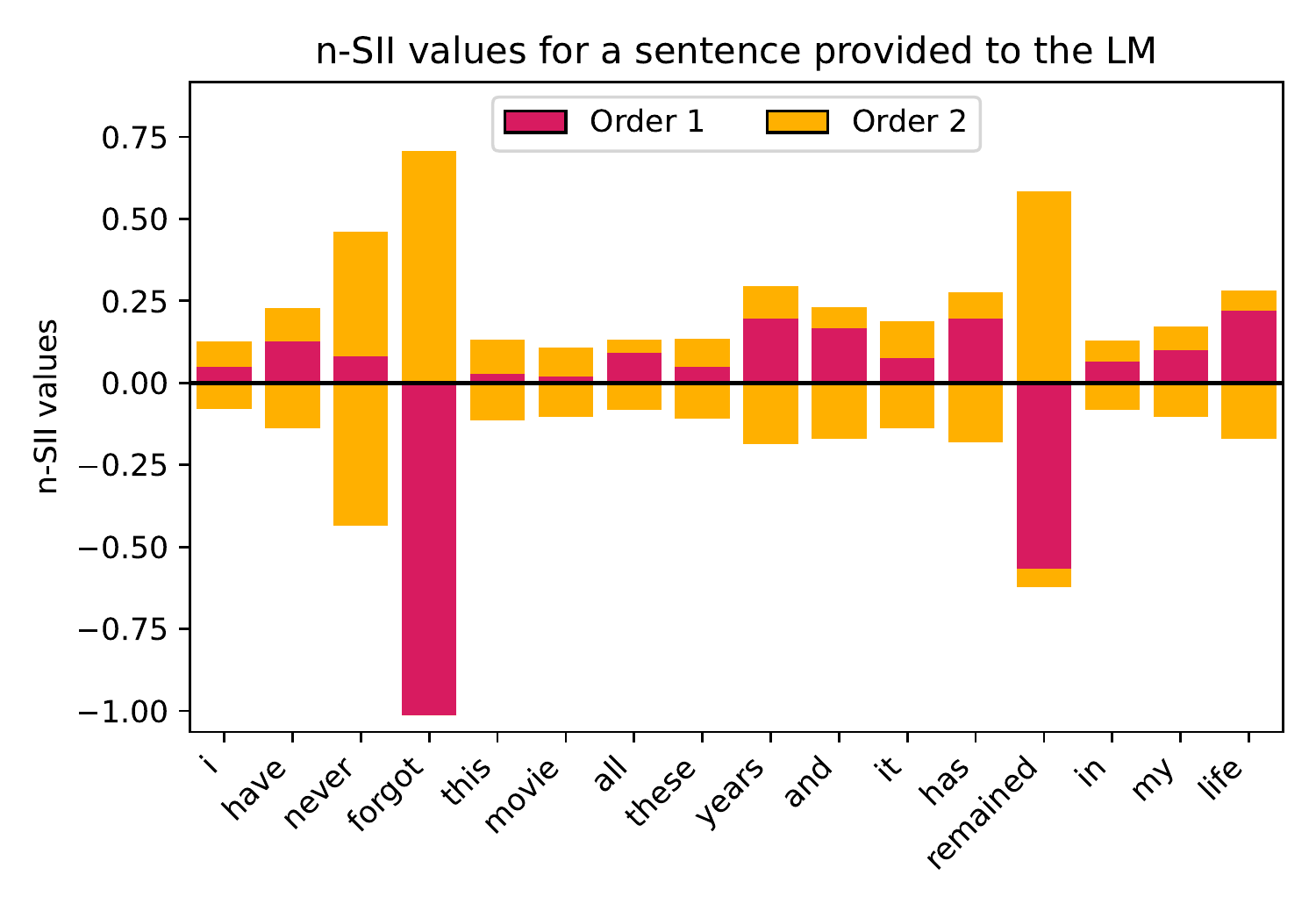}
     \end{minipage}
     \\
     \begin{minipage}[c]{0.47\textwidth}
         \includegraphics[width=\textwidth]{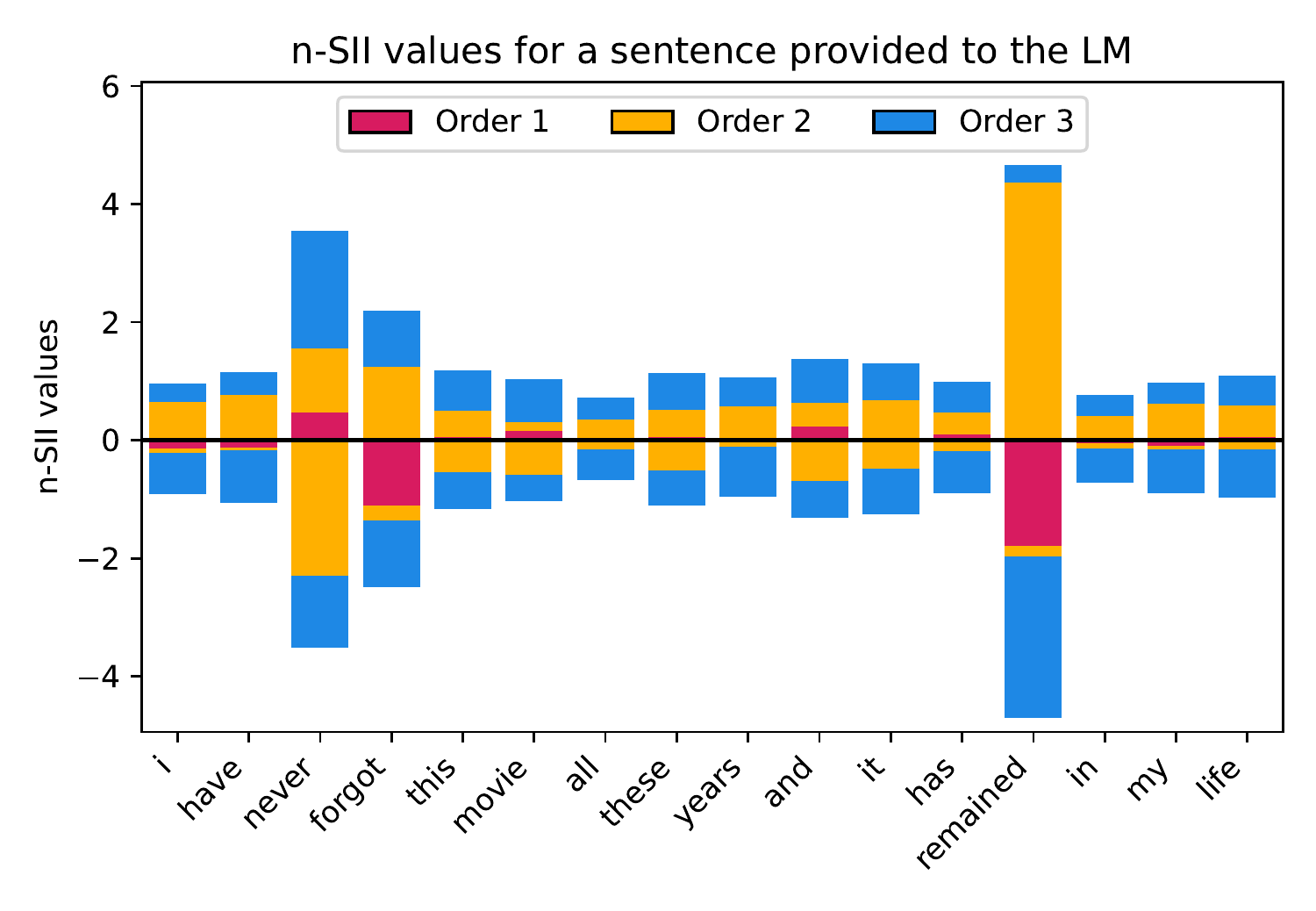}
     \end{minipage}
     \begin{minipage}[c]{0.47\textwidth}
         \includegraphics[width=\textwidth]{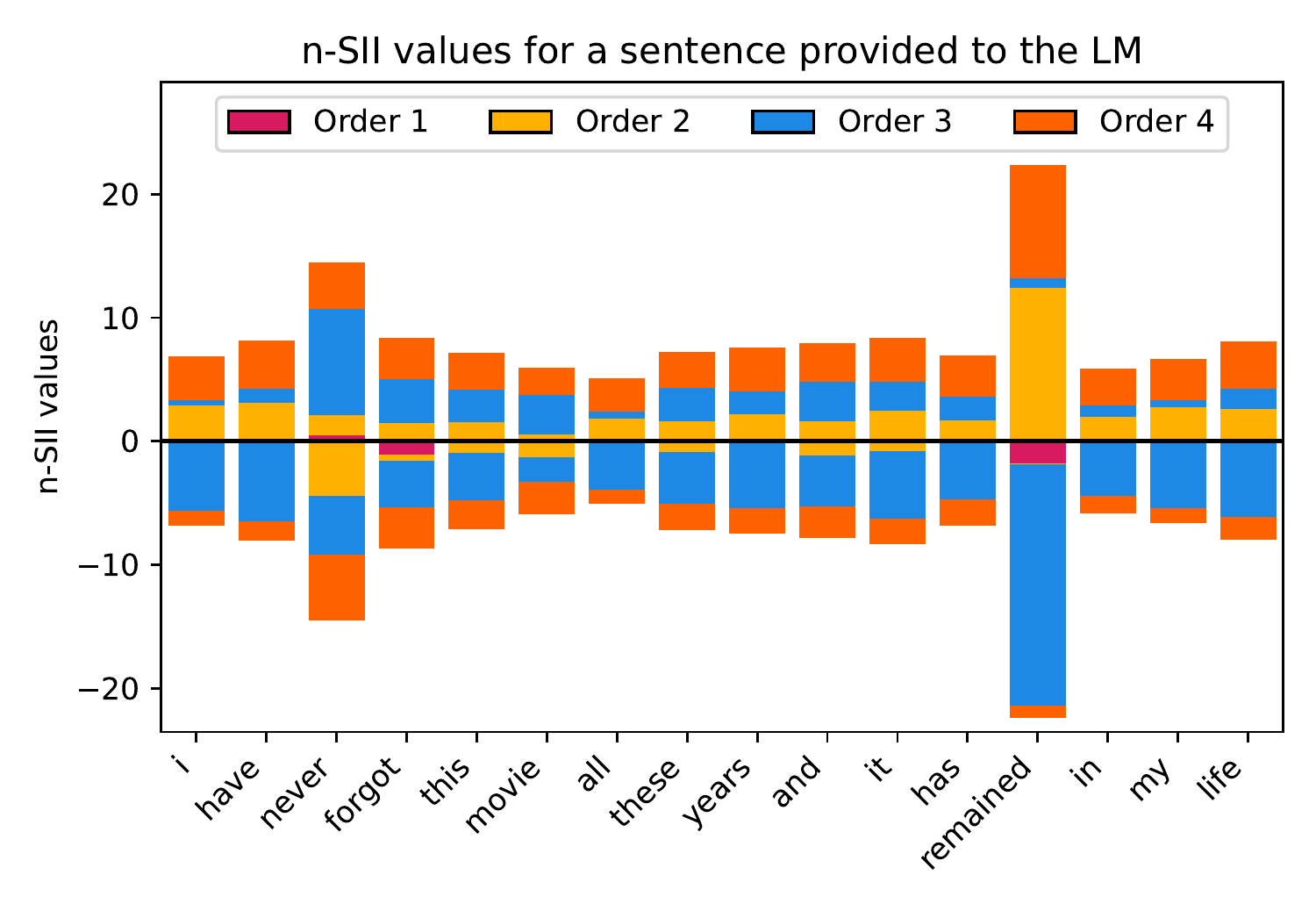}
     \end{minipage}
    \caption{Estimated SII values with orders $s=1,2,3,4$ for the sentence ``\textit{I have never forgot this movie. All these years and it has remained in my life.}'' ($d = 16$) provided to the LM. The plots show all orders of n-SII with maximum interaction order $s_0 = 1$ (top left), $s_0 = 2$ (top right), $s_0 = 3$ (bottom left), and $s_0 = 4$ (bottom right).}
    \label{fig:app_n_SII_2}
\end{figure}

\begin{figure}
    \centering
    \begin{minipage}[c]{0.47\textwidth}
         \includegraphics[width=\textwidth]{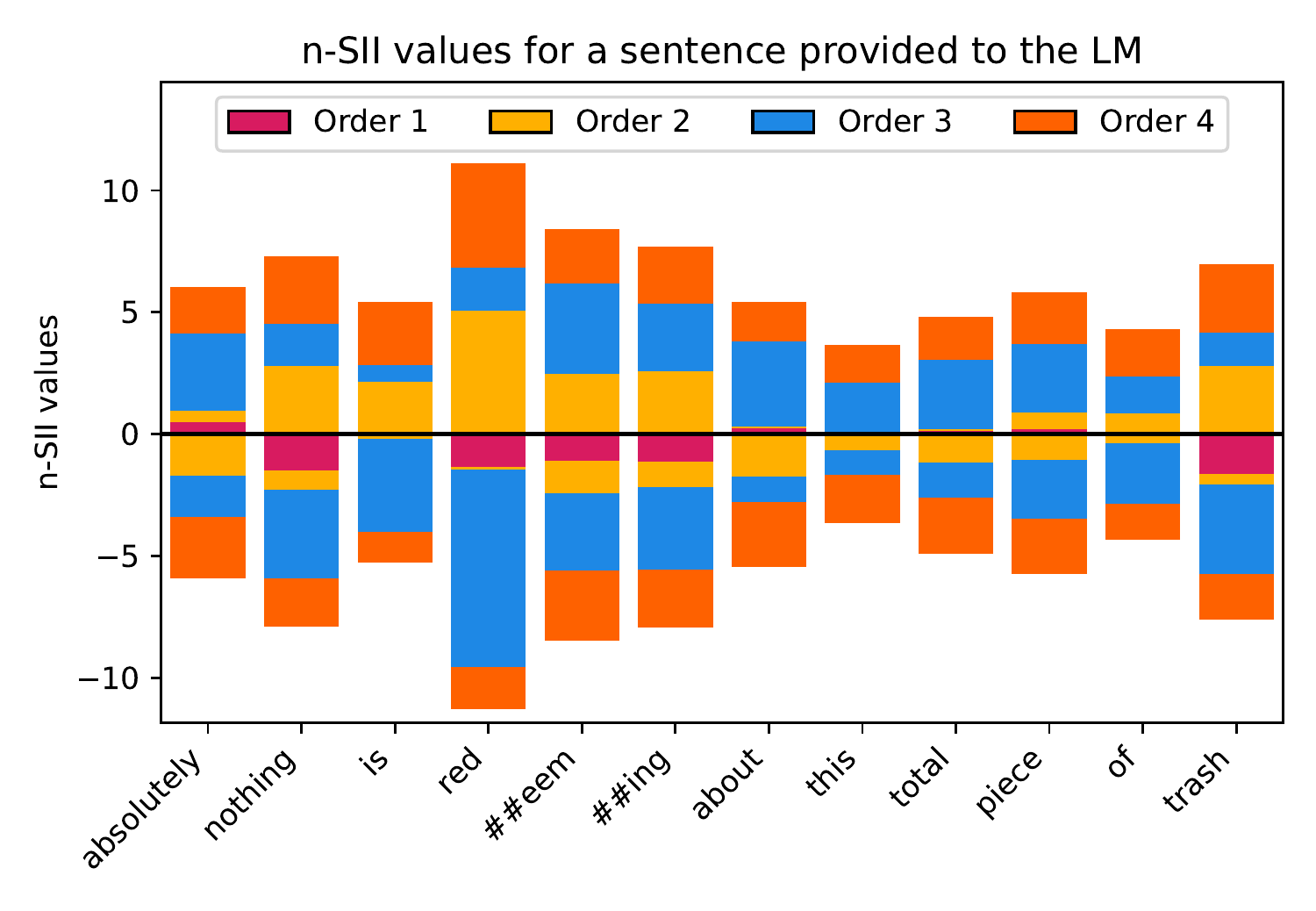}
     \end{minipage}
     \begin{minipage}[c]{0.47\textwidth}
         \includegraphics[width=\textwidth]{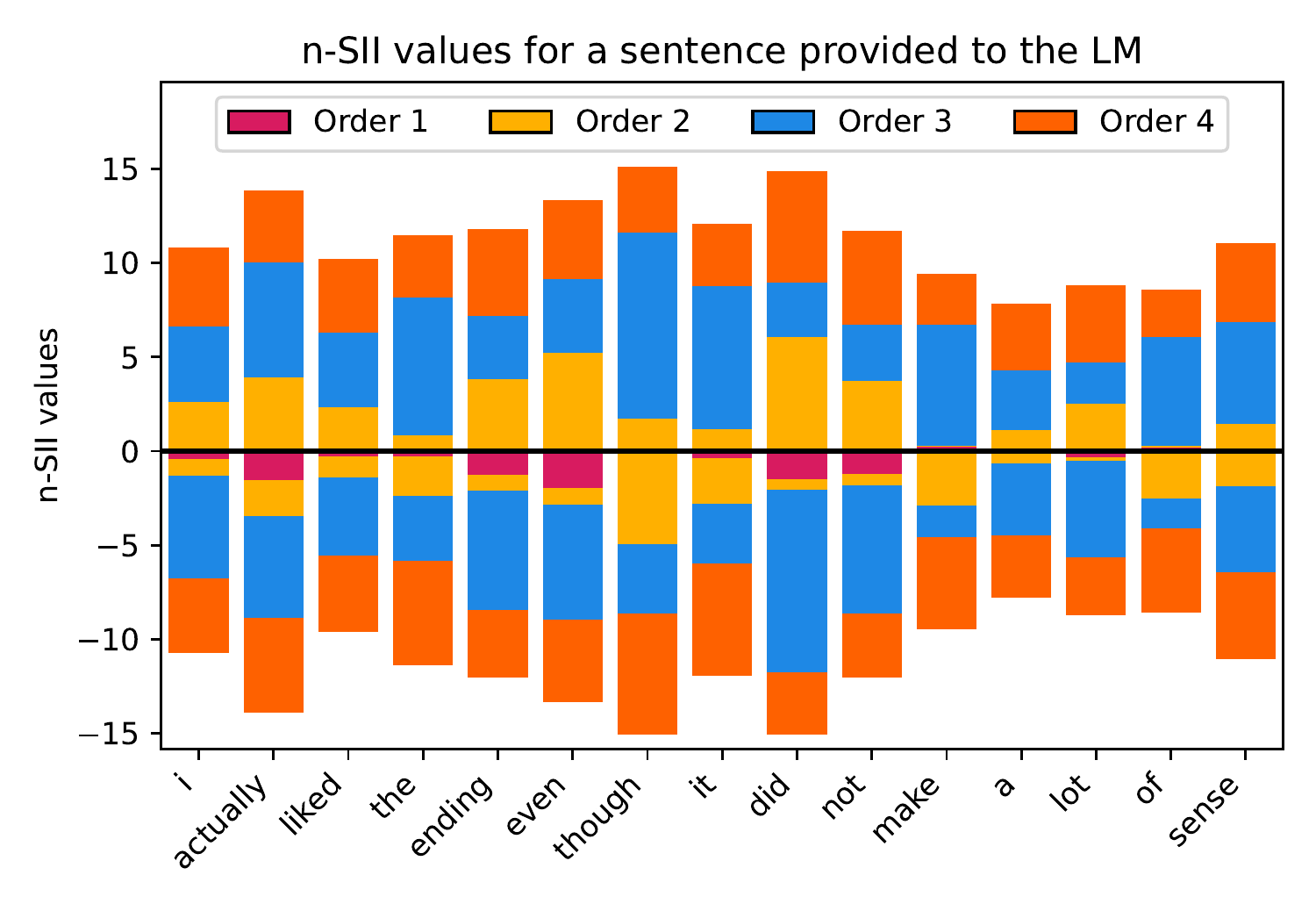}
     \end{minipage}
     \\
     \begin{minipage}[c]{0.47\textwidth}
         \includegraphics[width=\textwidth]{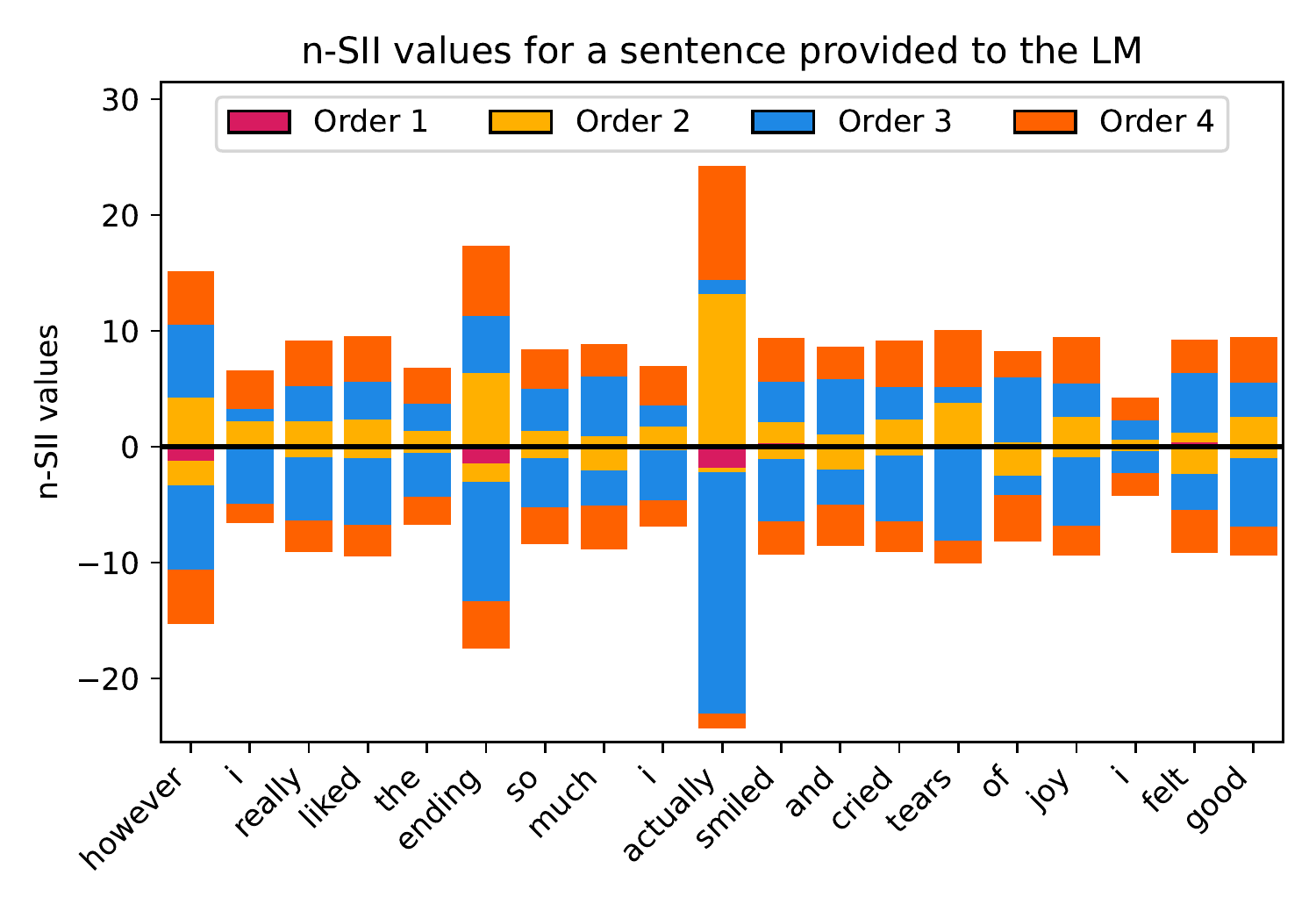}
     \end{minipage}
     \begin{minipage}[c]{0.47\textwidth}
         \includegraphics[width=\textwidth]{figures/appendix/n_SII_sentence_it_is_a_gr_10000_order-4.pdf}
     \end{minipage}
    \caption{Estimated n-SII values with $s_0=4$ for the sentences provided to the LM: ``\textit{Absolutely nothing is redeeming about this total piece of trash.}'' ($d=12$, top left), ``\textit{I actually liked the ending even though it did not make a lot of sense.}'' ($d=15$, top right), ``\textit{However, I really liked the ending so much, I actually smiled and cried tears of joy. I felt good.}'' with ($d=19$, bottom left), and ``\textit{It is a gruesome cannibal movie. But it's not bad. If you like Hannibal, you'll love this.}'' ($d=23$, bottom right).}
    \label{fig:app_n_SII_3}
\end{figure}

\clearpage
\section{Further theoretical Results for the Shapley Value}\label{appx::sv_theory}
In this section, we give two additional results for the special case of the SV.
We explicitly state the inverse of the covariance matrix from \cite{Covert_Lee_2021} and further present a simplified representation of SHAP-IQ, if applied on the SV, which aligns with Theorem~\ref{thm::SV_representation}.

\subsection{Explicit Inverse of Covariance Matrix of Unbiased KernelSHAP}\label{appx::inverse-uksh}

The covariance matrix has been explicitly computed in \cite[Appendix A]{Covert_Lee_2021}.
In this section, we provide the explicit form of the inverse $A^{-1}$.
This inverse admits the same structure and is a central element in the explicit calculation of U-KSH as a weighted sum, which is linked to SHAP-IQ.

\begin{proposition}[Explicit covariance matrix \cite{Covert_Lee_2021}]\label{appx::prop-cov-matrix}
    For the covariance matrix it holds
    \begin{align*}
        A := \mathbb{E}[ZZ^T] = \mu_2 \mathbf{J} + (\mu_1 - \mu_2) \mathbf{I},
    \end{align*}
    with $i,j \in \fset$ and constants
    \begin{align*}
     &A_{ii} := \mu_1 := \mathbb{P}(Z_i = 1) = \frac 1 2
     &A_{ij} := \mu_2 := \mathbb{P}(Z_i=Z_j=1) = \frac{1}{\fnum(\fnum-1)} \frac{\sum_{k=2}^{\fnum-1}\frac{k-1}{\fnum-k}}{\sum_{k=1}^{\fnum-1}\frac{1}{k(\fnum-k)}}.
    \end{align*}
\end{proposition}
\begin{proof}
    The proof is given in \cite[Appendix A]{Covert_Lee_2021}.
\end{proof}
We have for $\sum_{k=1}^{\fnum-1}\frac{1}{k(\fnum-k)} = 2 h_{\fnum-1}$, where $h_n$ is the n-th harmonic number.
Furthermore $\mu_1-\mu_2 = \frac{1}{2 h_{\fnum-1}}$.
We now give the explicit form of $A^{-1}$.

\begin{proposition}[Explicit inverse of covariance matrix]
Let $A := \mathbb{E}[ZZ^T]$. Then, we have an explicit form of the inverse $A^{-1}$ as 
\begin{equation*}
    A^{-1} = \tilde \mu_2 \mathbf{J} + (\tilde\mu_1-\tilde\mu_2) \mathbf{I} 
\end{equation*}
with constants 
\begin{align*}
    &(A^{-1})_{ii} := \tilde\mu_1 = 2h_{\fnum-1}\frac{\mu_1+(\fnum-2)\mu_2}{\mu_1+(\fnum-1)\mu_2} 
    &(A^{-1})_{ij} := \tilde\mu_2 = 2h_{\fnum-1}\frac{-\mu_2}{\mu_1+(\fnum-1)\mu_2}
\end{align*}
\end{proposition}

\begin{proof}
    The proof follows directly from Proposition \ref{appx::prop-cov-matrix} and Lemma \ref{lem::exact_inverse}.
\end{proof}

\subsection{Simplified Representation of SHAP-IQ for the Shapley Value}\label{appx::shap-iq-for-the-sv}
In this section, give an explicit form of SHAP-IQ for the SV that admits a similar form as the SV representation in Theorem \ref{thm::SV_representation}.
We consider the SV weights $m(t) := \frac{(\fnum-t-1)!t!}{\fnum!}$.

\begin{proposition}[SHAP-IQ for SV]\label{prop::shap_x_for_SV}
For SHAP-IQ with $p(T) \propto \mu(t)$ and sampling order $k_0=1$, it holds 
\begin{equation*}
    \hat I^m_1(i) = c_1(i) + \frac{2 h_{\fnum-1}}{K} \sum_{k=1}^K \nu_0(T_k)\left[ \mathbf{1}(i \in T_k) - \frac{t_k}{\fnum}\right],
\end{equation*}
where $h_n$ is the $n$-th harmonic number.
\end{proposition}

\begin{proof}
    Recall the definition of SHAP-IQ of order $1$
\begin{align*}
    \hat I^m_{1}(i) :=  c_{1}(i) + \frac 1 K \cdot \sum_{k=1}^K \nu(T_k) \frac{\gamma^m_s(t_k,\vert T_k \cap \{i\}\vert)}{p(T_k)}.
\end{align*}
with $p(T_k) := \mu(t_k)/R \propto \mu(t_k)$.
We proceed by rewriting $\gamma^m_s(t_k,\vert T_k \cap \{i\}\vert)=\gamma^m_s(t_k,\mathbf{1}(i \in T_k))$ for $T_k \in \mathcal T_1$ as
\begin{align*}
\gamma^m_s(t_k,\mathbf{1}(i \in T_k)) &= (-1)^{1-\mathbf{1}(i \in T_k)}\frac{(\fnum-t-1+\mathbf{1}(i\in T_k))!(t-\mathbf{1}(i\in T_k))!}{\fnum!} 
\\
&= \mu(t) \frac{1}{\fnum}\left[\mathbf{1}(i \in T_k) (\fnum - t_k) - \mathbf{1}(i \notin T_k)t_k \right]
\\
&= \mu(t) \left[\mathbf{1}(i \in T_k) -\frac{t_k}{\fnum} \right].
\end{align*}
Hence,
\begin{align*}
\hat I^m_{1}(i) &:=  c_{1}(i) + \frac 1 K \cdot \sum_{k=1}^K \nu(T_k) \frac{\mu(t) \left[\mathbf{1}(i \in T_k) -\frac{t_k}{\fnum} \right]}{p(T_k)} 
\\
&=  c_{1}(i) + \frac R K \cdot \sum_{k=1}^K \nu(T_k) \left[\mathbf{1}(i \in T_k) -\frac{t_k}{\fnum} \right].
\end{align*}
For the normalizing constant, we have
\begin{align*}
R = \sum_{T \in \mathcal T_1} \mu(t) = \sum_{t=1}^{\fnum-1} \mu(t)\binom{\fnum}{t} = \sum_{t=1}^{\fnum-1} \frac{\fnum}{t(\fnum-t)} = \sum_{t=1}^{\fnum-1}\left(\frac{1}{t} + \frac{1}{\fnum-t}\right) = 2h_{\fnum-1},
\end{align*}
which finishes the proof.
\end{proof}

\clearpage
\section{Approximation Methods for the SV}\label{appx::approx_sv}
There are two prominent representations of the SV, which are used for sampling-based approximation.
Both allow to update all SVs simultaneously with one sample as well as maintaining the efficiency property.

\subsection{Permutation-based (PB) Approximation}
Permutation-based (PB) approximation was introduced for SV \cite{Castro.2009}.
It is based on the observation that the marginal contributions $\delta^\nu_{i}(T)=\nu(T\cup \{i\}) - \nu(T)$ can be computed by using permutations $\pi \in \mathfrak S_{\fset}$ of $\fset$ and $\nu(u^-_i(\pi))-\nu(u^+_i(\pi))$, where $u^-_i(\pi),u^+_i(\pi)$ are the sets that consist of all elements preceding $i$ in $\pi$ with and without $i$, respectively.
For for each subset $T \subseteq \fset \setminus \{i\}$ of size $t$ there are exactly $t!(\fnum-t-1)!$ permutations with $T = u_i^-(\pi)$ and thus
\begin{equation*}
    I^{\text{SV}}(i) = \frac{1}{\fnum!}\sum_{\pi \in \mathfrak S_{\fset}} \delta^\nu_{i}(u_S^-(\pi)) = \mathbb{E}_{\pi \sim \text{unif}(\mathfrak{S}_{\fset})}[\delta^\nu_{i}(u_S^-(\pi))].
\end{equation*}
This expectation can be efficiently approximated by sampling $\pi \sim \text{unif}(\mathfrak S_{\fset})$ and using a Monte Carlo estimate for the expectation.
As $\sum_{i \in \fset} \delta^\nu_i(u_i^-(\pi)) = \nu(\fset)-\nu(\emptyset)$ for arbitrary permutations $\pi$, the efficiency constraint is maintained throughout the sampling procedure.
The Monte Carlo estimates allows to apply well-established statistical results to obtain bounds on the approximation error \cite{Castro.2017}.

\subsection{Kernel-based (KB) Approximation}
Kernel Shapley Additive Explanation Values \cite{Lundberg_Lee_2017}, short KernelSHAP (KSH), and Unbiased KernelSHAP (U-KSH) \cite{Covert_Lee_2021} make use of the representation of the SV as the solution to a constrained quadratic optimization problem \cite{Charnes_Golany_Keane_Rousseau_1988}

\begin{align}\label{eq::wlreg-optimization}
\begin{aligned}
        &I^{\text{SV}} = \argmin_{\beta} \sum_{T \in \mathcal T_1} \mu(t)\left(\nu_0(T)-\sum_{i\in T}\beta_i\right)^2
        \\
        &\text{s.t. } \sum_{i \in \fset} \beta_i = \nu_0(\fset)
\end{aligned}
\end{align}
with $\nu_0(T) := \nu(T)-\nu(\emptyset)$, $\mathcal T_k := \{T\subseteq \fset: k\leq t \leq \fnum-k\}$ and $\mu(t) := \frac{1}{\fnum-1}\binom{\fnum-2}{t-1}^{-1}$.
This quadratic optimization problem can be solved explicitly using the weighted least square solution 
\begin{equation}\label{eq::wlreg-solution}
    I^{\text{SV}} = (\mathbf{Z}^T \mathbf{W} \mathbf{Z})^{-1} \mathbf{Z}^T \mathbf{W} \mathbf{y},
\end{equation}
where $\mathbf{Z} \in \{0,1\}^{2^\fnum \times \fnum}$ is a row-wise binary encoding of all subsets of $T \subseteq \fset$, $\mathbf{W} \in \mathbb{R}^{2^\fnum \times 2^\fnum}$ is a diagonal matrix with the subset weights $\mu$ and $\mathbf{y}$ consists of the evaluations of $\nu_0(T)$ for each subset.
To include the optimization constraint, the (otherwise undefined) weights $\mu(\fnum),\mu(0)$ of $\fset$ and $\emptyset$ are set to a high positive constant.
Solving (\ref{eq::wlreg-solution}) still requires $2^\fnum$ model evaluations and thus KSH \cite{Lundberg_Lee_2017} approximates $I^{\text{SV}}$ by considering (\ref{eq::wlreg-solution}) as an expectation
\begin{equation*}
    \sum_{T \in \mathcal T_1} \mu(t)\left(\nu_0(T)-\sum_{i\in T}\beta_i\right)^2 \propto \mathbb{E}_{T \sim p(T)}\left[\left(\nu_0(T)-\sum_{i\in T}\beta_i\right)^2\right]
\end{equation*}
with $p(T) \propto \mu(t)$.
Note that the optimization problem is invariant in terms of scaling.
This expectation is the approximated similarly to SHAP-IQ, by computing high and low subset sizes explicitly and using Monte Carlo integration for the center sizes.
The KSH estimator is difficult to analyze and it is only known that it is asymptotically unbiased \cite{Williamson_Feng_2020}.

KSH constructs a collection of subsets by determining a \emph{sampling order} $k_0$, such that subsets with $k_0\leq t\leq \fnum - k_0$ are sampled from $p(T) \propto \mu(t)$ and for $t<k_0$ or $t>\fnum-k_0$ all possible subsets are used.
The value of $k_0$ is thereby found by successively comparing the expected number of subsets with the total number of subsets of that size.
As the number of subsets $\binom{\fnum}{t}$ for fixed $\fnum$ is a symmetric log-concave sequence of positive terms, it has a maximum at the middle term(s) $\lfloor \frac \fnum 2 \rfloor, \lceil \frac \fnum 2 \rceil$ and grows monotonically and symmetrically as $\binom{\fnum}{t}=\binom{\fnum}{\fnum-t}$ towards this maximum from both sides.
Thus, the implementation starts the comparison at $k_0=0$ and  iteratively increases the $k_0$ candidate.


\end{document}